\documentclass[12pt,pdftex,noinfoline]{imsart}

\makeatletter
\let\c@author\relax
\makeatother
\newcommand\textvtt[1]{{\normalfont\fontfamily{cmvtt}\selectfont #1}}

\usepackage{fullpage}
\usepackage[utf8]{inputenc} % allow utf-8 input
\usepackage[T1]{fontenc}    % use 8-bit T1 fonts
\usepackage{hyperref}       % hyperlinks
\usepackage{url}            % simple URL typesetting
\usepackage{booktabs}       % professional-quality tables
\usepackage{multirow}
\usepackage{amsfonts}       % blackboard math symbols
\usepackage{nicefrac}       % compact symbols for 1/2, etc.
\usepackage{microtype}      % microtypography
\usepackage{float}
\usepackage{wrapfig}
\usepackage{lmodern}

\usepackage{color}
\usepackage[dvipsnames]{xcolor}

\usepackage{parskip}
\usepackage{graphicx}
\usepackage{subcaption}
\usepackage{enumitem}

\usepackage{tikz}
\usetikzlibrary{arrows, arrows.meta, positioning, automata, shapes.multipart, calc}

\usepackage{pgfplots}
\pgfplotsset{compat=1.18}

\usepackage{setspace}
\usepackage{algorithm2e}
\RestyleAlgo{ruled}

\usepackage[
backend=biber, style=authoryear-comp, url=false, isbn=false, doi=false,
sortcites=true,
maxbibnames=10, %sorting=none, 
natbib=true, backref=true]{biblatex}
\DefineBibliographyStrings{english}{%
  backrefpage = {cited on page},% originally "cited on page"
  backrefpages = {cited on pages},% originally "cited on pages"
}
\AtEveryBibitem{%
  \clearfield{volume}%
  \clearfield{number}
  \clearfield{pages}}
\DeclareFieldFormat{title}{\mkbibquote{#1}}
\usepackage{pdfcomment}
\usepackage[noinline, margin, index, mode=multiuser, author=]{fixme}
\usepackage[us]{datetime}
\fxloadlayouts{pdfcnote}

\fxuseenvlayout{colorsig} % plain, signature, color, colorsig
\fxusetargetlayout{changebar} % plain, changebar, color, colorcb

\FXRegisterAuthor{aa}{anaa}{\color{NavyBlue}Awni}
\FXRegisterAuthor{jl}{anjl}{\color{RedViolet}John}
\FXRegisterAuthor{om}{anom}{\color{RubineRed}Omar}

\fxsetface{inline}{\itshape}
\fxsetface{margin}{\scriptsize\itshape}

\definecolor{fxnote}{HTML}{43AA8B}
\definecolor{fxwarning}{HTML}{F3722C}
\definecolor{fxfatal}{HTML}{F94144}
\hypersetup{%
    colorlinks=true, linktocpage=true, pdfstartpage=3, pdfstartview=FitV,%
    breaklinks=true, pdfpagemode=UseNone, pageanchor=true, pdfpagemode=UseOutlines,%
    plainpages=false, bookmarksnumbered, bookmarksopen=true, bookmarksopenlevel=1,%
    hypertexnames=true, pdfhighlight=/O,%nesting=true,%frenchlinks,%
    urlcolor=MidnightBlue , linkcolor=MidnightBlue, citecolor=ForestGreen, % Classic & Clear
    pdfsubject={},%
    pdfkeywords={},%
}

\newcommand{\edit}[1]{{\color{BrickRed} #1}}

\usepackage{mymathstyle}
\usepackage{paper_commands}
\usepackage{theorem_envs}

\setlist[itemize]{label=$\circ$}
\usepackage{fullpage}
\usepackage{times}

\begin{document}

\parindent0pt

%\maketitle
\def\snote#1{${}^{#1}$}
\setlength{\parskip}{0.5em}
\begin{frontmatter}
\mbox{\ }
\vskip1in
{\bf\Large CoT Information: Improved Sample Complexity \\[5pt] under Chain-of-Thought Supervision}
%\affil[**]{Department of Statistics and Data Science, Yale University}
\begin{aug}
\vskip15pt
\address{
\begin{tabular}{ccccc}
{\normalsize\rm\bfseries Awni Altabaa}\snote{1} &
{\normalsize\rm\bfseries Omar Montasser}\snote{2} & {\normalsize\rm\bfseries John Lafferty}\snote{3}\\[5pt]
\end{tabular}
\vskip5pt
\footnotetext{
\snote{1}Department of Statistics and Data Science, Yale University; \textit{awni.altabaa@yale.edu}.
\snote{2}Department of Statistics and Data Science, Yale University; \textit{omar.montasser@yale.edu}.
\snote{3}Department of Statistics and Data Science, Wu Tsai Institute, Yale University; \textit{john.lafferty@yale.edu}.
}
\today
\vskip10pt
}
\begin{abstract}
% archived on 05/19/2025
% Learning complex functions that involve multi-step reasoning poses a significant challenge for standard supervised learning from input-output examples.  Chain-of-thought (CoT) supervision, which provides intermediate reasoning steps together with the final output, has emerged as a powerful empirical technique, underpinning much of the recent progress in the reasoning capabilities of large language models.  This paper develops a statistical theory of learning under CoT supervision. A formal framework for representing CoT hypothesis classes is introduced, based on learning objectives focused on achieving small end-to-end (E2E) error using CoT supervision during training. Central to the theory is the CoT information measure $\cotinfo(\epsilon; \calH)$, which quantifies the additional discriminative power gained from observing the ``thought process'' for distinguishing hypotheses with different E2E behaviors. The main theoretical results demonstrate how CoT supervision can yield significantly faster learning rates compared to standard E2E supervision. Specifically, it is shown that the sample complexity required to achieve a target E2E error $\epsilon$ scales with the dimension $d$ of the hypothesis space as $d/\cotinfo(\epsilon; \calH)$, which can be much faster than standard $d/\epsilon$ rates.  Information-theoretic lower bounds in terms of the CoT information are also obtained. Together, these results suggest that CoT information is a fundamental measure of statistical complexity for learning under chain-of-thought supervision.
% slightly updated
Learning complex functions that involve multi-step reasoning poses a significant challenge for standard supervised learning from input-output examples.  Chain-of-thought (CoT) supervision, which provides intermediate reasoning steps together with the final output, has emerged as a powerful empirical technique, underpinning much of the recent progress in the reasoning capabilities of large language models.  This paper develops a statistical theory of learning under CoT supervision. 
A key characteristic of the CoT setting, in contrast to standard supervision, is the mismatch between the training objective (CoT risk) and the test objective (end-to-end risk).
A central part of our analysis, distinguished from prior work, is explicitly linking those two types of risk to achieve sharper sample complexity bounds.
This is achieved via the \textit{CoT information measure} $\cotinfo(\epsilon; \calH)$, which quantifies the additional discriminative power gained from observing the reasoning process.
The main theoretical results demonstrate how CoT supervision can yield significantly faster learning rates compared to standard E2E supervision. Specifically, it is shown that the sample complexity required to achieve a target E2E error $\epsilon$ scales as $d/\cotinfo(\epsilon; \calH)$, where $d$ is a measure of hypothesis class complexity, which can be much faster than standard $d/\epsilon$ rates.  Information-theoretic lower bounds in terms of the CoT information are also obtained. Together, these results suggest that CoT information is a fundamental measure of statistical complexity for learning under chain-of-thought supervision.
\end{abstract}

\end{aug}
\end{frontmatter}

\clearpage
\tableofcontents
\clearpage

\def\prompt#1{``\textvtt{\small #1}''}
\def\answer#1{``\textvtt{\small #1}''}
\def\cotz#1{``\textvtt{\small #1}''}
\long\def\comment#1{}

\section{Introduction}

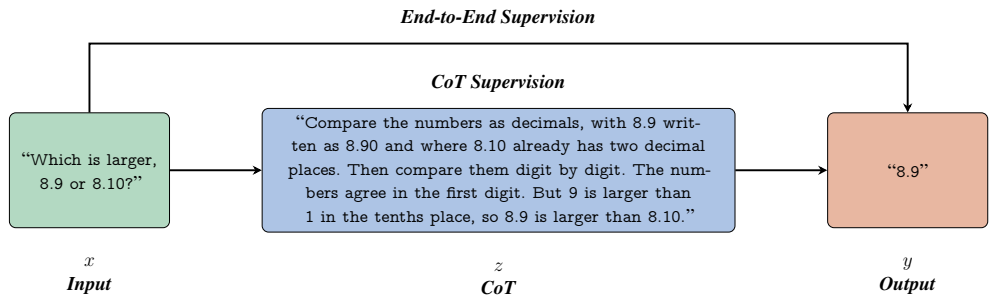
\begin{figure}[t]
    \centering
    \resizebox{0.8\textwidth}{!}{
    \begin{tikzpicture}[
    % Default vertical and horizontal distance between node centers/anchors
    node distance=1cm and 2cm,
    % Base style for all content boxes
    base_box/.style={
        draw,
        rectangle,
        rounded corners,
        minimum height=2.5cm,
        align=center
    },
    % Style for input boxes (x)
    input_style/.style={
        base_box,
        text width=3.2cm,
        fill=ForestGreen!30 % User's chosen color
    },
    % Style for the CoT box (z)
    cot_style/.style={
        base_box,
        text width=10.0cm,
        fill=RoyalBlue!30 % User's chosen color
    },
    % Style for output boxes (y)
    output_style/.style={
        base_box,
        text width=3.2cm,
        fill=BrickRed!30 % User's chosen color
    },
    % Style for column labels (below the boxes)
    col_label/.style={font=\bfseries\itshape, node distance=0.5cm, align=center},
    % Style for supervision labels (above the boxes/arrows)
    supervision_label_style/.style={font=\bfseries\itshape, align=center}
]

% Define content for x, y, z using the macros from the preamble
% (Preamble is in a separate document, ID: tikz_preamble_updated_may07)
\newcommand{\xcontent}{\prompt{Which is larger, 8.9 or 8.10?}}
\newcommand{\ycontent}{\answer{8.9}}
\newcommand{\zcontent}{\cotz{Compare the numbers as decimals, with 8.9 written as 8.90 and where 8.10 already has two decimal places. Then compare them digit by digit. The numbers agree in the first digit. But 9 is larger than 1 in the tenths place, so 8.9 is larger than 8.10.}}

% Main row of nodes: x -> z -> y
% x_node: Input node
\node[input_style] (x_node) {\xcontent};
% z_node: CoT node, positioned to the right of x_node
\node[cot_style, right=of x_node] (z_node) {\zcontent};
% y_node: Output node, positioned to the right of z_node
\node[output_style, right=of z_node] (y_node) {\ycontent};

% Column Labels (underneath the main row nodes)
\node[col_label, below=of x_node] (xlabel) {$x$\\ Input};
\node[col_label, below=of z_node] (zlabel) {$z$\\ CoT};
\node[col_label, below=of y_node] (ylabel) {$y$\\ Output};

% "CoT Supervision" Label - Placed above the z_node
\node[supervision_label_style, above=0.2cm of z_node.north] (cot_supervision_label) {CoT Supervision};

% Arrows
% E2E supervision: x -> y (long upward curved arrow, bypassing z)
% The label "End-to-End Supervision" is placed midway along this arrow.
% \draw[-{Stealth[length=2mm, width=2mm]}, very thick]
%     (x_node.north) to[bend left=40] % 'bend left' creates an upward curve
%     node[midway, above=2mm, supervision_label_style] {End-to-End Supervision}
%     (y_node.north);
\draw[-{Stealth[length=2mm, width=2mm]}, very thick]
    (x_node.north) -- ++(0,1.5cm) % Go up by 0.8cm from x_node.north
    coordinate (e2e_top_left) % Define a coordinate at the top-left bend
    -- (e2e_top_left -| y_node.north) % Go right, to be vertically aligned with y_node.north, at same height as e2e_top_left
    coordinate (e2e_top_right) % Define a coordinate at the top-right bend
    node[midway, above=2mm, supervision_label_style] {End-to-End Supervision} % Place label on this horizontal segment
    -- (y_node.north); % Go down to y_node.north

% CoT supervision arrows: x -> z -> y (straight arrows below the E2E path)
\draw[-{Stealth[length=2mm, width=2mm]}, very thick] (x_node.east) -- (z_node.west);
\draw[-{Stealth[length=2mm, width=2mm]}, very thick] (z_node.east) -- (y_node.west);

\end{tikzpicture}
    }
    \caption{\small An illustration of standard end-to-end supervision and CoT supervision. Our theoretical framework is aimed at understanding tradeoffs between end-to-end supervision and  CoT supervision, and in particular, how the potentially richer information in the CoT signal can result in faster learning rates.}
    \vskip-12pt
    \label{fig:cot_cartoon_example}
\end{figure}

``Chain-of-thought'' (CoT) reasoning has been a driving force behind recent advances in the capabilities of large language models. While chain-of-thought began as a prompting technique~\citep{nye2021show,cobbe_training_2021,wei2022chain}, CoT-supervised training is now an important component of the post-training pipeline for large language models, and has been found to be highly effective in recent empirical research~\citep{hoLargeLanguageModels2023,chung2022scalinginstructionfinetunedlanguagemodels,ouyang2022traininglanguagemodelsfollow}.

This paper proposes new concepts in statistical learning theory that are aimed at gaining insight into chain-of-thought learning. Consider the following concrete example of chain-of-thought, illustrated in~\Cref{fig:cot_cartoon_example}, to ground the theoretical framework to be developed. The input $\x$ is the sequence \prompt{Which is larger, 8.9 or 8.10?} and the intended answer $\y$ is \answer{8.9}.
% (which rules out the possibility that the question is about version numbers). 
When asked to answer directly, earlier systems (e.g., GPT-4) might respond incorrectly~\citep[e.g.,][]{nguyen2023evaluating}. 
However, newer models trained with chain-of-thought supervision will typically first output a CoT $\z$, represented as a sequence of tokens, enabling the model to arrive at the correct answer. For example, the CoT might be, \cotz{Compare the numbers as decimals, with 8.9 written as 8.90 and where 8.10 already has two decimal places. Then compare them digit by digit. The numbers agree in the first digit. But 9 is larger than 1 in the tenths place, so 8.9 is larger than 8.10.}.
At test time, the CoT $\z$ serves as an explanation of the answer. During training, however, the CoT $\z$ plays the role of a natural language description of the execution trace of an algorithm---a step-by-step procedure that is to be learned. In this way, CoT is used as a rich, additional supervised learning signal that goes beyond standard input-output (``end-to-end'') supervision. 
% However, newer models trained with chain-of-thought output a chain-of-thought $\z$ which, for example, might be a sequence of tokens such as \cotz{Compare the numbers as decimals, with 8.9 written as 8.90 and where 8.10 already has two decimal places. Then compare them digit by digit. The numbers agree in the first digit. But 9 is larger than 1 in the tenths place, so 8.9 is larger than 8.10.} At test time, the CoT $\z$ serves as an explanation of the answer. During training, however, the CoT $\z$ plays the role of a natural language description of the execution trace of an algorithm---a step-by-step procedure that is to be learned. In this way, CoT is used as a rich, additional supervised learning signal that goes beyond standard input-output supervision (which we refer to as ``E2E supervision'' in this paper).

%In general, an important role of theory is to shed light on the mechanisms and principles underlying the success of various empirical methods---to give insight into how, why, and when methods may succeed or fail. 
%As seen in this example, the setting of CoT supervision differs qualitatively from the standard supervised setting, in that each sample potentially encodes much more information. Not only does an example contain a labeled input-output example pair, it also includes further supervision about the intermediate computations executed by the reference hypothesis to generate its output. 

The focus of our theory is to describe how this additional information impacts the statistical complexity of CoT-supervised learning. A key contribution of the paper is to identify a quantity that we call the \textit{chain-of-thought information}, denoted by $\cotinfo(\epsilon; \calH)$. As we show, the CoT information characterizes the statistical complexity of CoT-supervised learning and captures the additional discrimination power granted to the learning algorithm by observing the chain-of-thought.  
%We show that $\cotinfo(\epsilon; \calH)$ 
%is a decreasing function in $\epsilon$ that satisfies $\cotinfo(\epsilon; \calH) \geq \epsilon$, and %that it is an increasing function in the hypothesis class $\calH$ under the subset relation.
In particular, the CoT information governs how the end-to-end error of the learned algorithm scales with the number of CoT training examples. Specifically, in the standard setting of PAC learning for binary classification, the sample complexity scales as $m = \calO\left(d/\epsilon\right)$ in the realizable setting, where $d$ describes the size of the hypothesis space (e.g., the VC dimension), and $\epsilon$ is the target classification error. In contrast, under CoT supervision, we show that the sample complexity scales according to $m = \calO\pparen{d/\cotinfo(\epsilon; \calH)}$.  A case where the chain-of-thought is highly informative will have $\cotinfo(\epsilon; \calH) \gg \epsilon$, which translates into favorable sample complexity. By establishing information-theoretic lower bounds, it is shown that the CoT information thus provides a fundamental measure of the value of this type of non-classical supervision. Because the construction of CoT training sets can be a time-consuming and expensive process, the theoretical framework developed in this paper may ultimately be of practical relevance, contributing to a formal understanding and quantification of the \textit{value} of chain-of-thought supervision.

\aawarning{
A discussion/way to motivate/interpret our results that we can consider incorporating in a future version of the paper:
In general, the error ($\epsilon$) dependence of the sample complexity, appearing in the denominator, reflects the amount of information revealed by each random sample and determines the rate at which the error decays with increasing sample size $m$. For instance, in the noise-free realizable setting, the sample complexity scales as $1/\epsilon$, yielding a $1/n$ error rate. In the agnostic setting with arbitrary noise, it scales as $1/\epsilon^2$, corresponding to a slower $1/\sqrt{n}$ rate. Under structural conditions on the noise---such as those studied by Tsybakov~\citep{mammen1999smooth,tsybakov2004optimal,tsybakov2007fast}, Massart~\citep{massartRiskBoundsStatistical2006}, and others~\citep{bartlett2006empirical}---it is possible to achieves rates interpolating rates of  $1/m^\alpha$ (with a sample complexity of $1/\epsilon^{1/\alpha}$), where $\alpha \in [1/2, 1]$ is a measure of the amount of noise. Here, the amount of noise can be interpreted as the amount of information per sample, and appears via the error dependence of the sample complexity. Analogously, under CoT supervision, where the learner observes more information than just the input-output example pairs, one would expect this increased information to appear in the error dependence of the sample complexity. We show that this effect is captured by the CoT information $\cotinfo(\epsilon; \calH)$, leading to a sample complexity that scales as $1/\cotinfo(\epsilon; \calH)$.
}

\textbf{Organization.} The remainder of the paper is organized as follows.
\Cref{sec:prelims} introduces an abstract model of chain-of-thought supervised learning, together with formal definitions of the learning objective and the key notions of risk, which will be the focus of our investigation.
In~\Cref{sec:cotinfo}, we motivate and introduce the CoT information measure, establish its fundamental properties, and, as an initial pedagogical result, show that it can be used to capture the improved sample complexity of CoT-supervised learning in the setting of finite-cardinality hypothesis classes.
\Cref{sec:upper_bounds} extends this analysis to infinite hypothesis classes, as well as to the agnostic setting where the data are not assumed to be generated by a member of the class. 
In~\Cref{sec:lower_bounds}, two types of information-theoretic lower bounds are established that, together with the upper bounds, lends further support to considering the CoT information as a fundamental characterization of the value of chain-of-thought supervision. 
In~\Cref{sec:simulations}, we present simulation results that numerically compute the CoT information for two examples of CoT hypothesis classes---deterministic finite automata and iterated linear thresholds---and empirically evaluate the sample complexity of learning with end-to-end supervision compared to CoT supervision, finding close alignment with theoretical predictions derived via the CoT information.
\Cref{sec:discussion} concludes with a summary of further extensions that are presented in the appendix, a discussion of related work, and directions for future research that are suggested by the results of this paper.

\comment{
\edit{Preliminary outline, points to touch on, etc.}

% a dump of some intro text I typed on my phone a while back, in case it's helpful
[[
One of the keys to the advancement in the reasoning capabilities of modern AI systems has been the so-called chain-of-thought training paradigm, where a model, seeking to learn some complex function or input-output mapping, such as complex mathematical reasoning functionality, is trained with supervision not only from input-output examples, as with standard supervised learning, but also on the step-by-step computational trace from some reference solution or algorithm. This paradigm has enabled significant advancements, enabling machine learning models, typically built on (large) language models, to learn more complex functions and to attain some degree of out-of-distribution generalization.

Chain-of-thought techniques, under this terminology, began as "prompting" methods, where the query given to a language model is modified to prime or induce the model to first generate a step-by-step solution before outputting a final answer. This is achieved either by simply asking the model to generate the solution and "show its work" or by conditioning it to do so via in-context learning techniques. In modern usage, chain-of-thought techniques take a broader meaning beyond prompting, and now form a key component of the post-training pipeline of modern large language models, where the chain-of-thought annotated datasets are procured and the models are explicitly trained on these datasets to learn complex (reasoning) behaviors. This now underpins much of the progress made in the reasoning capabilities of modern LLMs, and forms a central empirical tool.

Empirically, it is known that chain-of-thought training yields a significant statistical advantage, enabling learning of complex functions and behaviors that are otherwise difficult to learn from data (i.e., from input-output samples alone). However, the theoretical statistical foundations of these methods are not well understood.

]]

\begin{itemize}
    \item CoT is important in practice; has been a driving force behind much of the advancement of capabilities of LLMs, particularly in the reasoning domain. Started off as a prompting technique \cite{...} but CoT-supervised training is now a central component of the post-training pipeline for all modern training recipes \cite{...}. Beyond LLMs, CoT-based training techniques have been found to be highly effective for learning complex functions in recent empirical research \cite{...} and have been found to enable a degree of length generalization as well.
    \item An important role of theory is to shed light on the \textit{mechanisms} underlying the success of various empirical methods. Our focus is on characterizing and understanding the \textit{statistical} advantage behind chain-of-thought supervised training, formalizing the driving intuition behind these methods.
    \item Generating chain-of-thought supervised training datasets can be a laborious and expensive process, often requiring human annotators to manually create these datasets. Thus, it is important to develop a formal understanding and quantification of the \textit{value} of chain-of-thought supervision.
    \item Describe main ideas behind our work
    \begin{itemize}
        \item The statistical complexity of learning problems scale with the parameters of the problem. In particular, intuitively, the statistical complexity scales with \aanote*{is this a fair characterization?}{two main attributes}: 1) the complexity of the hypothesis class to be learned, and 2) how informative each sample is with respect to the target objective. For e.g., in the binary classification setting, under the PAC framework, the VC dimension characterizes the scaling of the sample complexity with respect to the complexity/size of the hypothesis space. On the other hand, in the standard supervised setting where samples encode input-output pairs, the sample complexity scales like $1/ \epsilon$ in the realizable setting and $1/ \epsilon^2$ in the agnostic setting.
        \item The CoT supervision setting differs qualitatively/fundamentally from the standard supervised setting in that each sample encodes more information: not only does it contain a labeled input-output example pair, it also includes further supervision about the intermediate computations executed by the reference hypothesis to generate its output. Thus, we focus this work primarily on how the statistical complexity of CoT-supervised learning scales with the error-parameter, capturing how informative each sample is with respect to the target objective.
        \item We identify a specific information measure, termed the \textit{chain-of-thought information} and denoted denoted by $\cotinfo(\epsilon; \calH)$, which characterizes the statistical complexity of CoT-supervised learning and captures the additional discrimination power granted to the learning algorithm by observing the chain-of-thought. 
        \item We establish both upper bound and lower bound results that suggest that the CoT information is a fundamental measure of statistical complexity in the CoT-supervised setting, considering both finite and general hypothesis classes, as well as both the realizable and agnostic settings.  
    \end{itemize}
    \aanote{The ``Modeling Chain-of-Thought as partial observations of algorithmic description'' section in the manuscript includes a few ways to describe what CoT supervision typically looks like. E.g., an input-dependent partial observation of different components of the algorithmic description/execution process.}
\end{itemize}
}

%\newpage

\section{Preliminaries: A Model of Chain-of-Thought Supervised Learning}\label{sec:prelims}

% Section: A Model of Chain-of-Thought-supervised Learning (Condensed)
\def\calF{\calH}

The standard statistical learning problem is formulated as the problem of selecting a distinguished member of a function class $\calF: \calX \to \calY$, mapping from an input space $\calX$ to an output space $\calY$. The learner observes a dataset of input-output examples $\sset{(x_i, y_i)}_{i \in [m]}$ and seeks to identify the ground truth function $\fstar \in \calF$ (in the realizable setting) or compete with its closest approximation in $\calF$ (in the agnostic setting). A learning algorithm in the standard (``end-to-end'') setting is a mapping $\calA: (\calX \times \calY)^{*} \to \calY^{\calX}$ from input-output datasets to predictors.
% \footnote{For a set $A$, $A^*$ denotes $\cup_{t=0}^{\infty} A^t$.
% % (i.e., the set of all finite sequences of elements from $A$)
% $\calY^\calX$ denotes the set of functions from $\calX$ to $\calY$.}.

%\textbf{\textit{Formulating Chain-of-Thought Hypothesis Classes.}} 
When the target function class $\calF$ is highly complex---such as functions representing multi-step reasoning processes---learning from input-output examples alone can be statistically intractable. To overcome this difficulty, a natural approach is to provide the learner with increased supervision through the step-by-step execution of the target function on the input. To formulate this, we assume that each example observed by the learner includes not only the input $x$ and output $y$, but also an auxiliary observation $z$ that represents information about the function's execution on $x$.% such as intermediate computations. 
% In LLM settings, the input space $\calX$ is often the collection of sequences $x\in\Sigma^*$ over some token vocabulary $\Sigma$, and the output $y\in\Sigma$ can be thought of as a single token. Again, in the setting of LLMs, the auxiliary observation $z\in\Sigma^*$ can be thought of as a variable-length sequence over the token vocabulary, in which case $\calZ = \Sigma^*$. 
% We will use notation that uses boldface for inputs $x$ and auxiliary variables $z$, thinking of these as variable-length strings.

A \textit{chain-of-thought (CoT) hypothesis class} $\calH$ is a family of functions $h: \calX \to \calY \times \calZ$. For each $h \in \calH$, an input $x \in \calX$ yields $h(x)=(y,z)$, where $y \in \calY$ is the output and $z \in \calZ$ is the corresponding CoT. We denote the components of h returning only the output as its end-to-end restriction, $\hete{h}: \calX \to \calY$, and the component returning only the CoT as its CoT restriction, $\hcot{h}: \calX \to \calZ$.
% As remarked above, the CoT space $\calZ$ will often be a space $\Sigma^*$ of variable-length strings over some alphabet $\Sigma$, representing a sequence of intermediate computations. 
In the chain-of-thought learning setting, the learner observes a dataset $\sset{(x_i, y_i, z_i)}_{i \in [m]}$ and seeks to learn the underlying end-to-end function. A \textit{chain-of-thought learning algorithm} is a mapping $\calA: (\calX \times \calY \times \calZ)^* \to \calY^{\calX}$ from CoT datasets to predictors
\footnote{Generally, a CoT learning algorithm need only output an end-to-end predictor ($\calX \to \calY$), not necessarily the chain-of-thought, though it might produce both.The upper bounds established in this work consider consistency and ERM rules, which output a hypothesis in $\calH$ that returns both the CoT and the output. However, these details in the formulation can be relevant for studying for complex learning algorithms (e.g., improper rules).}.

A key example captured by this framework is autoregressive sequence models, generating the CoT sequentially as $z_t = f(\xs{n}, \zs{t-1})$ until a final output $y = f(\xs{n}, \zs{t(x)})$ is generated. In this case, the spaces $\calX, \calY, \calZ$ would correspond to spaces of variable-length sequences over some vocabulary.
Sequence models like transformers~\citep{vaswani2017attention} are an important way to \textit{implement} such hypothesis classes in a way that allows for CoT supervision. However, the details of any such implementation are not important for our theoretical treatment.

% \begin{definition}[Chain-of-Thought Hypothesis Class]\label{def:cot_hyp_class}
%     A CoT hypothesis class is a class of functions $\calH: \calX \to \calY \times \calZ$, mapping from input instances $x \in \calX$ to outputs $y \in \calY$, and a chain-of-thought $z \in \calZ$. For $h \in \calH$, $\hete{h}: \calX \to \calY$ is the end-to-end restriction (output prediction) and $\hcot{h}: \calX \to \calZ$ is the CoT restriction (intermediate steps).
% \end{definition}

\textbf{\textit{Types of risk.}} It will be crucial to distinguish between two notions of risk: End-to-end risk and the chain-of-thought risk. Let $\calD$ be a distribution over $\calX$. For a reference hypothesis $\hstar \in \calH$ and a predictor $h \in (\calY \times \calZ)^{\calX}$, we define these risks as follows:
\begin{equation*}
    \Lete{\calD}(h) = \probunder{x \sim \calD}{\hete{h}(x) \neq \hete{\hstar}(x)}, \quad 
    \Lcot{\calD}(h) = 
    %\probunder{x \sim \calD}{(\hete{h}(x), \hcot{h}(x)) \neq (\hete{\hstar}(x), \hcot{\hstar}(x))}
    \probunder{x \sim \calD}{ h(x) \neq \hstar(x)}.
\end{equation*}
% Note that we could also write the CoT risk as
% $\Lcot{\calD}(h) = \pprobunder{x \sim \calD}{h(x) \neq \hstar(x)}$, observing that a hypothesis returns both a final answer $y$ and a chain-of-thought $z$.
That is, $\Lete{\calD}(h)$ is the probability that the predictor's end-to-end \textit{output} is incorrect, whereas $\Lcot{\calD}(h)$ is the probability that \textit{either} the output $\hete{h}(x)$ or the CoT $\hcot{h}(x)$ disagrees with $\hstar$.
A key characteristic of the chain-of-thought supervised learning setting is that the training objective is the CoT loss, whereas the testing evaluation metric is the end-to-end risk. This asymmetry has important information-theoretic implications, which are a main focus of this work.

We now define the chain-of-thought learning problem within a PAC-style framework.
In CoT learning, the learner sees training examples $S= \sset{(\x_i, y_i, \z_i)}_{i \in [m]}$, and the objective is to learn a predictor $\calA(S)$ with low \textit{end-to-end error} $\Lete{\calD}(\calA(S))$ on test inputs. Importantly, errors in predicting the chain-of-thought $\z$ are not penalized at test time, only errors in the final output $y$.

\vskip10pt
\begin{definition}[Realizable chain-of-thought PAC learning]\label{def:realizable_pac_cot_learning}
    $\calH \subset (\calY \times \calZ)^{\calX}$ is CoT-learnable with sample complexity $m_{\calH, \calD}(\epsilon, \delta)$ if there exists an algorithm $\calA: (\calX \times \calY \times \calZ)^* \to \calY^{\calX}$ such that for any distribution $\calD$ over $\calX$ and $\hstar \in \calH$, given $m \geq m_{\calH, \calD}(\epsilon, \delta)$ samples $S = \sset{(\x_i, \hete{\hstar}(\x_i), \hcot{\hstar}(\x_i))}_{i \in [m]}$, $x_1, \ldots, x_m \simiid \calD$, $\calA$ outputs $\calA(S)$ satisfying $\Lete{\calD}(\calA(S)) \leq \epsilon$ with probability at least $1 - \delta$.
\end{definition}
%\aanote{I guess here $\hstar$ is implicit in the definition of $\Lcot{\calD, \hstar}(h)$. We can put it in the subscript explicitly?}
%\omnote{that might work.}

%\aanote{Maybe need to spend a bit more time on the definition of the agnostic setting? Or say that that will come later?}
In the \textit{agnostic} setting, the data distribution $\calD$ over $\calX \times \calY \times \calZ$ may not be perfectly realizable by $\calH$. In this case, the goal is to compete with the best end-to-end error in $\calH$.

\vskip10pt
\begin{definition}[Agnostic chain-of-thought PAC learning]\label{def:agnostic_pac_cot_learning}
    $\calH \subset (\calY \times \calZ)^{\calX}$ is agnostically CoT-learnable with sample complexity $m_{\calH, \calD}(\epsilon, \delta)$ if there exists an algorithm $\calA: (\calX \times \calY \times \calZ)^* \to \calY^{\calX}$ such that for any distribution $\calD$ over $\calX \times \calY \times \calZ$, given $m \geq m_{\calH, \calD}(\epsilon, \delta)$ samples $S \sim \calD^m$, $\calA$ outputs $\calA(S)$ satisfying $\Lete{\calD}(\calA(S)) - \inf_{h \in \calH} \Lete{\calD}(h) \leq \epsilon$ with probability at least $1 - \delta$.
\end{definition}

\vskip10pt
\subsection{Interlude: The problem of linking the end-to-end and chain-of-thought risks}\label{ssec:key_idea_control_cot_error}

\begin{table}
    \centering
    \vskip15pt
    \caption{A comparison of analysis techniques for studying learning with chain-of-thought.}\label{table:comparison_table}
    \begin{small}
\begin{tabular}{llc}
\toprule
   Method &  Hypothesis class  & Sample complexity (realizable)  \\
\midrule
\multirow{2}{*}{\text{E2E supervision}}        
    & Finite $\calH$  & $\log \aabs{\calH} / \epsilon$                             \\
\addlinespace
    & General $\calH$ & $\VC(\calLete(\calH))/ \epsilon$                           \\
 \midrule
\multirow{2}{*}{\text{bounding CoT risk}} & Finite $\calH$  & $\log \aabs{\calH} / \epsilon$                             \\
\addlinespace
    & General $\calH$ & $\VC(\calLCoT(\calH))/ \epsilon$                           \\
\midrule
\multirow{2}{*}{\text{using CoT Information}}    
    & Finite $\calH$  & $\log \aabs{\calH} / \cotinfo(\epsilon; \calH)$    \\
\addlinespace % Optional extra space from booktabs
    & General $\calH$ & $\VC(\calLCoT(\calH)) / \cotinfo(\epsilon; \calH)$ \\
\bottomrule
\end{tabular}
\end{small}
    % \vskip2pt
\end{table}
% \begin{table}[H]
%     \centering
%     \begin{tabular}{lcc} % Left-aligned labels, centered math
%         \toprule % Rule from booktabs
%         & \textbf{Realizable} & \textbf{Agnostic} \\ % Bold headers
%         \midrule % Rule from booktabs
%         Finite $\calH$ & $m(\epsilon, \delta) = \frac{\log \abs{\calH} + \log(1 / \delta)}{\epsilon}$ & $m(\epsilon, \delta) = \frac{\log \abs{\calH} + \log(1 / \delta)}{\epsilon^2}$ \\
%         \addlinespace % Optional extra space from booktabs
%         General $\calH$ & $m(\epsilon, \delta) = \bigOtilde\paren{\frac{\VC(\calLete(\calH)) + \log(1 / \delta)}{\epsilon}}$ & $m(\epsilon, \delta) = \bigOtilde\paren{\frac{\VC(\calLete(\calH)) + \log(1 / \delta)}{\epsilon^2}}$ \\
%         \bottomrule % Rule from booktabs
%     \end{tabular}
% \end{table}

%\aawarning{Would it be helpful to include paragraph headings about the End-to-End analysis vs CoT error analysis? (like we used to have)}

Before introducing the CoT information measure and our main theoretical result, we first motivate a key aspect of the analysis, which is specific to the CoT setting. In CoT learning, the learner observes training examples $S= \sset{(x_i, y_i, z_i)}_{i \in [m]}$ and seeks to identify the \textit{input-output} relationship using information from \textit{both} the output $y_i$ and CoT labels $z_i$. 
That is, although the CoT error is used as a signal during training, only errors in the final output $y$ are penalized at test time. %, not errors in predicting the chain-of-thought $z$.
% That is, although the CoT error is used as a signal in training, errors in predicting the chain-of-thought $z$ are not penalized at test time, only errors in the final output $y$. 
Consequently, to derive sharp statistical rates, it is necessary to link the two risk functions precisely.

To explain, recall that standard statistical learning theory characterizes the statistical complexity of learning from input-output examples \textit{without} chain-of-thought supervision.
For example, focusing on the realizable case for clarity, standard results in PAC learning~\citep[e.g.,][]{vapnik:82} show that the sample complexity to obtain end-to-end error $\epsilon$ scales as $d / \epsilon$, where $d$ is a complexity measure such as log-cardinality or the VC dimension of the end-to-end loss class $\calLete(\calH)$.
% In this classical setting, $\calLete(\calH) = \sset{\ellete_{h}: (x, y) \mapsto \Ind{\hete{h}(x) \neq y} \ggiven h \in \calH}$ is the end-to-end loss class of  $\calH$. Focusing on the realizable case for clarity, and ignoring model dimension, the sample complexity to obtain end-to-end error $\epsilon$ scales as $1 / \epsilon$. 
Intuitively, the $\epsilon$-dependence can be understood in terms of the amount of information per sample, as $\calO(1 / \epsilon)$ samples are required to distinguish between two hypotheses whose outputs disagree on a subset of measure $\epsilon$ in the input space. Matching information-theoretic lower bounds validate that these are the optimal learning rates for the standard E2E-supervised setting.

In the CoT-supervised setting, the learning algorithm potentially has access to more information by observing the CoT, and thus faster rates of convergence are expected. The theoretical challenge lies in capturing this added information as improved rates in the analysis. Standard learning theory results cannot be directly applied to the CoT setting due to the mismatch between the training objective and the evaluation metric. One approach to address this challenge, which is taken by~\citet{joshi2025theorylearningautoregressivechain}, is to side-step this asymmetry by noting that the end-to-end error is always upper bounded by the CoT error, with $\Lete{\calD}(h) \leq \Lcot{\calD}(h)$, and to instead establish a guarantee on the CoT risk, allowing the use of standard results in learning theory. In particular, one can define the CoT loss class for the hypothesis class $\calH$ as a function class over $\calX \times \calY \times \calZ$ according to
\[\calLCoT(\calH) = \set{\ellcot_h: (x, y, z) \mapsto \Ind{h(x) \neq (y, z)} \given h \in \calH}.\]
Then, appealing to standard results in PAC learning \citep[e.g., Vapnik's ``General Learning'' framework][]{vapnik:82}, one can learn $\calH$ to obtain a \textit{CoT risk} of $\epsilon$ with a sample complexity $m(\epsilon) = \calO(\VC(\calLCoT(\calH)) /\epsilon)$, which in turn guarantees that the end-to-end risk is also bounded by~$\epsilon$.
% \[m(\epsilon, \delta) = \bigOtilde\paren{\frac{\VC(\calLCoT(\calH)) + \log(1/\delta)}{\epsilon}},\]
% \begin{table}[H]
%     \centering
%     \begin{tabular}{lcc} % Left-aligned labels, centered math
%         \toprule % Rule from booktabs
%         & \textbf{Realizable} & \textbf{Agnostic} \\ % Bold headers
%         \midrule % Rule from booktabs
%         Finite $\calH$ & $m(\epsilon, \delta) = \frac{\log \abs{\calH} + \log(1 / \delta)}{\epsilon}$ & ? \\
%         \addlinespace % Optional extra space from booktabs
%         General $\calH$ & $m(\epsilon, \delta) = \bigOtilde\paren{\frac{\VC(\calLCoT(\calH)) + \log(1 / \delta)}{\epsilon}}$ & ? \\
%         \bottomrule % Rule from booktabs
%     \end{tabular}
% \end{table}

This method of analysis leads to a sample complexity with the same $1 / \epsilon$ rate that we see in the end-to-end supervision setting, despite the increased amount of information per sample. In particular, this does not imply improved sample complexity over standard end-to-end supervision in the case of finite-cardinality classes (c.f. \Cref{table:comparison_table}).
In the general case, improved sample complexity hinges on whether or not the inequality $\VC(\calLCoT(\calH)) \ll \VC(\calLete(\calH))$ holds, which is \textit{a priori} unclear, even if it is possible to construct artificial classes for which this holds \citep{joshi2025theorylearningautoregressivechain}. This suboptimality stems from the fact that this approach does not distinguish between the two types of risk and does not explicitly measure the amount of information encoded in the chain-of-thought. As a consequence, this approach can not achieve matching information-theoretic lower bounds. Moreover, it is unclear whether it is meaningful to apply this type of analysis to the agnostic setting, where the distribution over input-output-CoT examples is not realizable by the CoT hypothesis class.

%\aanote[inline,nomargin]{Instead, the only place where this analysis can yield an advantage in sample complexity for CoT supervision is through the complexity of the CoT loss class compared to the E2E loss class in the numerator. This was the focus on~\citet{joshi2025theorylearningautoregressivechain}, where they constructed an example of a class where there is such a separation. We will discuss the relationship between our work and their's in more detail in the discussion section at the end of the paper.}

%\aawarning{Need to discuss difference in ``mechanism''? The mechanism that yields improved sample complexity }

\section{Key Idea: The CoT Information Measure}\label{sec:cotinfo}

%\subsection{The CoT information measure}\label{ssec:key_idea_cotinfo}

We now describe a new approach that explicitly accounts for the additional information provided in the CoT supervision for distinguishing between hypotheses with different end-to-end behaviors. The central quantity in this analysis is the \textit{CoT information}, defined below.

\begin{definition}[CoT information]\label{def:cotinfo}
    For a CoT hypothesis class $\calH \subset (\calY \times \calZ)^{\calX}$ and distribution $\calD$ over $\calX$, we define the CoT information measures as follows:
    \begin{align*}
      \relcotinfo(h_1, h_2)
      &= - \log \probunder{\x \sim \calD}{\hcot{h_1}(\x) = \hcot{h_2}(\x), \hete{h_1}(\x) = \hete{h_2}(\x)} \\
      \cotinfo(\epsilon; \calH) &= \inf_{h \in \Deltaete_{\calD}(\epsilon; \calH, \hstar)} \relcotinfo(\hstar, h).% \ \cotinfo(\epsilon; \calH) = \min_{h \in \calH} \cotinfo_{\calD}(\epsilon; \calH, h).
    \end{align*}
    where the infimum is over $\Deltaete_{\calD}(\epsilon; \calH, \hstar)$, the set of hypotheses that disagree with the  end-to-end behavior (i.e., output) of $\hstar$ with probability at least $\epsilon$,
    \[\Deltaete_{\calD}(\epsilon; \calH, \hstar) := \set{h \in \calH: \probunder{\x \sim \calD}{\hete{\hstar}(\x) \neq \hete{h}(\x)} > \epsilon}.\]
    %Alternatively, for an $\calH$-realizable distribution over $\calX \times \calY \times \calZ$, we may omit %the explicit dependence on $\hstar$ and write
    %\[\cotinfo(\epsilon; \calH) := \inf\set{- \log \probunder{\x, y, \z \sim \calD}{h(\x) = (y, \z)} \ : %\ h \in \Deltaete_{\calD}(\epsilon; \calH)}.\]
\end{definition}

The relative CoT information between two hypotheses $\relcotinfo(h_1, h_2)$ quantifies how effectively the observed CoT behavior distinguishes the two hypotheses. In particular, the probability $\pprob{\hcot{h_1}(x) = \hcot{h_2}(x), \, \hete{h_1}(x) = \hete{h_2}(x)} \in (0,1)$ represents the proportion of inputs on which a pair of hypotheses have matching behavior on \textit{both} the CoT and the end-to-end output, rendering them indistinguishable from these observations.
% This is the proportion of inputs on which $h_1, h_2$ cannot be distinguished using the observables. 
The relative CoT information between a pair of hypotheses is the negative logarithm of this probability; thus, $\relcotinfo(h_1, h_2)$ takes values in $[0, \infty)$.

The CoT information of a \textit{hypothesis class} $\calH$, relative to the reference hypothesis $\hstar$, is a function of the error level $\epsilon$, denoted $\cotinfo(\epsilon; \calH)$. It is defined as the minimal relative CoT information between $\hstar$ and every alternative hypothesis $h \in \Deltaete_\calD(\epsilon; \calH, \hstar)$ which disagrees with $\hstar$'s end-to-end output on more than an $\epsilon$ fraction of the inputs. 
A large $\cotinfo(\epsilon; \calH)$ thus ensures high distinguishability (via CoT) between $\hstar$ and any such "bad" alternative. 
%Observe that this definition only asks for discriminability against alternative hypotheses with significantly differing \textit{end-to-end} behavior. In particular, we do not need the CoT to discriminate between hypotheses that have similar input-output behavior.

A primary message of this work is that the CoT information characterizes the $\epsilon$-dependence of sample complexity in Chain-of-Thought supervised learning by quantifying the informativeness of CoT supervision. The CoT information can be much larger than $\epsilon$, yielding rapid learning under CoT supervision. The intuition is that when two hypotheses differ in terms of their end-to-end behavior, even with small probability, they will typically differ in terms of their computational traces (i.e., CoT) with high probability. Consequently, CoT supervision allows these differing hypotheses to be distinguished far more rapidly than by observing input-output samples alone.

\subsection{Properties of the CoT information}\label{ssec:cot_properties}

The following result outlines key properties of the CoT information measure. Among these, the property $\cotinfo(\epsilon; \calH) \geq \epsilon$ is particularly important. As will be demonstrated, this implies that, in the realizable setting, CoT supervision is never detrimental, information-theoretically. The proof of these properties is given in~\Cref{sec:proofs_properties}.
\begin{lemma}\label{lemma:cotinfo_props}
  % \hphantom{~}
  Let $\calH \subset (\calY \times \calZ)^{\calX}$ be a CoT hypothesis class. Then the CoT information  $\cotinfo(\epsilon; \calH)$ satisfies the following properties:
  \begin{enumerate}
    \item  $\cotinfo(\epsilon; \calH) \geq \epsilon$.
    \item $\cotinfo(\epsilon; \calH)$ is monotonically increasing in $\epsilon$.
    \item $\cotinfo(\epsilon; \calH)$ is monotonically decreasing in $\calH$ (under the subset relation).
  \end{enumerate}
\end{lemma}

Before proceeding with bounding sample complexity in terms of CoT information, we note how the measure behaves in extreme boundary conditions.  First, let us consider an example where the CoT annotations are entirely independent of the end-to-end behavior. In particular, consider a CoT hypothesis class with a product structure $\calH = \calF^{\CoT} \times \calF^{\ete}$, where $\calF^{\CoT} \subset \calZ^{\calX}, \calF^{\ete} \subset \calY^{\calX}$. In this case, we would expect no statistical advantage from observing the CoT---this is captured by the CoT information measure, which coincides with the ``end-to-end information'' in this case. At the other extreme, consider the case where the CoT from \textit{any} single example reveals the entire target function.
For example, let $\calF \subset \calY^{\calX}$ and consider the CoT hypothesis class $\calH = \sset{h_f: x \mapsto (f, f(x)) \,:\, f \in \calF}$. In this case, $\cotinfo(\epsilon; \calH) = \infty$, which corresponds to the fact that a single example is sufficient to attain zero error. 

Finally, consider the problem of learning a regular language with CoT supervision.
Here, we take the output $y$ to indicate whether or not the string $\x$ is in the language, and we let 
$\z$ be the sequence of states visited in a DFA representing the language as it processes $\x$. In~\Cref{sec:simulations}, we study this example in the context of our theory via empirical simulation. Additionally,~\Cref{sec:examples_appendix} provides a more detailed discussion of the above illustrative examples.

\subsection{Improved sample complexity via CoT information}\label{ssec:warmup}

To illustrate the main ideas and intuitions underpinning this paper's results, we next prove a sample complexity bound for CoT-supervised learning with \textit{finite} hypothesis classes in the \textit{realizable} setting. While other proofs are deferred to the appendix for clarity and brevity, this particular result is proven here due to its simplicity and pedagogical value.

The learning rule we consider is \textit{chain-of-thought consistency}, $\CoTCons(S; \calH)$: given a sample $S$, the learner returns any hypothesis in $\calH$ which is consistent with the sample $S = \sset{(x_i, y_i, z_i)}_{i \in [m]}$ in terms of both outputs and the chain-of-thought.

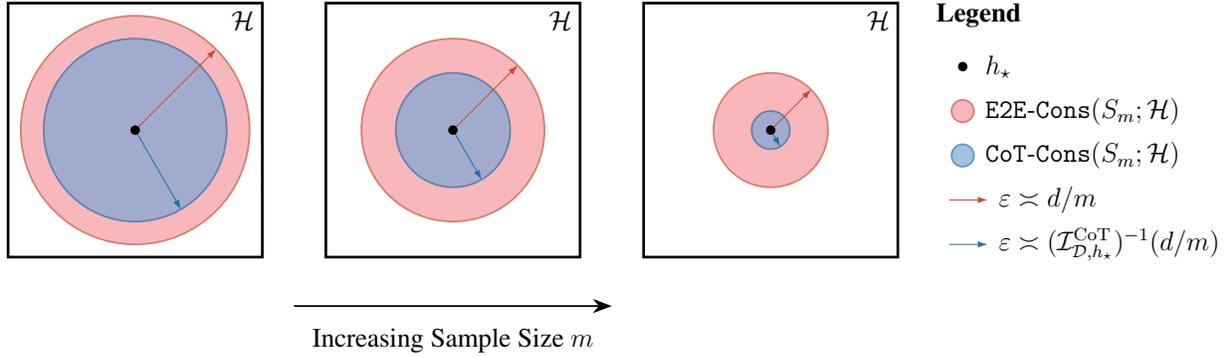
\begin{figure}[t]
    \centering
    \resizebox{\linewidth}{!}{
    \begin{tikzpicture}[
    scale=1.0, transform shape
  ]
    % \tikzset should be defined in the preamble or here
    \tikzset{
      hypothesis/.style={draw, very thick, minimum size=4cm}, % Square for H
      hstar/.style={fill, circle, inner sep=1.5pt},           % Dot for h*
      e2econs/.style={draw, thick, Red!80!black, fill=OrangeRed!50, opacity=0.7}, % Style for E2E circle
      cotcons/.style={draw, thick, RoyalBlue!80!black, fill=RoyalBlue!50, opacity=0.7}, % Style for CoT circle
      % Styles for radius lines (NEW/REVISED)
      e2e_radiusline/.style={
          draw=Red!80!black, % Match E2E border color
          % dashed,
          thin,
          -{Latex[length=1.5mm, width=1mm]} % Arrow tip
          },
      cot_radiusline/.style={
          draw=RoyalBlue!80!black, % Match CoT border color
          % dashed,
          % thin,
          -{Latex[length=1.5mm, width=1mm]} % Arrow tip
          }
    }
    
    % Definitions (same as before, unless you want to adjust spacing)
    \def\Hsize{2}
    \def\BaseRadius{1.8}
    \def\PanelShift{5cm}
    \def\LegendXShift{1cm}
    \def\LegendYShift{1cm}
    \def\LegendItemSep{0.7cm}
    \def\LegendCircleRadius{5pt}
    \def\LegendTextOffset{6pt}
    \def\RadiusAngleEtwoE{45} % Keep angles defined previously
    \def\RadiusAngleCoT{-60}
    \def\LegendLineLength{0.6cm} % Define length for legend line representation

  % --- Panel 1: m=1 ---
  \begin{scope}[shift={(0,0)}]
    \node[hypothesis, label={[anchor=north east]north east:$\calH$}] (H1) at (0,0) {};
    \def\E2ERadiusOne{\BaseRadius/1}
    \def\CoTRadiusOne{\BaseRadius/1.25}
    % Draw filled circles
    \path[e2econs] (0,0) circle (\E2ERadiusOne);
    \path[cotcons] (0,0) circle (\CoTRadiusOne);
    % Add radius lines (using new styles)
    \draw[e2e_radiusline] (0,0) -- ($(0,0) + (\RadiusAngleEtwoE:\E2ERadiusOne)$); % E2E radius line
    \draw[cot_radiusline] (0,0) -- ($(0,0) + (\RadiusAngleCoT:\CoTRadiusOne)$); % CoT radius line
    % Draw reference hypothesis h* on top
    \node[hstar] (hstar1) at (0,0) {};
  \end{scope}

  % --- Panel 2: m=2 ---
  \begin{scope}[shift={(\PanelShift,0)}]
    \node[hypothesis, label={[anchor=north east]north east:$\calH$}] (H2) at (0,0) {};
    \def\E2ERadiusTwo{\BaseRadius/1.25}
    \def\CoTRadiusTwo{\BaseRadius/2}
    % Draw filled circles
    \path[e2econs] (0,0) circle (\E2ERadiusTwo);
    \path[cotcons] (0,0) circle (\CoTRadiusTwo);
    % Add radius lines (using new styles)
    \draw[e2e_radiusline] (0,0) -- ($(0,0) + (\RadiusAngleEtwoE:\E2ERadiusTwo)$); % E2E radius line
    \draw[cot_radiusline] (0,0) -- ($(0,0) + (\RadiusAngleCoT:\CoTRadiusTwo)$); % CoT radius line
    % Draw reference hypothesis h* on top
    \node[hstar] (hstar2) at (0,0) {};
  \end{scope}

  % --- Panel 3: m=4 ---
  \begin{scope}[shift={(2*\PanelShift,0)}]
    \node[hypothesis, label={[anchor=north east]north east:$\calH$}] (H3) at (0,0) {};
    \def\E2ERadiusFour{\BaseRadius/2}
    \def\CoTRadiusFour{\BaseRadius/6}
    % Draw filled circles
    \path[e2econs] (0,0) circle (\E2ERadiusFour);
    \path[cotcons] (0,0) circle (\CoTRadiusFour);
     % Add radius lines (using new styles)
    \draw[e2e_radiusline] (0,0) -- ($(0,0) + (\RadiusAngleEtwoE:\E2ERadiusFour)$); % E2E radius line
    \draw[cot_radiusline] (0,0) -- ($(0,0) + (\RadiusAngleCoT:\CoTRadiusFour)$); % CoT radius line
    % Draw reference hypothesis h* on top
    \node[hstar] (hstar3) at (0,0) {};
  \end{scope}

  % --- Legend (Added radius line entries) ---
  \coordinate (legend_anchor) at ($(H3.east) + (\LegendXShift, \LegendYShift)$);

  % hstar Legend Item
  \coordinate (hstar_item_pos) at (legend_anchor);
  \node[hstar] at (hstar_item_pos) {};
  \node[right=\LegendTextOffset of hstar_item_pos, anchor=west, black] (hstar_text) {$\hstar$};

  % E2E Legend Item
  \coordinate (e2e_item_pos) at ([yshift=-\LegendItemSep]hstar_item_pos);
  \path[e2econs] (e2e_item_pos) circle (\LegendCircleRadius);
  \node[right=\LegendTextOffset of e2e_item_pos, anchor=west, black] (e2e_text) {$\EtECons(S_m; \calH)$};

  % CoT Legend Item
  \coordinate (cot_item_pos) at ([yshift=-\LegendItemSep]e2e_item_pos);
  \path[cotcons] (cot_item_pos) circle (\LegendCircleRadius);
  \node[right=\LegendTextOffset of cot_item_pos, anchor=west, black] (cot_text) {$\CoTCons(S_m; \calH)$};

  % E2E Radius Legend Item
  \coordinate (e2e_radius_item_pos) at ([yshift=-\LegendItemSep, xshift=-6pt]cot_item_pos);
  % Draw representation (short horizontal line with style)
  \draw[e2e_radiusline] (e2e_radius_item_pos) -- ++(\LegendLineLength, 0);
  % Add text label
  % \node[right=18pt of e2e_radius_item_pos, anchor=west, black] (e2e_radius_text) {$r_{\EtECons} \asymp d/m$};
  \node[right=18pt of e2e_radius_item_pos, anchor=west, black] (e2e_radius_text) {$\epsilon \asymp d/m$};

  % CoT Radius Legend Item
  \coordinate (cot_radius_item_pos) at ([yshift=-\LegendItemSep]e2e_radius_item_pos);
  % Draw representation
  \draw[cot_radiusline] (cot_radius_item_pos) -- ++(\LegendLineLength, 0);
  % Add text label
  \node[right=18pt of cot_radius_item_pos, anchor=west, black] (cot_radius_text) {$\epsilon \asymp (\cotinfo)^{-1}(d/m)$};
  % \node[right=18pt of cot_radius_item_pos, anchor=west, black] (cot_radius_text) {$r_{\CoTCons} \asymp (\cotinfo_{\calD})^{-1}(d/m)$};

  % Legend Title (Adjusted vertical position slightly)
  \node[anchor=south] at ([yshift=0.15cm, xshift=0.25cm]hstar_text.north -| hstar_item_pos) {\textbf{Legend}};

  % --- Arrow indicating increasing sample size ---
  \draw [-{Stealth[length=3mm]}, thick] ($(H1.south) + (\PanelShift*0.5, -0.75cm)$) -- ($(H3.south) + (-\PanelShift*0.5, -0.75cm)$)
        node[midway, below=2mm] {Increasing Sample Size $m$};

\end{tikzpicture}
    }
    \caption{
    Illustration of the statistical advantage of CoT supervision in terms of the geometry of the CoT consistency rule with respect to \textit{end-to-end} error. CoT supervision enables the construction of a tighter consistency set, when the CoT is informative (i.e.,, $\cotinfo(\epsilon; \calH) > \epsilon$), which leads to smaller end-to-end error and more sample-efficient learning.
    }\label{fig:cotcons_diagram}
\end{figure}

%\aanote{should we trim down the statement of the theorem? shorten the proof? E.g., do we need the ``that is, ...'' part?}

\begin{result}[Learning with Chain-of-Thought Supervision]\label{result:cotinfo_ete_learning}
    Let $\calH \subset (\calY \times \calZ)^{\calX}$ be a finite CoT hypothesis class. For any distribution $\calD$ over $\calX \times \calY \times \calZ$ realized by some $\hstar \in \calH$, the CoT consistency learning rule has a sample complexity of
    %with respect to the end-to-end error of
    \begin{equation*}
      m(\epsilon, \delta) = \frac{\log \abs{\calH} + \log(1 / \delta)}{\cotinfo(\epsilon; \calH)}.
    \end{equation*}
    That is, for any $m \geq m(\epsilon, \delta)$, with probability at least $1 - \delta$ over $S \sim \calD^m$, 
    any hypothesis $h$ that is CoT consistent on $S$ will have end-to-end risk satisfying $\Lete{\calD}(h) \leq \epsilon$.

    % That is, for $m \geq m(\epsilon, \delta)$, we have that with probability at least $1 - \delta$ over $S = \sset{x_1, \ldots, x_m} \simiid \calD$,
    % $\Lete{\calD}(h) \leq \epsilon$ for all $h \in \CoTCons(S; \calH)$.
  \end{result}
  \vskip-15pt
  \begin{proof}
    We aim to bound the probability of the ``bad event''
    \[\sset{\exists h \in \calH : \Lete{\calD}(h) > \epsilon,\, \empiricalcotrisk{S}(h) = 0}\]
    over the draw of $(x_1, \ldots, x_m) \simiid \calD$. We highlight that the training loss is the empirical \textit{CoT} risk, $\empiricalcotrisk{S}(h)$, whereas the test metric is the \textit{end-to-end} risk $\Lete{\calD}(h)$.
    
    Fix any $h \in \calH$ with end-to-end error larger than $\epsilon$, $\Lete{\calD}(h) = \pprobunder{x \sim \calD}{\hete{h}(x) \neq \ete(\hstar)(x)} > \epsilon$ (i.e., $h \in \Deltaete_{\calD}(\epsilon; \calH, \hstar)$). We bound the probability that $h$ is CoT consistent on $S$, $h \in \CoTCons(S; \calH) = \sset{h \in \calH : \empiricalcotrisk{S}(h) = 0}$, as follows
    \begin{align*}
      \probunder{S \sim \calD^{\otimes m}}{h \in \CoTCons(S; \calH)} &= \probunder{S \sim \calD^{\otimes m}}{\forall i, \hcot{h}(x_i) = \hcot{\hstar}(x_i), \hete{h}(x_i) = \hete(\hstar)(x_i)} \\
      &= \probunder{x \sim \calD}{\hcot{h}(x) = \hcot{\hstar}(x),\, \hete{h}(x) = \hete{\hstar}(x)}^m \\
      &\stepa{=} \paren{\exp\paren{- \relcotinfo(\hstar, h)}}^m \\
      &\stepb{\leq} \exp\paren{-m \cdot \cotinfo(\epsilon; \calH)},
    \end{align*}
    where step (a) is by the definition of the relative CoT information between a pair of hypotheses, and step (b) is by the definition of $\cotinfo(\epsilon; \calH)$ and the fact that $h \in \Deltaete_{\calD}(\epsilon; \calH, \hstar)$.

    Choosing $m = \frac{\log \aabs{\calH} + \log (1/\delta)}{\cotinfo(\epsilon; \calH)}$ implies that for each hypothesis $h \in \Deltaete_{\calD}(\epsilon; \calH, \hstar)$ with end-to-end error larger than $\epsilon$, the probability that it is in the CoT consistency set is bounded by
    \[\probunder{S \sim \calD^{\otimes m}}{h \in \CoTCons(S; \calH)} \leq \frac{\delta}{\abs{\calH}}.\]
    Applying a union bound over $\calH$ yields
    \[\probunder{S \sim \calD^{\otimes m}}{\exists h \in \calH : \Lete{\calD}(h) > \epsilon,\, \empiricalcotrisk{S}(h) = 0} \leq \delta\]
    to complete the proof.
\end{proof}

This result demonstrates that, for CoT learning, the $\epsilon$-dependence of the sample complexity is $\bigO(1 / \cotinfo(\epsilon; \calH))$, contrasting with the typical rate of $\bigO(1 / \epsilon)$. Intuitively, the ratio $\cotinfo(\epsilon; \calH) / \epsilon \geq 1$ quantifies the relative value of a CoT training example compared with an end-to-end training example.

\par

\section{Guarantees for CoT-Supervised Learning: Upper Bounds}\label{sec:upper_bounds}
%\aanote{think about section title}
This section extends our exploration of statistical upper bounds to infinite hypothesis classes and the agnostic learning setting, thereby further elucidating the statistical advantage of CoT supervision.

\subsection{The realizable setting: Extension to infinite classes}

\Cref{result:cotinfo_ete_learning} established a sample complexity bound determined by two key factors: the term $1 / \cotinfo(\epsilon; \calH)$, which captures the information per CoT-supervised sample, and the log-cardinality of the class, $\log \abs{\calH}$, which reflects its size or dimension. We now extend this result to infinite classes, replacing the log-cardinality term with the VC dimension of the CoT loss class. As before, the upper bound is achieved by the CoT consistency learning rule.

%\aanote{Should we continue to call these ``result''s or should we call them theorems or something else?}

\begin{result}[Learning infinite classes under CoT supervision]\label{result:cotcons_cotinfo_infH}
    Let $\calH \subset (\calZ \times \calY)^{\calX}$ be a CoT hypothesis class. For any distribution $\calD$ over $\calX \times \calY \times \calZ$ realized by some $\hstar \in \calH$, the CoT consistency learning rule has a sample complexity of
    % \[m(\epsilon, \delta) = \calO \paren{\max\paren{\frac{1}{\cotinfo(\epsilon; \calH)}, 1} \cdot \paren{ \VC(\calLCoT(\calH)) \cdot \log \max\paren{\frac{1}{\cotinfo(\epsilon; \calH)}, 1} + \log(1/\delta)}}.\]
    \[m(\epsilon, \delta) =  \bigO\paren{\Bigparen{\frac{1}{\cotinfo(\epsilon; \calH, \hstar)} + 1} \Bigparen{ \VC(\calLCoT(\calH)) \cdot \log \Bigparen{\frac{1}{\cotinfo(\epsilon; \calH)} + 1} + \log(1/\delta)}}.\]
    % \[m(\epsilon, \delta) =  \bigOtilde\paren{\frac{\VC(\calLCoT(\calH)) + \log(1/\delta)}{\cotinfo(\epsilon; \calH)}}\] % this version puts things under bigOtilde
    That is, for any $m \geq m(\epsilon, \delta)$, with probability at least $1 - \delta$ over $S \sim \calD^m$, 
    any hypothesis $h$ that is CoT consistent on $S$ will have end-to-end risk satisfying $\Lete{\calD}(h) \leq \epsilon$.
    %\[\forall h \in \CoTCons(S; \calH), \ \text{we have} \ \Lete{\calD}(h) \leq \epsilon.\]
\end{result}
\vskip10pt

The proof is provided in~\Cref{ssec:proof:result:cotcons_cotinfo_infH}. The result follows from a lemma that relates the CoT risk of any \textit{proper} CoT learning rule (i.e., one that returns a predictor in the hypothesis class) to its performance with respect to the end-to-end error.

We now contrast our result with the alternative approach in~\Cref{ssec:key_idea_control_cot_error}, which bounds the CoT risk directly (see the second row of~\Cref{table:comparison_table}), as in~\citet{joshi2025theorylearningautoregressivechain}.  Both analyses use the VC dimension of the CoT loss class to quantify hypothesis class complexity. However, our approach achieves sharper learning rates by improving the $\epsilon$-dependence via CoT information, which reflects the greater information content of each annotated sample.  In particular, the work of~\citet{joshi2025theorylearningautoregressivechain} establishes bounds on the VC dimension of the CoT loss class for hypothesis classes with a specific autoregressive structure---these bounds can be combined with~\Cref{result:cotcons_cotinfo_infH} to establish improved sample complexity bounds for these autoregressive hypothesis classes.  

%\aanote{In discussion, we can mention that we can apply the result of~\citet{joshi2025theorylearningautoregressivechain} to bound the VC dimension of the CoT loss class.}

%\aanote{We should emphasize that, although these results extend our analysis to infinite hypothesis classes, we don't want to imply the the VC dimension of the CoT loss class is the only or best measure of how the sample complexity scales with the ``size'' or complexity of the hypothesis class. In particular, recent results on multi-class learning suggests that perhaps its not. Our main focus is on the $\epsilon$-dependence, and our lower bound results suggest that the CoTInfo correctly characterizes the $\epsilon$-dependence.}

% \aanote{Should we incorporate any mention of the challenges of proving this via a symmetrization-type argument?}

\subsection{The agnostic setting}

The previous results assume that the data distribution $\calD$ over $\calX \times \calY \times \calZ$ is \textit{realizable} by the CoT hypothesis class $\calH$. Such an assumption can be stringent, particularly in the presence of noise. This section, therefore, addresses the \textit{agnostic} setting, where no restriction is made on the distribution; the goal, instead, is to compete with the best hypothesis in the class $\calH$ in terms of \textit{end-to-end risk}.

In the agnostic setting, a natural learning rule is \textit{CoT empirical risk minimization}, which selects a hypothesis that minimizes the \textit{empirical CoT risk}: 
$\CoTERM(S; \calH) = \argmin_{h \in \calH} \empiricalcotrisk{S}(h)$.
% $\CoTERM(S; \calH) = \argmin_{h \in \calH} \frac{1}{m} \sum_{i=1}^{m} \IInd{h(x_i) \neq (y_i, z_i)}$.

Recall that CoT supervision never hurts in the realizable setting since $\cotinfo(\epsilon; \calH) \geq \epsilon$ for any hypothesis class. The picture is more complicated in the agnostic setting. In particular, CoT supervision can be harmful or distracting, and discarding the CoT annotation and learning from only the input-output examples can be preferable, as the following example shows. The issue arises when the CoT hypothesis class $\calH$ is not aligned with the data distribution, especially when $\calH$ can fit the end-to-end behavior but not the CoT behavior.

\begin{example*}%[In the agnostic setting, CoT supervision may be actively harmful (for CoT-ERM)]\label{example:agnostic_challenges}
    Consider a CoT hypothesis class $\calH : \calX \to \calY \times \calZ$ and suppose $\calD$ is a distribution over $\calX \times \calY \times \calZ$ for which the output component is realizable by $\calH$ but the CoT component is not realizable. In particular, it is easy to construct examples for which $\inf_{h \in \calH} \Lete{\calD}(h) = 0$ while $\inf_{h \in \calH} \Lcot{\calD}(h) = 1$. Clearly, in such cases, the CoT-ERM learning rule provides no guarantees whatsoever since $\CoTERM(S; \calH) = \calH$ for any $S$ supported by $\calD$. In contrast, E2E-ERM enjoys the standard PAC learning guarantees, with a sample complexity $\calO\left(\sfrac{1}{\epsilon} \cdot \VC(\calLete(\calH))\right)$.
\end{example*}

%\aanote{Can add comment that this issue arises for CoT-ERM, but can perhaps be addressed by a more sophisticated learning rule with some adaptivity. E.g., try to fit with respect to E2E error and CoT error separately on some dataset, then evaluate each model on a separate independent dataset, and return the predictor with the better error. Or fit a linear combination of the two errors with some tuned weighting.}

Thus, CoT supervised learning in the agnostic setting requires a different notion of CoT information, which captures how well-aligned the data distribution is to the hypothesis class, and whether fitting the CoT aligns with fitting the end-to-end behavior. This is defined in the following result, which extends our results to the agnostic setting.

\begin{result}[Agnostic learning under CoT supervision]\label{result:coterm_agnostic}
    Let $\calH \subset (\calY \times \calZ)^{\calX}$ be a CoT hypothesis class. For any distribution $\calD$ over $\calX \times \calY \times \calZ$, the CoT-ERM learning rule has a sample complexity of
    % \[m(\epsilon, \delta) = \calO \paren{\max\paren{\frac{1}{\cotinfo(\epsilon; \calH)}, 1} \cdot \paren{ \VC(\calLCoT(\calH)) \cdot \log \max\paren{\frac{1}{\cotinfo(\epsilon; \calH)}, 1} + \log(1/\delta)}}.\]
    \[m(\epsilon, \delta) =  \calO \paren{\frac{\VC(\calLCoT(\calH)) + \log(1/\delta)}{\cotinfoag(\epsilon; \calH)^2}},\]
    where $\cotinfoag(\epsilon; \calH)$, the agnostic CoT information, is defined via excess risks as
    \[\cotinfoag(\epsilon; \calH) := \inf\set{\Lcot{\calD}(h) - \Lstarcot: h \in \calH, \Lete{\calD}(h) - \Lstarete \geq \epsilon},\]
    where $\Lstarcot := \inf_{h \in \calH} \Lcot{\calD}(h)$ and $\Lstarete := \inf_{h \in \calH} \Lete{\calD}(h)$.
    That is, for any $m \geq m(\epsilon, \delta)$, with probability at least $1 - \delta$ over the draw of $S \sim \calD^m$, the excess end-to-end risk is bounded as $\Lete{\calD}(h) \leq \Lstarete + \epsilon$ 
    for any minimizer of the CoT empirical risk.
%    \[\forall h \in \CoTERM(S; \calH), \ \text{we have} \ \Lete{\calD}(h) \leq \Lstarete + \epsilon.\]
\end{result}
\vskip10pt
%\aanote{$= \bigO$ or $\in \bigO$?}

The proof is presented in~\Cref{ssec:proof:result:coterm_agnostic}. Note that, unlike in the realizable case, we do not necessarily have the lower bound $\cotinfoag(\epsilon; \calH) \geq \epsilon$. For instance, in the motivating example above, $\cotinfoag(\epsilon; \calH) = 0$. However, CoT supervision yields an advantage whenever $\Lcot{\calD}(h) - \Lstarcot > \Lete{\calD}(h) - \Lstarete$.

%\aawarning{TODO --- bound $\cotinfoag(\epsilon; \calH)$ in terms of $\cotinfo$ when $\calD$ is noisy version of realizable distribution.}

\section{Information Theoretic Lower Bounds for CoT Supervised Learning}\label{sec:lower_bounds}

%In the previous section, we established statistical \textit{upper bounds} on the sample complexity of learning under CoT supervision, characterized in terms of the CoT information measures. 
This section establishes information-theoretic lower bounds on sample complexity, further validating the CoT information $\cotinfo(\epsilon; \calH)$ as a fundamental measure of statistical complexity for learning with CoT supervision.
In general, the statistical complexity of a learning problem depends on several parameters, including the size or complexity of the hypothesis class (e.g., $\VC(\calH)$ in binary classification) and the error parameter (e.g., $1/\epsilon$ or $1/\epsilon^2$ for the realizable and agnostic settings, respectively). Different types of lower bounds scale accordingly with one or both of these factors. Our main focus in this work is on the dependence of the sample complexity on the error parameter, which corresponds to the amount of information encoded in the CoT supervision for discriminating between hypotheses with different end-to-end behavior.

We begin with a lower bound demonstrating that CoT information $\cotinfo(\epsilon; \calH)$ characterizes the $\epsilon$ dependence of sample complexity. The essence of the result is to lower bound the minimum number of samples needed to distinguish a pair of hypotheses with a given end-to-end disagreement, reducing the learning problem to a binary hypothesis testing problem~\citep{lecamConvergenceEstimatesDimensionality1973}, and relating the total variation distance between distributions over $\calX \times \calY \times \calZ$ induced by a pair of hypotheses to the relative CoT information between them. 

\begin{result}[Lower bound via CoT information]\label{result:cotinfo_lowerbound_twopoint}
  Let $\calH \subset (\calY \times \calZ)^{\calX}$ be a CoT hypothesis class and let $\calD$ be a distribution on $\calX$. Let $\x_1, \ldots, \x_m \sim \calD$ be an i.i.d sample from $\calD$. For any $\hstar \in \calH$ and $\epsilon > 0$, if the sample size satisfies
  \[m < \frac{\log (1/\delta)}{\cotinfo(\epsilon; \calH)}\]
  then with probability at least $\delta$, there exists $h \in \calH$ with end-to-end error at least $\epsilon$ which is indistinguishable from $\hstar$ on the sample. Moreover, the expected error of any algorithm $\calA$ satisfies
    \begin{align*}
        \sup_{\hstar \in \calH} \expectunder{S \sim P_{\hstar}^{\otimes m}}{\Lete{\calD, \hstar}(\calA(S))} 
        %&\geq \frac{1}{2} \sup_{h_1, h_2 \in \calH} \probunder{x \sim \calD}{\hete{h_1}(x) \neq \hete{h_2}(x)} \cdot \exp(-m \cdot \relcotinfo(h_1, h_2)) \\
        &\geq \frac{1}{2} \sup_{\substack{\hstar \in \calH \\ \epsilon > 0}} \epsilon \cdot \exp(-m \cdot \cotinfo(\epsilon; \calH)).
    \end{align*}
\end{result}

% \begin{result}[Lower Bound]{}
%     Let $\calA: (\calX \times \calY \times \calZ)^* \to \calH$ be any learning algorithm that maps a dataset $S^m = \sset{(x_i, y_i, z_i)}_{i=1}^{m}$ to a predictor $\hhat$. Suppose there exists $h_1, h_2 \in \calH$ such that $\probunder{x \sim \calD}{\hete{h_1}(x) \neq \hete{h_2}(x)} \geq 2 \epsilon$. Assume that $S^m \sim (\calD \otimes \delta_{(y,z) = \hstar(x)})^{\otimes m}$ the sample size $m$ is upper bounded as 
%     \[m < \frac{\log(\frac{1}{2\delta})}{\relcotinfo(h_1, h_2)}.\]
%     Then, we must have
%     \[\inf_{\hstar \in \calH} \probunder{S^m \sim \hstar}{\Lete{\calD}(\calA(S^m); \hstar) \geq \epsilon} > \delta.\]
% \end{result}

% \begin{result}[Lower Bound on Expected End-to-End Error (via Le Cam)]{}
%     The expected error of any CoT learning algorithm $\calA$ is lower-bounded as,
%     \begin{align*}
%         &\sup_{\hstar \in \calH} \expectunder{S^m \sim \hstar}{\Lete{\calD}(\calA(S^m); \hstar)} \\
%         &\geq \frac{1}{2} \sup_{h_1, h_2 \in \calH} \probunder{x \sim \calD}{\hete{h_1}(x) \neq \hete{h_2}(x)} \cdot \exp(-m \cdot \relcotinfo(h_1, h_2)) \\
%         &\geq \frac{1}{2} \sup_{\substack{\hstar \in \calH \\ \epsilon > 0}} \epsilon \cdot \exp(-m \cdot \cotinfo(\epsilon; \calH)).
%     \end{align*}
% \end{result}

This result validates the CoT information as characterizing the $\epsilon$-dependence of the rate. However, a weakness of two-point methods is that they do not scale with the size of the hypothesis space. The following result addresses this by reducing the learning problem to that of testing \textit{multiple} hypotheses, 
using a packing of the hypothesis space with respect to the \textit{end-to-end error}. We then use Fano's inequality to lower bound the probability of error in terms of a mutual information,
% between the hypothesis and the CoT sample, 
and relate this mutual information to the CoT information. To apply Fano's method in this way, we extend the framework to allow the observed $\z$ to be a stochastic function of the 
hypothesis CoT.% $\hcot{h}(\x)$.

\begin{result}[Lower bound via Fano's method]\label{result:lower_bound_fano}
    Let $\calH \subset (\calY \times \calZ)^{\calX}$ be a CoT  hypothesis class and let $\calD$ be a distribution over $\calX$. Suppose that $x_1, \ldots, x_m \sim \calD$. Let $Q \in \calP(\calY \times \calZ \ggiven \calY \times \calZ)$ be a noisy channel from $h(x) = (y, z)$ to observations $\bar{y}, \bar{z}$. Let $C_Q = \max_{a,b} \KL{Q(\cdot \ggiven a)}{Q(\cdot \ggiven b)}$ be the capacity factor of the channel. The learner observes the noisy sample $S = \sset{(x_i, \bar{y}_i, \bar{z}_i)}_{i=1}^{m}$. Define the pseudo-metric $\dete(h_1, h_2) = \pprobunder{x}{\hete{h_1}(x) \neq \hete{h_2}(x)}$, and let $M(\epsilon; \calH, \dete)$ be the $\epsilon$-packing number of $\calH$ with respect to this pseudo-metric. 
    % Then, for any algorithm $\calA$ observing the CoT supervised sample $S$ of size $m$, the probability of having large end-to-end error is lower bounded as 
    % \begin{equation*}
    %     \sup_{\hstar \in \calH} \probunder{S \sim P_{\hstar}^{\otimes m}}{\Lete{\calD, \hstar}(\calA(S)) \geq \frac{\epsilon}{2}} \geq 1 - \frac{m \cdot C_Q \cdot  \sup_\pi \expectunder{h_1, h_2 \sim \pi}{\relcotinfo(h_1, h_2)} + \log 2}{\log M(\epsilon; \calH, \dete)}.
    % \end{equation*}
    Then, for any algorithm $\calA$ observing the CoT supervised sample $S$ of size $m$, we have that 
    \[m \leq \frac{\log M(\calH, \dete, \epsilon)}{2 \cdot \paren{C_Q \cdot  \sup_\pi \expectunder{h_1, h_2 \sim \pi}{\relcotinfo(h_1, h_2)} + \log 2}},\]
    implies large error for some $\hstar \in \calH$ with high probability, i.e.
    \begin{equation*}
        \sup_{\hstar \in \calH} \probunder{S \sim P_{\hstar}^{\otimes m}}{\Lete{\calD, \hstar}(\calA(S)) \geq \frac{\epsilon}{2}} \geq 1/2.
    \end{equation*}
\end{result}
%\aafatal{Is this statement of the result good? any edits to make?}

Here, $C_Q$ is a bound on the capacity of the channel that adds noise to the chain-of-thought. This lower bound relates the probability of large error to the CoT information measure, as with the previous result, but also scales with the size of the hypothesis space, as measured by its packing number. Additionally, the result also models noise in the learning process by observing the CoT label outputs through a noisy channel, which is important in the context of CoT learning, where the CoT labels are often manually created by human annotators in an error-prone process.
%\aawarning{Include some discussion of $C_Q$ here? Currently in appendix (we have limited space heres).}
The proofs of both lower bound results are presented in~\Cref{sec:proofs_lower_bounds}, together with further discussion.

\section{Simulations}\label{sec:simulations}

This section presents numerical simulations exploring the CoT information measure for simple CoT hypothesis classes and its ability to predict sample complexity gains from CoT-supervised learning.

\begin{figure*}[ht!]
    \centering
    \begin{subfigure}[t]{0.475\textwidth}
        \centering
        \includegraphics[width=\linewidth]{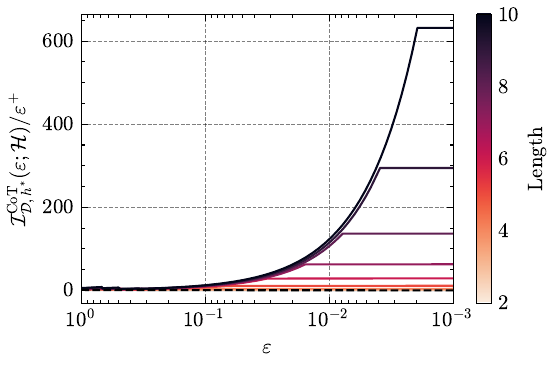}
        \captionsetup{width=.9\linewidth}
        \caption{Ratio of CoT Information to ``clipped'' $\epsilon$, varying input length. 
        % For more complex input distributions, the CoT supervision becomes more valuable.
        }\label{appendixfig:DFA_CoTInfoRatio_vs_Epsilon_by_length}
    \end{subfigure}
    \hfill
    \begin{subfigure}[t]{0.475\textwidth}
        \centering
        \includegraphics[width=\linewidth]{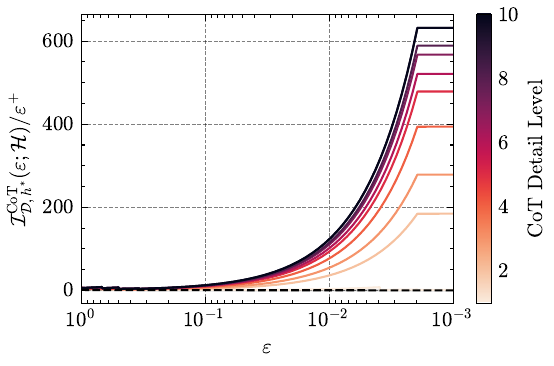}
        \captionsetup{width=.9\linewidth}
        \caption{Ratio of CoT Information to ``clipped'' $\epsilon$, varying level of detail in CoT.}\label{appendixfig:DFA_CoTInfoRatio_vs_Epsilon_by_detail}
        % \label{appendixfig:CoTInfoRatio_vs_Epsilon_by_length}
    \end{subfigure}

    \begin{subfigure}[t]{0.475\textwidth}
        \centering
        \includegraphics[width=\linewidth]{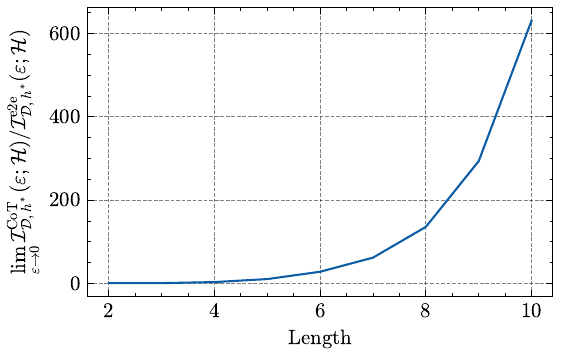}
        \captionsetup{width=.9\linewidth}
        \caption{Ratio of CoT Information to $\epsilon^+$, by input length}\label{appendixfig:DFA_lim_CoTInfo_Epsilon_Ratio_vs_Length}
    \end{subfigure}
    \hfill
    \begin{subfigure}[t]{0.475\textwidth}
        \centering
        \includegraphics[width=\linewidth]{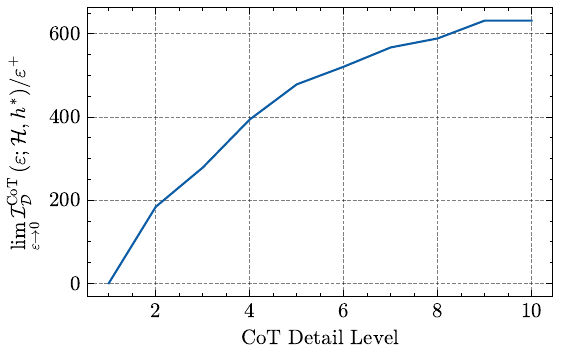}
        \captionsetup{width=.9\linewidth}
        \caption{Ratio of CoT Information to $\epsilon^+$, by level of detail in CoT.}\label{appendixfig:DFA_lim_CoTInfo_Epsilon_Ratio_vs_DetailLevel}
    \end{subfigure}

    \begin{subfigure}[t]{0.475\textwidth}
        \centering
        \includegraphics[width=\linewidth]{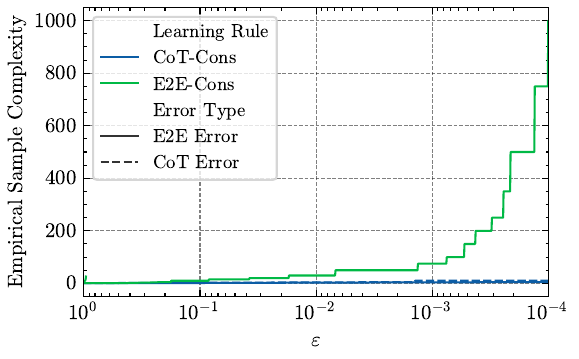}
        \captionsetup{width=.9\linewidth}
        \caption{Empirical Sample Complexity.}\label{appendixfig:DFA_empirical_sample_complexity}
    \end{subfigure}
    \hfill
    \begin{subfigure}[t]{0.475\textwidth}
        \centering
        \includegraphics[width=\linewidth]{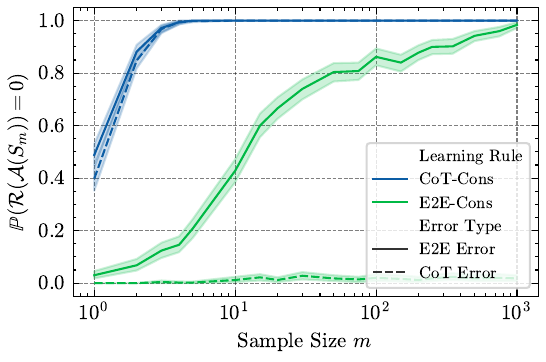}
        \captionsetup{width=.9\linewidth}
        \caption{Empirical probability of each learning rule returning a predictor with zero error.}\label{appendixfig:DFA_empirical_learning_curves}
    \end{subfigure}
    \caption{Numerical experiments for deterministic finite automata CoT hypothesis class.}\label{appendixfig:DFA_simulations}
    \vskip5pt
\end{figure*}

\begin{figure*}[ht!]
    \centering
    \begin{subfigure}[t]{0.45\textwidth}
        \includegraphics[width=\linewidth]{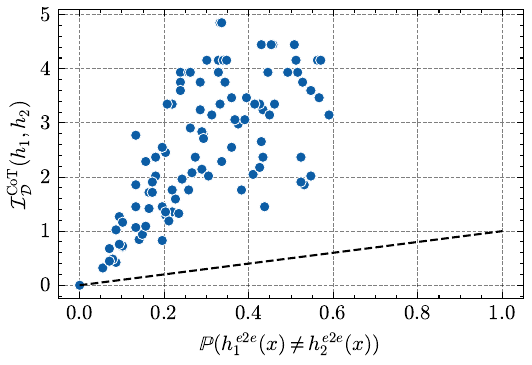}
        \caption{Relative CoT information between pairs of hypotheses plotted against their end-to-end disagreement.}\label{appendixfig:LinThresh:CoTInfo_vs_E2EDiff}
    \end{subfigure}
    \hfill
    \begin{subfigure}[t]{0.45\textwidth}
        \includegraphics[width=\linewidth]{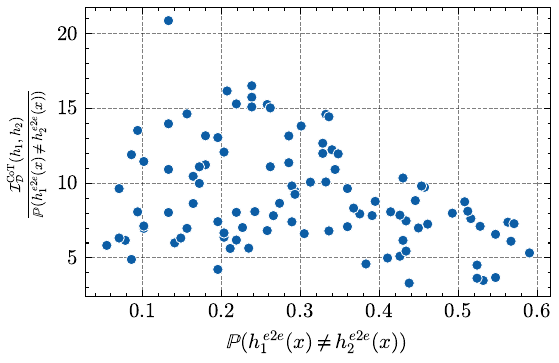}
        \caption{The ratio of pairwise CoT information to end-to-end disagreement, plotted against the end-to-end disagreement.}\label{appendixfig:LinThresh_Relative_CoTInfo_vs_E2EDiff}
    \end{subfigure}

    \begin{subfigure}[t]{0.45\textwidth}
        \includegraphics[width=\linewidth]{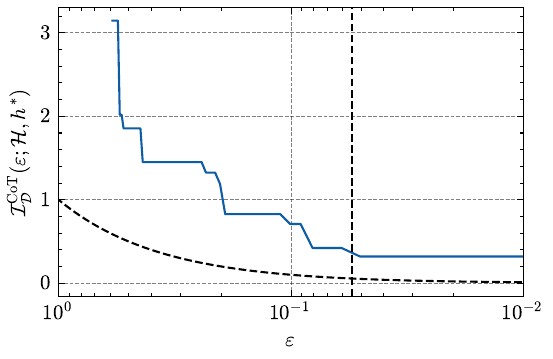}
        \caption{CoT information $\cotinfo(\epsilon; \calH)$ as a function of $\epsilon$.}\label{appendixfig:LinThresh:CoTInfo_vs_Epsilon}
    \end{subfigure}    
    % \begin{subfigure}[t]{0.45\textwidth}
    %     \includegraphics[width=\linewidth]{figs/simulation_figs/LinearThresh_Wind8A2L8/20250505-191322/figs/CoTInfoRatio_vs_Epsilon_log.pdf}
    %     \caption{}
    % \end{subfigure}
    \hfill
    % \begin{subfigure}[t]{0.45\textwidth}
    %     \includegraphics[width=\linewidth]{figs/simulation_figs/LinearThresh_Wind8A2L8/20250505-191322/figs/min_CoTInfo_vs_niters.pdf}
    %     \caption{Limiting CoT information as $\epsilon \to 0$, varying the number of autoregressive iterations.}\label{appendixfig:LinThresh:lim_CoTInfo_vs_niters}
    % \end{subfigure}
    \begin{subfigure}[t]{0.45\textwidth}
        \includegraphics[width=\linewidth]{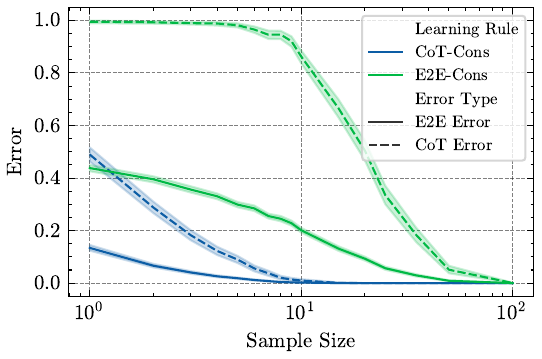}
        \caption{Learning curves with and without CoT supervision.}\label{appendixfig:LinThresh:learning_curves}
    \end{subfigure}

    \begin{subfigure}[t]{0.45\textwidth}
        \includegraphics[width=\linewidth]{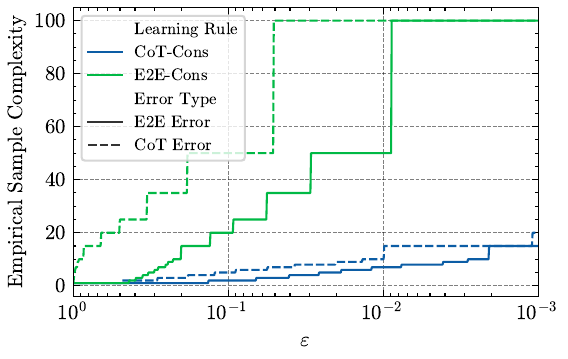}
        \caption{Empirical sample complexity.}\label{appendixfig:LinThresh:empirical_sample_complexity}
    \end{subfigure}
    \hfill
    \begin{subfigure}[t]{0.45\textwidth}
        \includegraphics[width=\linewidth]{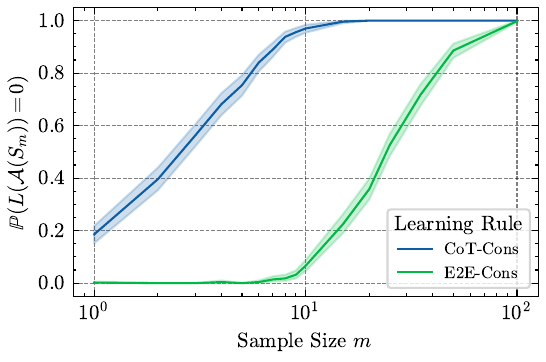}
        \caption{Empirical probability of each learning rule returning a predictor with zero error.}\label{appendixfig:LinThresh:learning_curves_prob_success}
    \end{subfigure}
    \caption{Numerical experiments for iterated linear thresholds CoT hypothesis class.}
    \label{appendixfig:LinThresh}
\end{figure*}
% describe hypothesis class & CoT structure

\subsection{Deterministic finite automata}

%A deterministic finite automaton (DFA), a type of finite state machine (FSM), is a foundational model of computation. A DFA defines a function from the space of variable-length strings to a binary classification: accept or reject. DFAs recognize exactly the class of regular languages~\citep{hopcroft2001introduction}.

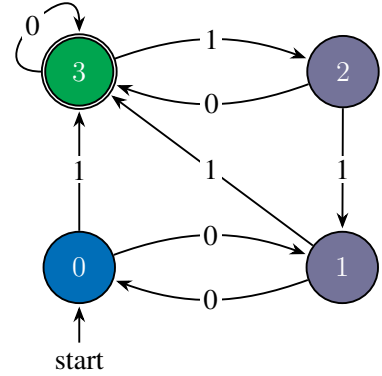
\begin{wrapfigure}{R}{0.35\linewidth}
    %\vskip-10pt
    % \includegraphics[width=.9\linewidth]{figs/simulation_figs/DFA_S4A2L10_learning/20250504-150506/hstar_fsa.pdf}
    \centering
    \resizebox{\linewidth}{!}{
        \begin{tikzpicture}[
        ->, % makes the edges directed
        >=Stealth, % makes the arrow heads nicer (using Stealth from arrows.meta)
        shorten >=1pt, % don't touch the arrow head to the node
        auto,
        node distance=1cm and 1.5cm, % Default vertical and horizontal node distances for relative positioning like "above right=of"
        thick % line style for paths
    ]
    
        % Define state styles
        \tikzstyle{stateblue}=[state, fill=NavyBlue, text=white]
        \tikzstyle{stategray}=[state, fill=CadetBlue, text=white]
        \tikzstyle{stategreen}=[state, accepting, fill=Green, text=white]
    
        % Define nodes with positions aiming to match the image
        % q0 (bottom-left, blue, initial)
        \node[initial, initial where=below, stateblue] (q0) {$0$};
    
        % q1 (bottom-right, gray)
        \node[stategray, right=2.8cm of q0] (q1) {$1$};
    
        % q3 (middle-left, green, accepting) - positioned relative to q0
        \node[stategreen, above=1.8cm of q0] (q3) {$3$};
    
        % q2 (top, gray) - positioned relative to q3 to be above and to its right
        \node[stategray, right=2.8cm of q3] (q2) {$2$};

        % Define transitions
        \path (q0) edge [bend left=20] node[fill=white, inner sep=1.5pt, auto=false] {0} (q1)
                     edge              node[fill=white, inner sep=1.5pt, auto=false] {1} (q3);
    
        \path (q1) edge [bend left=20] node[fill=white, inner sep=1.5pt, auto=false] {0} (q0)
                     edge              node[fill=white, inner sep=1.5pt, auto=false] {1} (q3);
    
        \path (q2) edge [bend left=20] node[fill=white, inner sep=1.5pt, auto=false] {0} (q3)
                     edge              node[fill=white, inner sep=1.5pt, auto=false] {1} (q1);
    
        \path (q3) edge [loop, out=180, in=90, distance=1.1cm] node[fill=white, inner sep=1.5pt, auto=false] {0} (q3) % Loop on q3 (0)
                     edge [bend left=20] node[fill=white, inner sep=1.5pt, auto=false] {1} (q2);
    
    \end{tikzpicture}
    }
    \caption{The state transition graph of the DFA corresponding to the target hypothesis $\hstar$.}
    %This DFA is identified with the regular language that checks if the final symbol is a one or a zero.
    \vskip-15pt
    \label{fig:target_dfa}
\end{wrapfigure}

CoT learning is often used as a means of providing supervision on the intermediate computation of a reference algorithm to be learned. In the following experiments, we use deterministic finite automata (DFAs) as a model of computation to study this type of supervision.

Recall that a DFA is specified by a transition function $\delta: \calS \times \Sigma \to \calS$ that maps the current state and current symbol to the next state; the automaton starts in an initial state $s_{\mathrm{init}} \in \calS$, and an accept state $s_{\mathrm{accept}} \in \calS$ is used to indicate that the string is accepted by the automaton. The final output of the DFA is $y=\IInd{z_n = s_{\mathrm{accept}}}$.
We consider a CoT hypothesis class $\calH \subset (\calZ \times \calY)^{\calX}$ where the input space is $\calX = \Sigma^n$ (or $\Sigma^*$) for some alphabet $\Sigma$. The hypothesis class $\calH$ is the set of all DFAs with state space $\calS$ operating over the alphabet $\Sigma$. The output $y = \hete{h}(\x)$ is the acceptance ($y=1$) or rejection ($y=0$) of the string $\x \in \calX = \Sigma^n$, where the chain-of-thought $\z = (\zs{n}) = \hcot{h}(\x) \in \calS^n$ is the sequence of states that the DFA visits during its execution.

In our simulations, we place a uniform distribution over the input space $\calD = \Unif(\calX) = \Unif(\Sigma^n)$. We then choose $\hstar \sim \Unif(\calH)$ randomly over the set of all automata operating on the fixed state space $\calS$ and vocabulary $\Sigma$, with a fixed initial state and accept state, and numerically compute $\relcotinfo(\hstar, h)$ and $\cotinfo(\epsilon; \calH)$.~\Cref{appendixfig:DFA_simulations} shows the simulation results for these DFA experiments, 
with the target hypothesis $\hstar$ shown in Figure~\ref{fig:target_dfa}.

\def\substudy#1{\textit{\bfseries #1.}}

\substudy{Value of CoT example \textit{vs.}~E2E example} \Cref{appendixfig:DFA_CoTInfoRatio_vs_Epsilon_by_length,appendixfig:DFA_CoTInfoRatio_vs_Epsilon_by_detail} show the ratio between the CoT information $\cotinfo(\epsilon; \calH)$ and $\epsilon$ as a function of $\epsilon$. This ratio can be interpreted as the \textit{value} of one CoT example compared to an end-to-end example, since the learning rate for E2E-supervision scales as $\log \aabs{\calH} /\epsilon$, whereas the rate for CoT supervision scales as $\log \aabs{\calH}/ \cotinfo(\epsilon; \calH)$. We clip $\epsilon$ in this ratio 
according to $\epsilon^+ = \epsilon\,\vee\,\epsilon^*$, where the smallest non-zero end-to-end error is $\epsilon^* = \min\sset{\Lete{\calD,\hstar}(h): \Lete{\calD,\hstar}(h) \neq 0}$. In other words, to achieve a target error $\epsilon < \epsilon^*$ smaller than this critical level, only $\calO(1/\epsilon^*)$ samples are required, not $\calO(1/\epsilon)$ samples. The quantity $\lim_{\epsilon \to 0} \cotinfo(\epsilon; \calH) / \epsilon^+$ can be interpreted as the ratio of the number of samples needed to achieve zero error under CoT supervision compared to the number required under E2E supervision.

\substudy{Varying input length} In~\Cref{appendixfig:DFA_CoTInfoRatio_vs_Epsilon_by_length} we vary the input sequence length $n$ in the input distribution $\calD$. That is, we take $\calD_n = \Unif(\Sigma^n)$, and we compute the CoT information as a function of $\epsilon$ for that distribution, varying $n$. We observe that the CoT information is increasing relative to $\epsilon$ as the input length $n$ increases. An intuitive explanation for this is that longer inputs allow a greater portion of the DFA's state transition map to be explored in a single example. ~\Cref{appendixfig:DFA_lim_CoTInfo_Epsilon_Ratio_vs_Length} depicts $\lim_{\epsilon \to 0} \cotinfodomain{\calD_n,\hstar}(\epsilon; \calH) / \epsilon^+$ as a function of the sequence length $n$. We see that this increases rapidly with $n$, suggesting that, for large $n$, the probability of a CoT trajectory agreeing for a pair of hypotheses with very different end-to-end behaviors is vanishingly small. For $n=10$, we see that this value is roughly $600$. By our theory (e.g.,~\Cref{result:cotinfo_ete_learning}), this would suggest a $600 \times$ improvement in sample complexity for learning with zero target error. This is indeed supported by our numerical simulations on learning with $\CoTCons$ and $\EtECons$ learning rules, as depicted in~\Cref{appendixfig:DFA_empirical_sample_complexity,appendixfig:DFA_empirical_learning_curves} and discussed further below.

\substudy{Varying CoT detail level} Next, we consider fixing the input length to $n = 10$ and instead varying the level of detail in the CoT annotations. We do this by varying the proportion of the state trajectory that is revealed to the learner, denoted by $T \in [n]$. For each $T$, we run a simulation where the CoT trajectory is limited to the first $T$ symbols of the state trajectory. As expected, the CoT information monotonically increases with $T$. In~\Cref{appendixfig:DFA_CoTInfoRatio_vs_Epsilon_by_detail} we plot the ratio of CoT information to $\epsilon$ as a function of $\epsilon$, varying the level of detail $T$, and in~\Cref{appendixfig:DFA_lim_CoTInfo_Epsilon_Ratio_vs_DetailLevel} we plot $\lim_{\epsilon \to 0} \cotinfo(\epsilon; \calH) / \epsilon^+$. While this is monotonically increasing in $T$, it begins to plateau as $T$ increases, suggesting diminishing returns in distinguishing hypotheses via their CoT trajectories---most of the information is revealed in the earlier portions of the CoT.

\substudy{Empirical sample complexity of $\CoTCons$ and $\EtECons$} Next, we directly evaluate the sample complexity of CoT-supervised learning compared to E2E-supervised learning by running simulations using the $\CoTCons$ and $\EtECons$ learning rules. We vary the sample size $m$, randomly draw a dataset $S_m$ of size $m$, and apply each learning rule to return a predictor $\calA(S_m)$, computing the end-to-end risk $\Lete{\calD}(\calA(S_m))$ of the returned predictor. The  $\CoTCons$ and $\EtECons$ learning rules are implemented by constructing the respective consistency sets and returning a predictor uniformly at random from these sets. We repeat this for 500 independent trials to estimate the distribution of risk $\Lete{\calD}(\calA(S_m))$ as a function of the sample size $m$ for each learning rule.~\Cref{appendixfig:DFA_empirical_sample_complexity} depicts the empirical sample complexity. It is computed by calculating the empirical average of the risk for each sample size $m$, and plotting the first sample size at which each target error level $\epsilon$ is attained. Giving a complementary view,~\Cref{appendixfig:DFA_empirical_learning_curves} plots the empirical probability (over the random draw of the sample $S_m$) of returning a predictor with zero loss as a function of the sample size $m$. Across both figures, we see a gain in sample efficiency from CoT supervision of the order of $10^2$--$10^3$, which agrees with the theoretical predictions via the CoT information, $\lim_{\epsilon \to 0} \cotinfo(\epsilon; \calH) / \epsilon^+ \approx 600$.

% describe experiments

% discuss results

%\aafatal{need to regenerate figures to use updated notation. E.g., $\cotinfo$ and $\Lete{\calD}$, etc.}

\aanote{Note for future version: there is a slightly updated version of these DFA simulations (difference is removing an accidental $[1:]$ in definition of FSM (basically equivalent to $L-=1$?). Okay for now, but can add updated version later after checking everything. CoT info ratio a bit larger.}

\subsection{Iterated linear thresholds}

In practice, a common way of implementing CoT supervision is to consider a sequence model class (e.g., transformers) and to train the model to generate the CoT as a sequence token-by-token, before returning the final output. In this section, we consider another CoT hypothesis class that simulates a simple form of this autoregressive generation, using a sequence model class that generates tokens as a linear function of a fixed-size window of the history.

Fix a window size $d$, and let $w \in \sset{-1, 0, 1}^{d}$ be a set of weights over this window. For a binary sequence $\x = (\xs{n}) \in \sset{0,1}^n$, we define $f_w: \x \mapsto (\x, z) \in \sset{0,1}^{\aabs{x} + 1}$ as the function that returns a sequence with the symbol $z$ appended to $\x$, where $z$ is computed by applying a threshold to the $w$-weighted linear combination of the prior $d$ symbols,
\[f_w: \x \mapsto (\x, z) \in \sset{0,1}^{\aabs{\x} + 1}, \ z = \Ind{\sum_{i=0}^{d-1} w_i \cdot x_{n - i} \geq 0}.\]
The CoT hypothesis class is defined by iterating $f_w$ for $T$ steps, taking the produced sequence as the CoT, and the final symbol as the output. That is, $\calH = \sset{h_w : w \in \sset{-1, 0, 1}^d}$, with
\begin{align*}
    \hcot{h_w}: \x \mapsto (\zs{T}), \ \hete{h_w}: \x \mapsto z_T \\
    (\x, (z_1, \ldots, z_T)) = \underbrace{(f_w \circ \dots \circ f_w)}_{T \text{ times}}(\x).
\end{align*}
This represents a simple type of autoregressive CoT hypothesis class, similar to the one studied in~\citet{joshi2025theorylearningautoregressivechain}. 

In this section, we carry out a series of numerical simulations to explore the implications of CoT information for such a class. We take the window size to be $d=8$ and the number of steps to be $T=16$. The experimental results are summarized in~\Cref{appendixfig:LinThresh}.
In~\Cref{appendixfig:LinThresh:CoTInfo_vs_Epsilon} we plot the CoT information $\cotinfo(\epsilon; \calH)$, which illustrates the monotonicity in $\epsilon$ established in~\Cref{lemma:cotinfo_props}. We observe that $\lim_{\epsilon \to 0} \cotinfo(\epsilon; \calH) \approx 0.32 > \epsilon^* \approx 0.05$, and $\lim_{\epsilon \to 0} \cotinfo(\epsilon; \calH) / \epsilon^+ \approx 6$. Consequently, our theory would suggest a $6 \times$ gain in sample efficiency from CoT supervision. This matches remarkably well with the experimental learning results shown in~\Cref{appendixfig:LinThresh:learning_curves,appendixfig:LinThresh:empirical_sample_complexity,appendixfig:LinThresh:learning_curves_prob_success}. For example,~\Cref{appendixfig:LinThresh:empirical_sample_complexity} indicates a roughly $5$-fold improvement in sample complexity for the $\CoTCons$ learning rule, compared to the $\EtECons$ learning rule at the smallest target error levels.

% describe experiments

% discuss results

% \begin{figure}
%     \centering
%     \begin{subfigure}
%         \includegraphics[width=\linewidth]{figs/simulation_figs/}
%         \caption{}\label{appendixfig:...}
%     \end{subfigure}
    
%     \caption{Caption}
%     \label{fig:enter-label}
% \end{figure}

% \subsection{}

% \subsection{Attention-based autoregressive }

\section{Discussion and Related Work}
\label{sec:discussion}

To conclude, we summarize further explorations that are presented in the appendix, position our work in the context of related literature, and highlight a few promising directions for future research.

\subsection{Further explorations}

We describe additional results not included in the main paper, and defer to the appendix for details.

\textbf{\textit{Learning with mixed CoT and E2E supervision.}} In practice, obtaining CoT-annotated examples can be a costly and labor-intensive process, limiting their quantity, whereas input-output examples without CoT annotation can be relatively cheap and plentiful. In many scenarios, one might have access to a large number of end-to-end examples and a limited number of CoT examples. This suggests a need for learning algorithms that can make use of both types of examples.~\Cref{ssec:mixed_supervision} studies learning from datasets with a mix of E2E and CoT supervision.

\textbf{\textit{CoT learning with inductive priors.}} Encoding prior knowledge about solution structure is critical for learning complex functions, such as those representing multi-step reasoning processes, particularly from limited data. \Cref{ssec:inductive_priors} explores chain-of-thought learning with inductive priors.

\textbf{\textit{Transfer learning and out-of-distribution generalization.}} Chain-of-thought supervision has significant implications for out-of-distribution generalization, as it guides the learning algorithm toward solutions that exhibit the correct step-by-step reasoning, potentially enabling robust generalization to novel input instances. In~\Cref{ssec:transfer_learning}, we define a variant of the CoT Information measure, $\cotinfodomain{\calD_{\tr} \to \calD_{\test}}(\epsilon; \calH)$, that captures transfer learning under CoT supervision. We also present a result on learning under CoT supervision with distribution shift, supported by experimental simulations.

\subsection{Related work}

Early usage of the term ``chain-of-thought'' referred to empirical prompting techniques that conditioned large language models to generate a series of intermediate reasoning steps before returning the final answer~\citep{nye2021show,cobbe_training_2021,wei2022chain,kojima2022large}. Such prompting often employs in-context learning, where CoT examples are provided within the model's context before it processes the input~\citep{wei2022chain}. Today, the term \textit{chain-of-thought} takes a broader meaning, as it now comprises a core component of the training of large language models~\citep{hoLargeLanguageModels2023,chung2022scalinginstructionfinetunedlanguagemodels,ouyang2022traininglanguagemodelsfollow}. %Training language models to iteratively generate a step-by-step derivation of the solution has proven to be a powerful empirical technique that enables learning of complex functions.

Several works have sought to theoretically understand the advantages of the chain-of-thought paradigm by analyzing the representational capacity of neural sequence models with and without chain-of-thought~\citep{perez2021attention,merrill2023expresssive,feng2023revealingmysterychainthought,li2024chainthoughtempowerstransformers}. 
% That is, these works study the class of functions that can be efficiently represented by neural architectures such as transformers when imbued with the ability to generate a sequence of intermediate tokens before the final answer. 
For example,~\citet{perez2021attention} show that Transformers can simulate Turing machines by generating CoT tokens, and~\citet{merrill2023expresssive} extend this analysis by providing a more refined characterization of function classes in terms of the number of CoT steps.

While these studies demonstrate the \textit{existence} of neural network models capable of computing a given function via a specific chain-of-thought, they do not address the statistical question of whether such models can be efficiently \textit{learned} from data. Our work focuses on these statistical learning aspects, a direction also pursued by a few recent studies.
For example, \citet{hu2024unveilingstatisticalfoundationschainofthought} studies statistical aspects of chain-of-thought \textit{prompting} techniques (e.g., in-context learning) via a latent variable model, relating CoT prompting to Bayesian model averaging~\citep{hoeting1999bayesian}.
% This work addresses questions that are different from our focus on learning from CoT-annotated datasets.

Addressing a setting more similar to ours, \citet{malach2024autoregressivenexttokenpredictorsuniversal} considered learning from CoT datasets. The core idea of this work is to express a CoT function as a composition of $T$ (a fixed number) different sequence-domain functions, $\calH = \calH_1 \times \cdots \times \calH_T$, where each $\calH_t: \calX \times \Sigma^{t - 1} \to \Sigma$ maps the input and CoT generated so far $(\x, z_1, ..., z_{t-1})$ to the next CoT symbol $z_t$, with the $T$-th CoT symbol serving as the final output. With this formulation,~\citet{malach2024autoregressivenexttokenpredictorsuniversal} proposes learning each $\calH_t$ independently (with independent parameters), enabling the direct application of standard PAC results. While this approach simplifies the analysis, its assumption of independently learned functions at each iteration is a notable limitation, which does not accurately reflect real-world settings.

Building on this work, \citet{joshi2025theorylearningautoregressivechain} considers a time-invariant composition of sequence-domain functions, where the function at each iteration remains the same. Their analysis relies on bounding the \textit{CoT error} using standard PAC learning tools based on the VC dimension of the CoT loss class, noting that the CoT error provides an upper bound on the end-to-end error (cf. \Cref{ssec:key_idea_control_cot_error} and the second row of~\Cref{table:comparison_table}). 
% A key technical contribution of~\citet{joshi2025theorylearningautoregressivechain} is to relate the VC dimension of the CoT loss class to the VC dimension of the iterated base class and the number of iterations $T$. 
%In this analysis, the statistical advantage of CoT supervision only appears when the VC dimension of the CoT loss class is smaller than the VC dimension of the end-to-end loss class.
%Moreover,~\citet{joshi2025theorylearningautoregressivechain} construct a synthetic autoregressive class exhibiting a gap between the VC dimension of the end-to-end loss class and CoT loss class, implying a statistical advantage for CoT supervision. \omnote{a "gap" doesn't clarify which is smaller}
Moreover,~\citet{joshi2025theorylearningautoregressivechain} construct a synthetic autoregressive class where the VC dimension of the CoT loss class is much smaller than the VC dimension of the end-to-end loss class, implying a statistical advantage for CoT supervision.
% In their time-invariant setting, \citet{joshi2025theorylearningautoregressivechain} demonstrated a statistical advantage of CoT supervision by showing that the VC dimension of the CoT loss class can be much smaller than the VC dimension of the end-to-end loss class.
While our results also involve the CoT loss class, thus inheriting the advantages of such class complexity differences, the focus of our analysis is on the content of \textit{information} per CoT-supervised sample. This is represented in the dependence of the sample complexity on the target error $\epsilon$, captured by the CoT information measure. This provides a more complete description of the statistical advantage of CoT supervision in statistical learning.
We contend that this information-theoretic analysis, centered on CoT information rather than solely on loss class complexity, identifies a more fundamental source of statistical advantage in CoT-supervised learning. This view is supported by our lower bound results and the close agreement between our theory and simulations.

%\aanote{How about we mention that they construct a synthetic example where they exhibit a gap between the two VC dimensions in favor of CoT learning, without the specificity of the $T$-dependence, which I think might not make sense to most readers. I'm not sure if these specific $(d,T)$-dependent VC bounds will be meaningful to a reader who isn't closely familiar with the framework of the \citet{joshi2025theorylearningautoregressivechain} paper. also, as we discussed before, I'm not totally convinced that the VC dimension of the ``base class'' (as a function class over variable-length inputs) is a meaningful measure beyond being a proof device. In their work, the only example they give is over fixed-length input windows (the iterated thresholds example). So it might be confusing for us to talk about the VC dimension of the base class.}

%\omnote{I think that we should be careful with using the word "characterization". Specifically, while CoT information measure is definitely an important part of the picture, what remains to be figured out is the right complexity measure or dimension of $\calH$ that governs CoT learnability. So far, the analysis relies on $\VC(\calLCoT(\calH))$, but that is likely not the right complexity measure. We can mention this as a direction for future work.} 

\subsection{Conclusion and future work}

This work provides a theoretical analysis of learning with chain-of-thought (CoT) supervision, introducing the CoT information measure to characterize its statistical advantages via both upper and lower bounds. This opens several promising directions for future theoretical study of CoT learning.

The upper bounds obtained in this work are based on analyzing natural but relatively simple learning rules: CoT-consistency in the realizable setting, and CoT-ERM in the agnostic setting. Investigating the optimality of these algorithms and exploring the design of optimal learning strategies remain key open questions. This may be especially relevant in the agnostic setting, where the alignment between the data distribution and the CoT hypothesis class is critical. For instance, future work could explore adaptive learning rules that balance the optimization of the CoT error and end-to-end error to avoid over-optimizing the CoT when the hypothesis class is poorly aligned with the data. 
While we use the VC dimension of CoT loss class as the measure of complexity or size of the hypothesis class, it will also be important to consider other measures of model complexity, including covering numbers, local Radamacher complexities, and one-inclusion graphs~\citep{bartlett2002localized,bartlett2005local,bousquetLocalMeasuresComplexity2004,lugosiComplexityRegularizationLocalized2004,mendelsonImprovingSampleComplexity2002,DBLP:journals/iandc/HausslerLW94}. 

Furthermore, while our current lower bound results address the realizable setting, establishing corresponding lower bounds for the agnostic setting remains an important open problem. Additionally, future research could investigate the formulation of different structural conditions, such as low-noise assumptions~\citep{mammen1999smooth,massartRiskBoundsStatistical2006} in the CoT setting, to achieve faster learning rates. Developing more sophisticated probabilistic analysis, beyond the standard formulation of agnostic learning, holds promise for more faithfully capturing the complexities of training language models with chain-of-thought reasoning traces, which are often inherently probabilistic.

Finally, our lower bound approach using Fano's method leads naturally to viewing a chain-of-thought sequence $\hcot{h}(\x)$ as a codeword that redundantly encodes an algorithmic procedure for computing the correct output $y$ for the input $\x$. This codeword is then rendered in natural language as an observed sequence $\z$ after passing through a noisy channel. From this perspective, a CoT hypothesis class $\calH$ can be viewed as a \textit{family} of codes. In CoT learning, rather than observing fixed input-output examples, there often exists some flexibility in the design of the CoT training trajectories, as they are human-generated.   As a coding theory problem, we would like the CoT trajectories for different hypotheses to look as different as possible to make them easy to distinguish, even under noise. The concept of CoT information may allow this to be formalized, leading to the design families of CoT codes that can be efficiently identified from data.

\section*{Acknowledgements}

We thank Dana Angluin and Peter Bartlett for helpful comments on this work. This research was supported by the funds provided by the National Science Foundation and by DoD OUSD (R\&E) under Cooperative Agreement PHY-2229929 (The NSF AI Institute for Artificial and Natural Intelligence). 

\printbibliography

%%%%%%%%%%%%%%%%%%%%%%%%%%%%%%%%%%%%%%%%%%%%%%%%%%%%%%%%%%%%

\clearpage

%\listoffixmes\newpage

\appendix

\section{Proofs of Properties of the CoT Information (\texorpdfstring{\Cref{ssec:cot_properties}}{Section~\ref{ssec:cot_properties}})}\label{sec:proofs_properties}

\begin{lemma*}[Restatement: Properties of the CoT-information]%\label{lemma:cotinfo_props}
  % \hphantom{~}
  Let $\calH \subset (\calZ \times \calY)^{\calX}$ be a CoT hypothesis class.
  \begin{enumerate}
    \item The CoT information is always larger than the ``end-to-end information'':\newline
    For any $h_1, h_2 \in \calH$, $\relcotinfo(h_1, h_2) \geq \pprobunder{x \sim \calD}{\hete{h_1}(x) \neq \hete{h_2}(x)}$. Moreover, \newline
    $\cotinfo(\epsilon; \calH) \geq \epsilon,\, \forall \epsilon \in [0,1], \forall \hstar \in \calH$.
    \item $\cotinfo(\,\cdot\,; \calH)$ is monotonically increasing in $\epsilon$: \newline For any $\calH, \hstar \in \calH$, and $\epsilon_1 \leq \epsilon_2$, $\cotinfo(\epsilon_1; \calH) \leq \cotinfo(\epsilon_2; \calH)$.
    \item $\cotinfo(\epsilon; \,\cdot)$ is monotonically decreasing in the hypothesis class: \newline For any hypothesis classes $\calH \subseteq \calH'$ and $\hstar \in \calH$, $\cotinfo(\epsilon; \calH) \geq \cotinfo(\epsilon; \calH')$.
  \end{enumerate}
\end{lemma*}

\begin{proof}
  \hphantom{~}

  \textit{Property 1.}   To prove the first claim, take any $h_1, h_2 \in \calH$, and observe that
  \begin{align*}
    \relcotinfo(h_1, h_2) &:= - \log \probunder{\x \sim \calD}{\hcot{h_1}(\x) = \hcot{h_2}(\x), \hete{h_1}(\x) = \hete{h_2}(\x)} \\
    &\geq - \log \probunder{\x \sim \calD}{\hete{h_1}(\x) = \hete{h_2}(\x)} \\
    &= - \log \paren{1 - \probunder{\x \sim \calD}{\hete{h_1}(\x) \neq \hete{h_2}(\x)}} \\
    &\geq \probunder{\x \sim \calD}{\hete{h_1}(\x) \neq \hete{h_2}(\x)},
  \end{align*}
  where the final inequality is by the identity $- \log (1-x) \geq x$.

  We now show the second claim,
  \begin{align*}
    \cotinfo(\epsilon; \calH) &:= \min_{h \in \Deltaete_{\calD}(\epsilon; \calH, \hstar)} \relcotinfo(\hstar, h) \\
    &\geq \min_{h \in \Deltaete_{\calD}(\epsilon; \calH, \hstar)} \probunder{\x \sim \calD}{\hete{h}(\x) \neq \hete{\hstar}(\x)} \\
    &\geq \epsilon
  \end{align*}
  The final inequality follows because $\probunder{\x \sim \calD}{\hete{h}(\x) \neq \hete{\hstar}(\x)} \geq \epsilon,\, \forall h \in \Deltaete_{\calD}(\epsilon; \calH, \hstar)$, by definition.

  \textit{Property 2.} This follows from the fact that $\Deltaete_{\calD}(\epsilon; \calH, \hstar) := \sset{h \in \calH: \pprob{\hete{h}(\x) \neq \hete{\htilde}(\x)} > \epsilon}$ is decreasing in $\epsilon$. For $\epsilon_1 \leq \epsilon_2$, we have $\Deltaete_{\calD}(\epsilon_1; \calH, \hstar) \supseteq \Deltaete_{\calD}(\epsilon_2; \calH, \hstar)$, and hence
  \begin{align*}
    \cotinfo(\epsilon_1; \calH) := \min_{h \in \Deltaete_\calD(\epsilon_1; \calH, \hstar)} \relcotinfo(h, \hstar)
    \leq  \min_{h \in \Deltaete_\calD(\epsilon_2; \calH, \hstar)} \relcotinfo(h, \hstar)
    =: \cotinfo(\epsilon_2; \calH).
  \end{align*}

  \textit{Property 3.} This property similarly follows from the fact that $\Deltaete_{\calD}(\epsilon; \calH, \hstar)$ is increasing in $\calH$: $\Deltaete_{\calD}(\epsilon; \calH, \hstar) \subset \Deltaete_{\calD}(\epsilon; \calH', \hstar)$ for $\calH \subseteq \calH'$. Thus,
  \begin{align*}
    \cotinfo(\epsilon; \calH) := \min_{h \in \Deltaete_\calD(\epsilon; \calH, \hstar)} \relcotinfo(h, \hstar)
    \geq  \min_{h \in \Deltaete_\calD(\epsilon; \calH', \hstar)} \relcotinfo(h, \hstar)
    =: \cotinfo(\epsilon; \calH').
  \end{align*}
\end{proof}

% \section{Details for Examples in~\Cref{sec:examples_props}}\label{sec:examples_appendix}
\section{Simple Examples of CoT Hypothesis classes and their CoT Information}\label{sec:examples_appendix}

In this section, we provide a more detailed discussion on the illustrative examples presented in~\Cref{sec:prelims}. The first two examples represent the two extremes on the informativeness of CoT supervision and serve as sanity checks to confirm that the CoT information captures the expected statistical complexity in each case. The next example considers a hypothesis class where the CoT supervision includes $T$ independent samples of the input-output function, and shows that the CoT information scales linearly with $T$ as expected. Finally, we consider CoT hypothesis classes based on models of computation such as finite-state machines, where the CoT is taken to be the state trajectory of the computational process.

\begin{example}[Uninformative CoT yields small CoT information]\label{example:uninformative}
    In some cases, the CoT annotations may be entirely ``independent'' from the end-to-end behavior, and hence uninformative for the purposes of learning with respect to the end-to-end error. We will see that the CoT information $\cotinfo(\epsilon; \calH)$ captures this.  We will model the ``independence'' between the CoT and the end-to-end behavior via a hypothesis class with a \textit{product structure}. Let $\calF^{\ete} \subset \calY^{\calX}$ be a function class from inputs $\calX$ to outputs $\calY$ and let $\calF^{\CoT} \subset \calZ^{\calX}$ be a function class from inputs $\calX$ to the CoT space $\calZ$. We consider a CoT hypothesis class $\calH = \calF^{\CoT} \times \calF^{\ete}$. where
    \[\calH = \sset{h_{g,f}: x \mapsto (g(x), f(x)) \ | \ g \in \calF^{\CoT}, f \in \calF^{\ete}}.\] 
    % Due to this product structure, the chain-of-thought is independent from the end-to-end behavior and should not be informative for learning the end-to-end function.
    % This is captured by the CoT information. 
    Let $\hstar = (g_\star, f_\star) \in \calH = \calF^{\CoT} \times \calF^{\ete}$. Let $\bar{f} \in \calF^{\ete}$ be the end-to-end hypothesis with smallest disagreement with $f_\star$ among hypothesis with end-to-end error at least $\epsilon$:
    \[\bar{f} = \argmin_{f \in \calF^{\ete}} \set{\probunder{x \sim \calD}{f_\star(x) \neq f(x)} \ : \ \probunder{x \sim \calD}{f_\star(x) \neq f(x)} > \epsilon}.\]
    Let $\epsilon^+ := \min \sset{\pprob{f_\star(x) \neq f(x)} \ : \ \pprob{x \sim \calD}{f_\star(x) \neq f(x)} > \epsilon}$. By the product-structure definition of $\calH$, there exists a hypothesis $\bar{h} := (g_\star, \bar{f}) \in \Deltaete_{\calD}(\epsilon; \calH)$ such that $\pprob{\hstar(x) \neq \bar{h}(x)} = \pprob{\hete{\hstar}(x) \neq \hete{\bar{h}}(x)} = \epsilon^+$, and hence $\cotinfo(\epsilon; \calH, \hstar) = - \log(1 - \epsilon^+)$. Thus, there is no statistical advantage in observing the CoT annotations.
\end{example}

\begin{example}[Fully Informative CoT yields infinite CoT information]\label{example:degenerate_inf_cotinfo}
    Recall the definition of the CoT information as
    \[\cotinfo(\epsilon; \calH) := \inf\set{- \log \probunder{x}{h(x) = \hstar(x)} : h  \in \Deltaete_{\calD}(\epsilon; \calH, \hstar)}.\]
    This can be infinite when $\forall h  \in \Deltaete_{\calD}(\epsilon; \calH, \hstar)$, $\pprobunder{x}{h(x) = \hstar(x)} = 0$. This occurs in the extreme case where a single CoT annotation uniquely identifies the end-to-end behavior of the hypothesis (i.e., on every input in the support of $\calD$, each hypothesis has a unique CoT). To illustrate this, let $\calF^{\ete} \subset \calY^{\calX}$ be a class of functions from the input space $\calX$ to the output space $\calY$. Consider the CoT hypothesis class $\calH = \sset{h_f: x \mapsto (f, f(x)) \ : \ f \in \calF^{\ete}}$. In this extreme example, a single sample is enough to learn the function perfectly. This is captured by the CoT information since $\forall h_1 \neq h_2 \in \calH$, we have $\pprob{h_1(x) = h_2(x)} = 0$ and hence $\cotinfo(\epsilon; \calH) = \infty$.
\end{example}

\begin{example}[CoT Information captures i.i.d. examples in CoT]\label{example:iid}
    % Given a fixed end-to-end function class $\calF^{\ete} \subset \calY^\calX$ mapping instances in $\calX$ to labels in $\calY$, there exists an infinite number of CoT hypothesis classes $\calH \subset (\calZ \times \calY)^{\calX}$ with matching end-to-end behavior. One aspect we might be interested in is studying the statistical complexity of learning in the CoT setting, as the \textit{level of detail} in the CoT annotations varies. In practical settings, the chain-of-thought is typically a sequence that represents the step-by-step computations carried out by the end-to-end function. The level of detail in the sequence can vary and this affects the statistical complexity of learning. 
    In this example, we consider a setting where the chain-of-thought represents i.i.d. observations from the end-to-end function, as a toy model that allows us to vary the informativeness of the CoT supervision for a fixed end-to-end function class. We will confirm that the CoT information implies the sample complexity rates that we would expect. Consider the CoT hypothesis class $\calH^{(T)}$ where the CoT encodes $T$ independent observations, defined as follows:
    \[\calH^{(T)} := \set{h_f^{(T)} : (x_1, \ldots, x_T) \mapsto (\underbrace{(f(x_1), \ldots,f(x_T))}_{\z = (\zs{T})}, \underbrace{f(x_T)}_{y}) \ \given \ f \in \calF}.\]
    Here, $\calF$ is a function class from $\bar{\calX}$ to $\bar{\calY}$, and $\calX = \bar{\calX}^T, \calZ = \bar{\calY}^T, \calY = \bar{\calY}$. 
    Let $\calD = \bar{\calD}^{\otimes T}$ for some distribution $\bar{\calD}$ over $\bar{\calX}$. Fix $\hstar = h_{\fstar}$ and let $h = h_f \in \Deltaete_{\calD}(\epsilon; \calH, \hstar)$. We have
    \begin{align*}
        \relcotinfo(\hstar, h) &= - \log \probunder{x \sim \calD}{\hstar(x) = h(x)} = - \log \probunder{x_1, \ldots, x_T \simiid \bar{\calD}}{\forall t \in [T]: \fstar(x_t) = f(x_t)} \\
        &\geq - \log (1 - \epsilon)^T = T \cdot (- \log(1- \epsilon)) \geq T \cdot \epsilon.
    \end{align*}
    This in turn implies that $\cotinfo(\epsilon; \calH) \geq T \cdot \epsilon$. That is, one CoT sample is worth $T$ end-to-end samples, and the CoT sample complexity is smaller by a factor of $T$. This is what we would expect for this example since a CoT example consists of $T$ independent samples.
\end{example}

%\aawarning{Maybe call them DFAs rather than FSMs?}
\begin{example}[Learning Regular Languages with State-Trajectory CoT]\label{example:fsm_cotinfo_lowerbound}
    Let $\calH$ be the class of Finite-State Machines with common state space $\calS$ and operating over an alphabet $\Sigma$. That is, $\calH = \sset{h_\delta : \delta \in \calT}$, where $\calT$ is the set of transition functions $\calT = \calS^{\calS \times \Sigma}$. The Chain-of-Thought observed by the learner is the sequence of states visited by the DFA during its execution: for an input $\x = (\xs{n}) \in \Sigma^n$, the CoT of $h_\delta$ is $\z = (\zs{n})$, where $z_{t+1} = \delta(z_t, x_t)$. Observing the CoT can be interpreted as providing the learner with an input-dependent partial observation of the DFA's underlying transition function. Once the learner has identified all components of the transition function (or all components that are necessary to specify the input-output behavior), the learning objective is achieved. We can use this interpretation to lower-bound the CoT information. Let $\Delta(h_1, h_2) = \sset{(s, x) \in \calS \times \Sigma : \delta_1(s, x) \neq \delta_2(s, x)}$ be the set of state-symbol pairs on which $h_1$ and $h_2$'s transition functions differ. Then, we have
    \begin{align*}
        & 1 - \probunder{\x \sim \calD}{\hcot{\hstar}(\x) = \hcot{h}(\x), \hete{\hstar}(\x) = \hete{h}(\x)} \\
        &\geq \probunder{\x \sim \calD}{\hcot{\hstar}(\x) \neq \hcot{h}(\x)} \\
        &= \probunder{\x \sim \calD}{\exists t \in [n]: z_t \neq z_t^*} \\
        &= \probunder{\x \sim \calD}{\text{$\hstar$ visits any $(s, a) \in \Delta(h, \hstar)$}}
        % &\begin{aligned}
        %     = \underset{\x \sim \calD}{\Prob} \Big[\exists i \in [L]: &&\paren{i \in \partialobs{\fstar}(\x) \cap \partialobs{f}(\x) \text{ and } \Desc(\fstar)[i] \neq \Desc(f)[i]} \\
        %     &&\text{ or } \paren{i \in \partialobs{\fstar}(\x) \, \Delta \, \partialobs{f}(\x)}\Big]
        % \end{aligned}\\
        % &\geq \probunder{\x \sim \calD}{\exists i \in [L]: i \in \partialobs{\fstar}(\x) \text{ and } \Desc(\fstar)[i] \neq \Desc(f)[i]} \\
        % &= \probunder{\x \sim \calD}{\partialobs{\fstar}(\x) \cap \Deltadesc(\hstar, h) \neq \emptyset} \\
        % &\geq \min_{h \in \Deltaete_{\calD}(\epsilon; \calH, \hstar)} \probunder{\x \sim \calD}{\partialobs{\fstar}(\x) \cap \Deltadesc(\hstar, h) \neq \emptyset}. \\
    \end{align*}
    Suppose that $\hstar$'s transition graph is $\ell$-connected in the sense that for every state $s \in \calS$ which is reachable by some input supported by $\calD$, $\exists \ell' \leq \ell, a_1, \ldots, a_{\ell'} \in \Sigma$ such that $s_{\ell'} = s$ where $s_{t+1} = \delta^*(s_t, a_t), s_1 = s_{init}$. Then, if e.g. $\calD = \Unif(\Sigma^n), n \geq \ell$, the above calculation implies that for all $h \neq \hstar$, 
    \[\probunder{\x \sim \calD}{\hcot{\hstar}(\x) = \hcot{h}(\x), \hete{\hstar}(\x) = \hete{h}(\x)} \leq 1 - \aabs{\Sigma}^{-(\ell + 1)}.\]
    Thus, the CoT information is lower bounded as
    \[\min_{\epsilon > 0} \cotinfo(\epsilon; \calH) \geq \abs{\Sigma}^{- (\ell + 1)}.\]
    Note that this bound may be loose since it only counts a single trajectory that can be used to distinguish between the pair of hypotheses. But, its strength is that it lower bounds the CoT information at all error levels $\epsilon$, and hence upper bounds the sample complexity of achieving \textit{zero} error.
    A rich literature exists on learning regular languages \citep[e.g.][]{goldComplexityAutomatonIdentification1978,angluinNoteNumberQueries1981,angluinLearningRegularSets1987,rivestDiversitybasedInferenceFinite1987,rivestInferenceFiniteAutomata1989,freundEfficientLearningTypical1993}.

%    \aanote{Also, as inputs get longer, the probability of visiting a given state-symbol pair and revealing the associated component of the description increases, as the machine visits a larger number of state-symbol pairs in a single run. Moreover, if the FSM forms an ergodic Markov chain, the probability that all state-action pairs are visited approaches 1 as the input length increases. Perhaps we can provide such an analysis by representing a FSM as a Markov chain~\citep{davisMarkovChainsRandom1961}. We should also try to understand related work on learning automata and see if they consider such settings. Maybe reach out to Dana Angluin?}

%    \aawarning{Cite rich literature on learning automata or regular languages:}

\end{example}

\begin{example}[Learning Shuffle Ideals by Observing Computational Trace of Finite State Machines]{}
    The class of shuffle ideals is a simple subclass of regular languages that has been studied in the context of efficient PAC learning~\citep{simon1975piecewise,angluin2013learnability}. For a string $u \in \Sigma^n$, the shuffle ideal generated by $u$ is the language $\Sigma^* u_1 \Sigma^* u_2 \Sigma^* \cdots \Sigma^* u_{n} \Sigma^*$ consisting of all strings which contain $u$ as a subsequence. The class of shuffle ideals of strings of length $n$ can be represented by finite state automata with $n+1$ states. For a string $u \in \Sigma^n$, it's shuffle ideal is recognized by the finite state machine with the following transition function
    \begin{equation*}
        \delta(s, a) = \begin{cases} s + 1, &\text{if } a = u_s \\ s &\text{otherwise}.\end{cases}
    \end{equation*}
    The acceptance state is $s = n + 1$. This finite state machine has a state space with a sequential structure, with each state ``looking for'' a particular symbol. When that symbol is observed, the state progresses to the next. A string is accepted if state $n+1$ is reached, signifying that all symbols in the string $u$ are observed in the correct order. Due to the structure of this hypothesis class, each state has exactly one symbol that causes it to progress. Thus, to learn the FSA perfectly, it is enough to learn which symbol each state accepts. In fact, due to the sequential structure of this class of finite-state machines, this information is revealed on a single trajectory from a positive example. Thus, the CoT information can be bounded in terms of the probability of observing a positive example. In particular, for $\hstar \neq h \in \calH$, we have
    \begin{align*}
        \probunder{x \sim \calD}{\hcot{\hstar}(x) \neq \hcot{h}(x)} \geq \probunder{x \sim \calD}{\hete{\hstar}(x) = \text{accept}},
    \end{align*}
    and hence $\min_{\epsilon > 0} \cotinfo(\epsilon; \calH) \geq \probunder{x \sim \calD}{\hete{\hstar}(x) = \text{accept}}$.
\end{example}

\aawarning{Should wee keep or remove the Turing machine example? The connectivity condition becomes a bit more artificial in this case, so we're not including a lower bound calculation, but is it still illustrative to include it as an example of what the CoT could look like?}
\begin{example}[Turing Machines]{}
    It is possible to consider learning Turing Machines with CoT-supervision in a manner similar to~\Cref{example:fsm_cotinfo_lowerbound}. Recall that a Turing machine is specified by a transition function $\delta: \calS \times \Sigma \to \calS \times \Sigma \times \sset{\pm 1}$, mapping the current state $s$ and observed symbol $\sigma$ on the current position in the tape to the next state $s'$, the symbol to be written $\gamma$, and the direction to move the tape $d$. We may consider Turing machines with chain-of-thought-supervision, where the CoT is the trajectory of states, written symbols, and tape movements: $\z = (\langle s_1, \gamma_1, d_1 \rangle, \ldots, \langle s_{t(x)}, \gamma_{t(x)}, d_{t(x)}\rangle)$, where $t(x)$ is the halting time of the Turing machine, and can depend on the input $x$ and the instance $\hstar$. Similar to the case of DFAs, observing the CoT reveals an input-dependent partial specification of the underlying transition function.
    To lower bound the CoT information, one can consider analogous ``connectivity'' conditions to those discussed in~\Cref{example:fsm_cotinfo_lowerbound}. For example, one such condition is that for every $s \in \calS$, there exists an input prefix $w \in \Sigma^{\leq \ell}$ such that the Turing machine $\hstar$ lands in state $s$ when reading $w$ and starting at $s_{init}$.

    \aanote{Has there been work that explores or formalizes ``connectivity''-type conditions on Turing Machines?}
\end{example}

\section{Proofs for \texorpdfstring{\Cref{sec:upper_bounds}}{Section~\ref{sec:upper_bounds}}: Upper Bounds}\label{sec:proof_upperbounds}

\subsection{Proof of \texorpdfstring{\Cref{result:cotcons_cotinfo_infH}}{Result~\ref{result:cotcons_cotinfo_infH}}}\label{ssec:proof:result:cotcons_cotinfo_infH}

\begin{result*}[Restatement: Learning Infinite Classes under CoT-Supervision]%\label{result:cotcons_cotinfo_infH}
    Let $\calH \subset (\calZ \times \calY)^{\calX}$ be a CoT hypothesis class. For any distribution $\calD$ over $\calX \times \calY \times \calZ$ realized by some $\hstar \in \calH$, the CoT-consistency learning rule has a sample complexity of
    % \[m(\epsilon, \delta) = \calO \paren{\max\paren{\frac{1}{\cotinfo(\epsilon; \calH)}, 1} \cdot \paren{ \VC(\calLCoT(\calH)) \cdot \log \max\paren{\frac{1}{\cotinfo(\epsilon; \calH)}, 1} + \log(1/\delta)}}.\]
    \[m(\epsilon, \delta) =  \bigO\paren{\Bigparen{\frac{1}{\cotinfo(\epsilon; \calH)} + 1} \Bigparen{ \VC(\calLCoT(\calH)) \cdot \log \Bigparen{\frac{1}{\cotinfo(\epsilon; \calH)} + 1} + \log(1/\delta)}}\]
    % \[m(\epsilon, \delta) =  \bigOtilde\paren{\frac{\VC(\calLCoT(\calH)) + \log(1/\delta)}{\cotinfo(\epsilon; \calH)}}\] % this version puts things under bigOtilde
    That is, for any $m \geq m(\epsilon, \delta)$, we have that with probability at least $1 - \delta$ over $S \sim \calD^m$, 
    \[\forall h \in \CoTCons(S; \calH), \ \text{we have} \ \Lete{\calD}(h) \leq \epsilon.\]
\end{result*}

The key to proving this result will be to establish the following lemma, which relates the performance of any \textit{proper} CoT learner with respect to the CoT error to its performance with respect to the end-to-end error. As before, the intuition is that achieving small CoT error implies \textit{very small} end-to-end error, because the CoT error measures any algorithmic errors, not only errors in the answer (which might be reachable via an incorrect algorithm). This relationship is captured by the CoT information.

Recall that a \textit{proper} learner for $\calH$ is defined as a learning algorithm that returns a predictor \textit{in the hypothesis class}. In general, an improper learner may return any predictor, not necessarily in the hypothesis class (and this can have some computational advantages).

\begin{lemma}[Relating CoT performance to E2E performance via CoT Information]\label{lemma:cot_ete_conversion}
    Any \textit{proper} CoT-learner $\calA: (\calX \times \calY \times \calZ)^* \to \calH$ which achieves CoT-error $\epsilon$  with sample complexity $m_{\calA}^{\CoT}(\epsilon, \delta)$ also achieves end-to-end error $\epsilon$ with sample complexity $m_{\calA}^{\ete}(\epsilon, \delta) \leq m_{\calA}^{\CoT}(\gamma(\epsilon)^{-}, \delta)$, where $\gamma: (0, 1) \to (0,1)$ is defined as
    \[\gamma(\epsilon) := \inf\set{\Lcot{\calD}(h): \ h \in \Deltaete_{\calD}(\epsilon; \calH)}.\]
    % \[\gamma(\epsilon) := \inf\set{\Lcot{\calD}(h): \Lete{\calD}(h) > \epsilon}.\]
    Here, $m(\epsilon^{-}, \delta)$ denotes the limit to $\epsilon$ from below. Moreover, $\gamma$ can be related to the CoT information as follows
    \[\max\paren{\frac{\cotinfo(\epsilon; \calH)}{1 + \cotinfo(\epsilon; \calH)}, \epsilon} \leq \gamma(\epsilon) \leq \min(\cotinfo(\epsilon; \calH), 1).\]
\end{lemma}

\begin{proof}
    Let $\calA$ be a proper CoT learner for $\calH$ (i.e., it returns a hypothesis in $\calH$) with CoT-error sample complexity $m_{\calA}^{\CoT}(\epsilon, \delta)$. Let $m \geq m_{\calA}^{\CoT}(\gamma(\epsilon), \delta)$ and let 
    \[S = \sset{(x_1, y_1, z_1), \ldots, (x_m, y_m, z_m)} \simiid \calD\] be an i.i.d dataset drawn from the distribution $\calD$ . Note that we fold the hypothesis $\hstar$ into $\calD$ for notational convenience, and we have $y_i, z_i = \hstar(x_i)$. By the assumption on the CoT-error sample complexity of $\calA$, we have that with probability at least $1 - \delta$, $\hhat = \calA(S)$ satisfies $\Lcot{\calD}(\hhat) \leq \gamma(\epsilon)^{-} < \gamma(\epsilon)$.

    To show that $\hhat$ has end-to-end error smaller than $\epsilon$, we proceed by contradiction. Suppose we are in the event  $\sset{S: \Lcot{\calD}(\calA(S)) < \gamma(\epsilon)}$ and that $\Lete{\calD}(\hhat) > \epsilon$. This implies $\hhat \in \Deltaete_{\calD}(\epsilon; \calH)$ and hence
    \[\gamma(\epsilon) := \inf\set{\Lcot{\calD}(h): \ h \in \Deltaete_{\calD}(\epsilon, \calH)} \leq \Lcot{\calD}(\hhat) \leq \gamma(\epsilon)^{-} < \gamma(\epsilon).\]
    This yields a contradiction. Thus, on the event $\sset{S: \Lcot{\calD}(\calA(S)) < \gamma(\epsilon)}$, which occurs with probability at least $1 - \delta$, we have $\Lete{\calD}(\hhat) \leq \epsilon$. This proves the first part of the lemma.

    We now proceed to relate $\gamma$ to the CoT information. Note that the CoT information can be written in terms of $\gamma$ as follows
    \begin{align*}
        \cotinfo(\epsilon; \calH) &:= \inf_{h \in \Deltaete_{\calD}(\epsilon; \calH)} \set{-\log \probunder{x, y, z}{h(x) = (y, z)}} \\
        &= - \log \sup_{h \in \Deltaete_{\calD}(\epsilon; \calH)} \paren{1 - \Lcot{\calD}(h)} \\
        &= - \log \paren{1 - \inf_{h \in \Deltaete_{\calD}(\epsilon; \calH)} \Lcot{\calD}(h)} \\
        &= - \log (1 - \gamma(\epsilon)).
    \end{align*}

    The identity $- \log (1-x) \geq x$ gives $\cotinfo(\epsilon; \calH) \geq \gamma(\epsilon)$. The identity $- \log(1-x) \leq \frac{x}{1-x}$ gives $\cotinfo(\epsilon; \calH) \leq \gamma(\epsilon) / (1 - \gamma(\epsilon))$ which can be rearranged to give $\gamma(\epsilon) \geq \cotinfo(\epsilon; \calH) / (1 + \cotinfo(\epsilon; \calH))$. Finally, note that $\gamma(\epsilon) \geq \epsilon$ by definition since $\Lcot{\calD}(h) \geq \Lete{\calD}(h), \forall h$.
\end{proof}

Note that the restriction that the CoT-learning algorithm $\calA$ is \textit{proper} was crucial in the proof above. In particular, we used $\hhat \in \calH$ in order to derive the contradiction.

For a CoT hypothesis class $\calH \subset (\calZ \times \calY)^{\calX}$, recall that we define the CoT loss class over $\calX \times \calY \times \calZ \to \sset{0,1}$ as the $0-1$ class
\[\calLCoT(\calH) := \set{\ellcot_h: (x, y, z) \mapsto \Ind{h(x) \neq (y, z)}: \ h \in \calH}.\]
The complexity of this loss class will appear in our analysis since we will be analyzing learning algorithms that learn with respect to the CoT loss $\ellcot$.

We are now ready to prove the main result.

\begin{proof}[Proof of~\Cref{result:cotcons_cotinfo_infH}]
    By~\Cref{lemma:cot_ete_conversion}, a CoT learner $\calA$ with a sample complexity of $m_{\calA}^{\CoT}(\epsilon, \delta)$ with respect to the CoT error has a sample complexity with respect to the end-to-end error of at most $m_{\calA}^{\ete}(\epsilon, \delta) \leq m_{\calA}^{\CoT}(\gamma(\epsilon)^{-}, \delta)$, where $\gamma$ is defined in the lemma. The CoT-consistency rule enjoys a sample complexity of
    \[m_{\calA}^{\CoT}(\epsilon, \delta) = \calO\paren{\frac{1}{\epsilon} \cdot \paren{\VC(\calLCoT(\calH)) \cdot \log (1/\epsilon) + \log(1/\delta)}}.\]
    For the end-to-end error, this translates to the sample complexity of
    \begin{align*}
        m_{\calA}^{\ete}(\epsilon, \delta) &\leq m_{\calA}^{\CoT}(\gamma(\epsilon), \delta) \leq \calO\paren{\frac{1}{\gamma(\epsilon)} \cdot\paren{\VC(\calLCoT(\calH)) \cdot \log (1/\gamma(\epsilon)) + \log(1/\delta)}} \\
        &\leq \calO \paren{\frac{1 + \cotinfo(\epsilon; \calH)}{\cotinfo(\epsilon; \calH)} \cdot \paren{\VC(\calLCoT(\calH)) \cdot \log \paren{\frac{1 + \cotinfo(\epsilon; \calH)}{\cotinfo(\epsilon; \calH)}} + \log(1/\delta)}}\\
        &= \calO \paren{\paren{\frac{1}{\cotinfo(\epsilon; \calH)} + 1} \cdot \paren{ \VC(\calLCoT(\calH)) \cdot \log \paren{\frac{1}{\cotinfo(\epsilon; \calH)} + 1} + \log(1/\delta)}}.
    \end{align*}
\end{proof}

\aanote{Should we add any more discussion here? e.g., on the appearance of $\VC(\calLCoT(\calH))$? Also, the $+1$ factor in the sample complexity? Is it necessary/real or just an artifact of the analysis technique?}

\subsection{Proof of \texorpdfstring{\Cref{result:coterm_agnostic}}{Result~\ref{result:coterm_agnostic}}}\label{ssec:proof:result:coterm_agnostic}

\begin{result*}[Restatement: Agnostic Learning under CoT-Supervision]
    Let $\calH \subset (\calY \times \calZ)^{\calX}$ be a CoT hypothesis class. For any distribution $\calD$ over $\calX \times \calY \times \calZ$, the CoT-ERM learning rule has a sample complexity of
    % \[m(\epsilon, \delta) = \calO \paren{\max\paren{\frac{1}{\cotinfo(\epsilon; \calH)}, 1} \cdot \paren{ \VC(\calLCoT(\calH)) \cdot \log \max\paren{\frac{1}{\cotinfo(\epsilon; \calH)}, 1} + \log(1/\delta)}}.\]
    \[m(\epsilon, \delta) =  \calO \paren{\frac{\VC(\calLCoT(\calH)) + \log(1/\delta)}{\cotinfoag(\epsilon; \calH)^2}},\]
    where the agnostic version of the CoT information is defined as follows
    \[\cotinfoag(\epsilon; \calH) := \inf\set{\Lcot{\calD}(h) - \Lstarcot: h \in \calH, \Lete{\calD}(h) - \Lstarete \geq \epsilon},\]
    where $\Lstarcot := \inf_{h \in \calH} \Lcot{\calD}(h), \, \Lstarete := \inf_{h \in \calH} \Lete{\calD}(h)$.
    That is, for any $m \geq m(\epsilon, \delta)$, we have that with probability at least $1 - \delta$ over $S \sim \calD^m$, the excess end-to-end risk is bounded as
    \[\forall h \in \CoTERM(S; \calH), \ \text{we have} \ \Lete{\calD}(h) \leq \Lstarete + \epsilon.\]
\end{result*}

Our aim is to analyze the performance of the $\CoTERM$ learning rule, which seeks to minimize the CoT-penalized error. This is a natural learning rule to consider in the CoT-supervised setting, and corresponds to optimization procedures that are implemented in practice in CoT learning. Similar to the realizable setting, the key to proving this learning guarantee is to relate the CoT error of a CoT learner to its end-to-end error. This is established in the following lemma, which is an analogue of~\Cref{lemma:cot_ete_conversion}.

Recall that, for a distribution $\calD$ over $\calX \times \calY \times \calZ$ and a CoT hypothesis class $\calH: \calX \to \calY \times \calZ$, we define the optimal end-to-end and CoT errors achievable by $\calH$ as
\[\Lstarcot := \inf_{h \in \calH} \Lcot{\calD}(h), \ \Lstarete := \inf_{h \in \calH} \Lete{\calD}(h).\]

\begin{lemma}[Relating CoT performance to E2E performance in the Agnostic Setting]\label{lemma:cot_ete_conversion_agnostic}
    Any agnostic \textit{proper} CoT-learner $\calA: (\calX \times \calY \times \calZ)^* \to \calH$ which achieves \textit{excess} CoT-error $\epsilon$  with sample complexity $m_{\calA}^{\CoT}(\epsilon, \delta)$ also achieves \textit{excess} end-to-end error $\epsilon$ with sample complexity $m_{\calA}^{\ete}(\epsilon, \delta) \leq m_{\calA}^{\CoT}(\gamma(\epsilon)^{-}, \delta)$, where $\gamma: (0, 1) \to (0,1)$ is defined as
    \[\gamma(\epsilon) := \cotinfoag(\epsilon; \calH) = \inf\set{\Lcot{\calD}(h) - \Lstarcot: \ h \in \calH, \,\Lete{\calD}(h) - \Lstarete \geq \epsilon}.\]
    % \[\gamma(\epsilon) := \inf\set{\Lcot{\calD}(h): \Lete{\calD}(h) > \epsilon}.\]
    % Moreover, $\gamma$ can be related to the CoT information as follows
    % \[\max\paren{\frac{\cotinfo(\epsilon; \calH)}{1 + \cotinfo(\epsilon; \calH)}, \epsilon} \leq \gamma(\epsilon) \leq \min(\cotinfo(\epsilon; \calH), 1).\]
\end{lemma}
\begin{proof}
    Let $\calA$ be a proper CoT learner for $\calH$ (i.e., it returns a hypothesis in $\calH$) with CoT-error sample complexity $m_{\calA}^{\CoT}(\epsilon, \delta)$. Let $m \geq m_{\calA}^{\CoT}(\gamma(\epsilon)^{-}, \delta)$ and let 
    \[S = \sset{(x_1, y_1, z_1), \ldots, (x_m, y_m, z_m)} \simiid \calD\] be an i.i.d dataset drawn from the distribution $\calD$. By the assumption on the CoT-error sample complexity of $\calA$, we have that with probability at least $1 - \delta$, $\hhat = \calA(S)$ satisfies $\Lcot{\calD}(\hhat) < \Lstarcot + \gamma(\epsilon)$.

    To show that $\hhat$ has end-to-end error smaller than $\epsilon$, we proceed by contradiction. Suppose we are in the event $\sset{S: \Lcot{\calD}(\calS(S)) < \Lstarcot + \gamma(\epsilon)}$ and that the end-to-end error is larger than desired $\Lete{\calD}(\hhat) > \Lstarete + \epsilon$. This implies
    \[\gamma(\epsilon) := \inf\set{\Lcot{\calD}(h) - \Lstarcot: \ h \in \calH, \, \Lete{\calD}(h) \geq \Lstarete + \epsilon} \leq \Lcot{\calD}(\hhat) - \Lstarcot < \gamma(\epsilon).\]
    This yields a contradiction. Thus, on the event $\sset{S: \Lcot{\calD}(\calS(S)) < \Lstarcot + \gamma(\epsilon)}$, which occurs with probability at least $1 - \delta$, we must have $\Lete{\calD}(\calA(S)) \leq \Lstarete + \epsilon$.
\end{proof}

We can now prove our main result, which follows by a similar argument to~\Cref{result:cotcons_cotinfo_infH}.

\begin{proof}[Proof of~\Cref{result:coterm_agnostic}]
    By~\Cref{lemma:cot_ete_conversion_agnostic}, a CoT learner $\calA$ with a sample complexity of $m_{\calA}^{\CoT}(\epsilon, \delta)$ with respect to the CoT error has a sample complexity with respect to the end-to-end error of at most $m_{\calA}^{\ete}(\epsilon, \delta) \leq m_{\calA}^{\CoT}(\gamma(\epsilon)^{-}, \delta)$, where $\gamma(\epsilon) = \cotinfoag(\epsilon; \calH)$ is the agnostic version of the CoT information. The CoT-ERM rule enjoys a sample complexity of
    \[m_{\calA}^{\CoT}(\epsilon, \delta) = \calO\paren{\frac{1}{\epsilon^2} \cdot \paren{\VC(\calLCoT(\calH)) + \log(1/\delta)}}.\]
    For the end-to-end error, this translates to the sample complexity
    \[m(\epsilon, \delta) =  \calO \paren{\frac{\VC(\calLCoT(\calH)) + \log(1/\delta)}{\cotinfoag(\epsilon; \calH)^2}},\]
\end{proof}

\section{Proofs of \texorpdfstring{\Cref{sec:lower_bounds}}{Section~\ref{sec:lower_bounds}}: Lower Bounds}\label{sec:proofs_lower_bounds}

\subsection{Proof of \texorpdfstring{\Cref{result:cotinfo_lowerbound_twopoint}}{Result~\ref{result:cotinfo_lowerbound_twopoint}}}\label{ssec:proof:result:cotinfo_lowerbound_twopoint}

\aanote{Add a bit more of an introduction and discussion here?}
We will break down~\Cref{result:cotinfo_lowerbound_twopoint} into several statements and prove each separately.

\begin{result*}[First Part of~\Cref{result:cotinfo_lowerbound_twopoint}]
  Let $\calH \subset (\calY \times \calZ)^{\calX}$ be a CoT hypothesis class and let $\calD$ be a distribution on $\calX$. Let $x_1, \ldots, x_m \sim \calD$ be an i.i.d sample from $\calD$. For any $\hstar \in \calH$ and $\epsilon > 0$, we have that
  \[m < \frac{\log (1/\delta)}{\cotinfo(\epsilon; \calH)}\]
  implies that with probability at least $\delta$, there exists $h \in \calH$ with end-to-end error at least $\epsilon$ which is indistinguishable from $\hstar$ on this sample.
\end{result*}

\begin{proof}
    % Let us first prove the first part of the result, which states that when $m < \log (1/\delta) / \cotinfo(\epsilon; \calH)$ there exists an alternate hypothesis with error at least $\epsilon$ which is indistinguishable from $\hstar$ on the sample with probability at least $\delta$.
    
    Fix $\hstar \in \calH$ and $\epsilon \in [0, 1]$. Let $\bar{h} \in \argmin_{h \in \Deltaete_{\calD}(\epsilon; \calH, \hstar)} \relcotinfo(\hstar, h)$. Then, by definition, we have that $\bar{h}$ has end-to-end error at least $\epsilon$ and $\cotinfo(\epsilon; \calH) = \relcotinfo(\hstar, \bar{h})$. Thus, the probability that $\hstar$ and $\bar{h}$ agree on a random input $x$ with respect to both the CoT and E2E behavior can  be expressed as
    \[\probunder{x \sim \calD}{\hcot{\bar{h}}(x) = \hcot{\hstar}(x), \, \hete{\bar{h}}(x) = \hete{\hstar}(x)} = \exp(- \cotinfo(\epsilon; \calH)).\]
    Now, we compute the probability that $\hstar$ and $\bar{h}$ are indistinguishable on the CoT-annotated sample of $m$ points $x_1, \ldots, x_m \simiid \calD$:
    \begin{align*}
    &\probunder{x_1, \ldots, x_m \simiid \calD}{\forall i \in [m], \hcot{\bar{h}}(x_i) = \hcot{\hstar}(x_i), \, \hete{\bar{h}}(x_i) = \hete{\hstar}(x_i)} \\
    &= \paren{\probunder{x \sim \calD}{\hcot{\bar{h}}(x) = \hcot{\hstar}(x), \, \hete{\bar{h}}(x) = \hete{\hstar}(x)}}^m \\
    &= \exp(-m \cdot \cotinfo(\epsilon; \calH)).
    \end{align*}
    
    This occurs with probability at least $\delta$ when
    \[m \leq \frac{\log (1/\delta)}{\cotinfo(\epsilon; \calH)}.\]
\end{proof}

The next lower bound result is based on relating the learning problem to binary hypothesis testing and lower bounding the sample complexity of hypothesis testing via the total variation distance. The basic idea of relating the performance of a statistical estimator to the total variation distance is due to~\citet{lecamConvergenceEstimatesDimensionality1973}. We also point to~\citet{yuAssouadFanoCam1997} for a classic reference on statistical lower bounds, including Le Cam's method, as well as~\citet{polyanskiyInformationTheoryCoding2025} for a modern reference.
\aanote{add further discussion? any other references?}

Henceforth, for a hypothesis $h \in \calH$, we will denote by $P_h \in \calP(\calX \times \calY \times \calZ)$ the distribution over input-CoT-output tuples where the marginal on $\calX$ is the input distribution $\calD$ and the distribution over $(y,z)$ given $x$ is the Dirac measure at $h(x)$.

\begin{result*}[Second Part of~\Cref{result:cotinfo_lowerbound_twopoint}]
    Let $\calA: (\calX \times \calY \times \calZ)^* \to \calH$ be any learning algorithm that maps a dataset $S^m = \sset{(x_i, y_i, z_i)}_{i=1}^{m}$ to a predictor $\hhat$. Suppose there exists $h_1, h_2 \in \calH$ such that $\probunder{x \sim \calD}{\hete{h_1}(x) \neq \hete{h_2}(x)} \geq 2 \epsilon$. Assume that $S^m \sim (\calD \otimes \delta_{(y,z) = \hstar(x)})^{\otimes m}$ the sample size $m$ is upper bounded as 
    \[m < \frac{\log(\frac{1}{2\delta})}{\relcotinfo(h_1, h_2)}.\]
    Then, we must have
    \[\inf_{\hstar \in \calH} \probunder{S^m \sim \hstar}{\Lete{\calD,\hstar}(\calA(S^m)) \geq \epsilon} > \delta.\]
    Moreover, the expected error of any CoT-learning algorithm $\calA$ is lower-bounded as,
    \begin{align*}
        \sup_{\hstar \in \calH} \expectunder{S^m \sim \hstar}{\Lete{\calD, \hstar}(\calA(S^m))} 
        &\geq \frac{1}{2} \sup_{h_1, h_2 \in \calH} \probunder{x \sim \calD}{\hete{h_1}(x) \neq \hete{h_2}(x)} \cdot \exp(-m \cdot \relcotinfo(h_1, h_2)) \\
        &\geq \frac{1}{2} \sup_{\substack{\hstar \in \calH \\ \epsilon > 0}} \epsilon \cdot \exp(-m \cdot \cotinfo(\epsilon; \calH)).
    \end{align*}
\end{result*}
\begin{proof}
    Let us consider the pseudometric on the hypothesis space $\calH$, defined by 
    \[\dete(h_1, h_2) = \probunder{x \sim \calD}{\hete{h_1}(x) \neq \hete{h_2}(x)},\]
    which measures the end-to-end disagreement. Note that $\Lete{\calD}(h) = \dete(h, \hstar)$. Moreover, note that $\dete$ satisfies
    \[\dete(h_1, h_3) \leq \dete(h_1, h_2) + \dete(h_2, h_3).\]
    Let $\calA: (\calX \times \calY \times \calZ)^* \to \calY^{\calX}$ be any learning algorithm. Assume towards a contradiction that
    $\probunder{S \sim P_h^{\otimes m}}{\Lete{\calD}(\calA(S)) \geq \epsilon} \leq \delta, \forall h \in \calH.$
    By assumption, there exists a pair of hypotheses $h_1, h_2$ such that $\dete(h_1, h_2) \geq 2 \epsilon$. Consider the event that the predictor returned by $\calA$ is close to $h_0$ in end-to-end behavior, $\calE := \sset{\dete(h_0, \calA(S)) < \epsilon}$. We will consider the probability of this event when the data is generated by $h_1$ and $h_2$. By the assumption on the performance of the algorithm, we have that the probability of this event under $h_1$ is bounded as
    \[\probunder{S \sim P_{h_1}^{\otimes m}}{\dete(h_0, \calA(S)) < \epsilon} \geq 1 - \delta.\]
    On the other hand, under $S \sim P_{h_2}^{\otimes m}$ note that by the triangle inequality of $\dete$, we have
    \[\dete(h_1, h_2) \leq \dete(h_1, \calA(S)) + \dete(\calA(S), h_1) \iff \dete(h_0, \calA(S) \geq \underbrace{\dete(h_1, h_2)}_{\geq 2 \epsilon} - \underbrace{\dete(\calA(S), h_1)}_{\geq \epsilon \,\withprob\, \leq \delta},\]
    where the first bound is by the assumption on $h_1, h_2$ and the second is by the assumption on the performance of the learning algorithm $\calA$. This then implies that
    \[\probunder{S \sim P_{h_2}^{\otimes m}}{\dete(h_0, \calA(S)) < \epsilon} \leq \delta.\]
    By the definition of the total variation distance, this then implies that the total variation distance between $P_{h_1}^{\otimes m}$ and $P_{h_2}^{\otimes m}$ must be at least
    \[\TV{P_{h_1}^{\otimes m}, P_{h_2}^{\otimes m}} := \sup_{A} \abs{P_{h_1}^{\otimes m}(A) - P_{h_2}^{\otimes m}(A)} \geq P_{h_1}^{\otimes m}(\calE) - P_{h_2}^{\otimes m}(\calE) \geq 1 - 2 \delta.\]
    Thus, to derive a contradiction, we will relate the total variation distance to the relative CoT information and choose $m$ small enough such that $\TV{P_{h_1}^{\otimes m}, P_{h_2}^{\otimes m}} < 1 - 2 \delta$. We compute the total variation distance as follows:
        \begin{align*}
        \TV{P_{h_1}^{\otimes m}, P_{h_2}^{\otimes m}} &:= 
        \frac{1}{2} \sum_{(x_{1:m}, y_{1:m}, z_{1:m})} \abs{P_{h_1}(x_{1:m}, y_{1:m}, z_{1:m}) - P_{h_2}(x_{1:m}, y_{1:m}, z_{1:m})} \\
        &= \frac{1}{2}\sum_{x_{1:m}} \calD(x_{1:m}) \sum_{y_{1:m}, z_{1:m}} \abs{\Ind{(y_{1:m}, z_{1:m}) = h_1(x_{1:m})} - \Ind{(y_{1:m}, z_{1:m}) = h_2(x_{1:m})}} \\
        &\stepa{=} \sum_{x_{1:m}} \calD(x_{1:m})  \Ind{h_1(x_{1:m}) \neq h_2(x_{1:m})} \\
        &= \probunder{x_{1:m} \sim \calD^{\otimes m}}{\exists i \in [m]: h_1(x_i) \neq h_2(x_i)} \\
        &= 1 - \probunder{x \sim \calD}{h_1(x) = h_2(x)}^m \\
        &= 1 - \exp(- m \cdot \relcotinfo(h_1, h_2))
    \end{align*}
    To see step (a), note that the function
    \[\Delta\sset{h_1, h_2}(x_{1:m}, y_{1:m}, z_{1:m}):= \abs{\Ind{(y_{1:m}, z_{1:m}) = h_1(x_{1:m})} - \Ind{(y_{1:m}, z_{1:m}) = h_2(x_{1:m})}}\]
    takes the value $1$ either if $h_1(x_{1:m}) = (y_{1:m}, z_{1:m})$ and $h_2(x_{1:m}) \neq (y_{1:m}, z_{1:m})$ or if $h_2(x_{1:m}) = (y_{1:m}, z_{1:m})$ and $h_1(x_{1:m}) \neq (y_{1:m}, z_{1:m})$. Thus, in the sum $\sum_{y_{1:m}, z_{1:m}}$ we only need to consider values of $y_{1:m}, z_{1:m}$ that agree with at least one of $h_1, h_2$.
 
    To guarantee that $\TV{P_{h_1}, P_{h_2}} < 1 - 2 \delta$, it is enough to have
    \[m < \frac{\log(\frac{1}{2\delta})}{\relcotinfo(h_1, h_2)}.\]
    This proves the first statement. 

    The second statement, in terms of the expected error, can be proven by a an analogous argument and related to the CoT information via the above calculation of the TV distance. In particular, fix any $h_1, h_2 \in \calH$, and consider the predictor returned by the algorithm $\hhat = \calA(S)$. Convert this  predictor to a randomized test as follows
    \begin{align*}
        \htilde = \begin{cases}
            h_1 & \withprob\ \frac{\dete(h_2, \hhat)}{\dete(h_1, \hhat) + \dete(h_2, \hhat)}\\
            h_2 & \withprob\ \frac{\dete(h_1, \hhat)}{\dete(h_1, \hhat) + \dete(h_2, \hhat)}.
        \end{cases}
    \end{align*}
    Under $h_1$, we can lower bound the expected error via the triangle inequality as follows
    \[\expectunder{S \sim P_{h_1}^{\otimes m}}{\dete(\htilde, h_1)} 
    % = \dete(h_1, h_2) \probunder{S \sim P_{h_1}^{\otimes m}}{\htilde = h_2}
    = \dete(h_1, h_2) \expectunder{S \sim P_{h_1}^{\otimes m}}{ \frac{\dete(h_1, \hhat)}{\dete(h_1, \hhat) + \dete(h_2, \hhat)}} \leq \expectunder{S \sim P_{h_1}^{\otimes m}}{\dete(\hhat, h_1)},\]
    where we used the fact that $\dete(h_1, \hhat) + \dete(h_2, \hhat) \geq \dete(h_1, h_2)$. Similarly, under $h_2$, we have
    \[\expectunder{S \sim P_{h_2}^{\otimes m}}{\dete(\htilde, h_2)} 
    % = \dete(h_1, h_2) \probunder{S \sim P_{h_1}^{\otimes m}}{\htilde = h_2}
    = \dete(h_1, h_2) \expectunder{S \sim P_{h_2}^{\otimes m}}{ \frac{\dete(h_2, \hhat)}{\dete(h_1, \hhat) + \dete(h_2, \hhat)}} \leq \expectunder{S \sim P_{h_2}^{\otimes m}}{\dete(\hhat, h_1)}.\]
    Now, consider the prior $\pi = \frac{1}{2}(\delta_{h_1} + \delta_{h_2})$ and let $\hstar \sim \pi$. Then, we have
    \begin{align*}
        \sup_{\hstar \in \calH} \expectunder{S \sim P_{\hstar}^{\otimes m}}{\dete(\calA(S), \hstar)} &\geq \expectunder{\hstar \sim \pi}{\expectunder{S \sim P_{\hstar}^{\otimes m}}{\dete(\calA(S), \hstar)}} \\ 
        &= \frac{1}{2} \cdot \paren{\expectunder{S \sim P_{h_1}^{\otimes m}}{\dete(\calA(S), h_1)} + \expectunder{S \sim P_{h_2}^{\otimes m}}{\dete(\calA(S), h_2)}} \\
        &\geq \dete(h_1, h_2) \cdot \frac{1}{2} \cdot \paren{\probunder{S \sim P_{h_1}^{\otimes m}}{\htilde \neq h_1} + \probunder{S \sim P_{h_2}^{\otimes m}}{\htilde \neq h_2}} \\
        &\geq \frac{\dete(h_1, h_2)}{2} \cdot \paren{1 - \TV{P_{h_1}^{\otimes m}, P_{h_2}^{\otimes m}}},
    \end{align*}
    where the last inequality follows from the minimum average probability of error in binary hypothesis testing (or, equivalently, the supremum representation of the total variation distance).
    Now, using the previous calculation of the total variation distance in terms of the CoT information, we have
    \begin{align*}
    \sup_{\hstar \in \calH} \expectunder{S \sim P_{\hstar}^{\otimes m}}{\dete(\calA(S), \hstar)}
        &\geq\sup_{h_1, h_2 \in \calH} \probunder{x \sim \calD}{\hete{h_1}(x) \neq \hete{h_2}(x)} \cdot \exp(-m \cdot \relcotinfo(h_1, h_2)) \\
        &\geq \sup_{\hstar \in \calH} \sup_{\substack{\epsilon > 0 \\ h \in \Deltaete_{\calD}(\epsilon; \calH)}} \probunder{x \sim \calD}{\hete{\hstar}(x) \neq \hete{h}(x)} \cdot \exp(-m \cdot \relcotinfo(\hstar, h)) \\
        &\stepa{\geq} \sup_{\hstar \in \calH} \sup_{\epsilon > 0} \, \epsilon \cdot \exp\pparen{-m \cdot \inf_{h \in \Deltaete_{\calD}(\epsilon; \calH)} \relcotinfo(\hstar, h)} \\
        &\stepb{=} \sup_{\substack{\hstar \in \calH \\ \epsilon > 0}} \epsilon \cdot \exp\pparen{-m \cdot \cotinfo(\epsilon; \calH)}.
    \end{align*}
    In step (a) we used the fact that $\probunder{x \sim \calD}{\hete{\hstar}(x) \neq \hete{h}(x)} \geq \epsilon, \forall h \in h \in \Deltaete_{\calD}(\epsilon; \calH)$, and in step (b) we used the definition of the CoT information $\cotinfo(\epsilon; \calH) := \inf_{h \in \cotinfo(\epsilon; \calH)} \relcotinfo(h, \hstar)$.
\end{proof}

\subsection{Proof of \texorpdfstring{\Cref{result:lower_bound_fano}}{Result~\ref{result:lower_bound_fano}}}\label{ssec:proof:result:lower_bound_fano}

The next upper bound will use information-theoretic tools to establish a lower bound that scales with the size of the hypothesis space. As with the previous lower bound result, the strategy will be to reduce the learning problem into a hypothesis testing problem. However, unlike the previous results, which considered binary hypothesis testing, here we will consider a reduction to multiple hypothesis testing. The main idea is to test between a finite collection of hypotheses whose minimum end-to-end disagreement is $\epsilon$. If we can show that it is impossible to reliably distinguish between these hypotheses with a given sample size, then this implies that the best learning algorithm must at least incur an end-to-end error proportional to $\epsilon$.

To state our result, we begin by recalling the definition of an $\epsilon$-packing.
\begin{definition*} Let $\bbX$ be a set and let $d$ be a (pseudo)metric. A subset $\sset{x_1, \ldots, x_M} \subset \bbX$ is called an $\epsilon$-packing of $\bbX$ with respect to $d$ if $\min_{i \neq j} d(x_i, x_j) \geq \epsilon$. The packing number is defined as the size of the maximum packing, $M(\epsilon; \bbX, d) := \max\sset{m: \exists \ \text{$\epsilon$-packing of  $\bbX$ of size $m$}}$.
\end{definition*}

The main information-theoretic tool we will use will be Fano's inequality~\citep{yuAssouadFanoCam1997,coverElementsInformationTheory2006}, stated below.
\begin{lemma*}[Fano's Inequality]
    Let $W \to X \to Y \to \hat{W}$ be a Markov chain, and assume $W \sim \Unif([M])$. Then,
    \[P_e := \prob{W \neq \hat{W}} \geq 1 - \frac{I(X; Y) + h_b(P_e)}{\log M} \geq 1 - \frac{I(X; Y) + \log 2}{\log M}.\]
\end{lemma*}

Similar to the previous section, for $h \in \calH$, we will denote by $P_h$ the distribution on $\calX \times \calY \times \calZ$ induced by the hypothesis $h$. As before, the marginal over $\calX$ is $\calD$ for all $P_h$. However, unlike the previous section, $P_h(y, z \ggiven x)$ is not a Dirac measure since to first passes through the noisy channel $Q$. The distribution $P_h$ is defined as
\[P_h(x, y, z) = \calD(x) \cdot Q(y, z \ggiven h(x)).\]

\aawarning{What is the role or interpretation of replacing $I(\hstar; S)$ with $\sup_{\pi} I(\hstar; S)$? This gives us a looser bound? Is it more interpretable? Can it be related to $\cotinfo(\epsilon)$? It seems, no, not directly? Seems worth discussing. In the construction of the packing, the worst-case packing is one that maximizes $\dete$ (i.e., $\epsilon$-seperated) while minimizing $\cotinfo(h_1, h_2)$. Not all packings are equally hard.}
We are now ready to prove the result.
\begin{result*}[Restatement of~\Cref{result:lower_bound_fano}]
    Let $\calH \subset (\calY \times \calZ)^{\calX}$ be a CoT  hypothesis class and let $\calD$ be a distribution over $\calX$. Suppose that $x_1, \ldots, x_m \sim \calD$. Let $Q \in \calP(\calY \times \calZ \ggiven \calY \times \calZ)$ be a noisy channel from $h(x) = (y, z)$ to observations $\bar{y}, \bar{z}$. Let $C_Q = \max_{a,b} \KL{Q(\cdot \ggiven a)}{Q(\cdot \ggiven b)}$ be the capacity factor of the channel. The learner observes the noisy sample $S = \sset{(x_i, \bar{y}_i, \bar{z}_i)}_{i=1}^{m}$. Define the pseudo-metric $\dete(h_1, h_2) = \pprobunder{x}{\hete{h_1}(x) \neq \hete{h_2}(x)}$, and let $M(\epsilon; \calH, \dete)$ be the $\epsilon$-packing number of $\calH$ with respect to this pseudo-metric. Then, for any algorithm $\calA$ observing the CoT-supervised sample $S$ of size $m$, the probability of having large end-to-end error is lower bounded as 
    \begin{equation*}
        \sup_{\hstar \in \calH} \probunder{S \sim P_{\hstar}^{\otimes m}}{\Lete{\calD}(\calA(S)) \geq \frac{\epsilon}{2}} \geq 1 - \frac{m \cdot C_Q \cdot  \sup_\pi \expectunder{h_1, h_2 \sim \pi}{\relcotinfo(h_1, h_2)} + \log 2}{\log M(\calH, \dete, \epsilon)}.
    \end{equation*}
\end{result*}
\begin{proof}
    Let $\calH' := \sset{h_1, \ldots, h_M} \subset \calH$ be an $\epsilon$-packing of $\calH$ with respect to the end-to-end distance $\dete(h_1, h_2) = \pprobunder{x \sim \calD}{\hete{h_1}(x) \neq \hete{h_2}(x)}$, where $M = M(\epsilon; \calH, \dete)$. Consider the prior distributed uniformly on this packing, $\pi = \Unif(\sset{h_1, \ldots, h_M})$. For any learning algorithm $\calA: (\calX \times \calY \times \calZ)^* \to \calH$, consider the modified algorithm $\calA': (\calX \times \calY \times \calZ)^* \to \calH'$ which projects $\calA$ onto the packing $\calH'$. That is,
    \[\calA'(S) := \argmin_{h \in \calH'} \dete(\calA(S), h).\]
    The test error of $\calA'$ can be related to the test error of $\calA$ via the geometry of $\calH$ under the pseudometric $\dete$. In particular, letting $\hhat = \calA(S)$ and $\htilde = \calA'(S)$, we have that for all $h \in \calH'$, 
    \[\dete(h, \htilde) \leq \dete(h, \hhat) + \dete(\hhat, \htilde) \leq 2 \, \dete(h, \hhat),\]
    where we used $\dete(\hhat, \htilde) \leq \dete(h, \hhat), \forall h \in \calH'$, which follows by the definition of $\htilde$ as the projection of $\hhat$ onto $\calH'$. Thus, we have that, for any $h \in \calH'$, $\dete(h, \htilde) \geq \epsilon \implies \dete(h, \hhat) \geq \epsilon/2$. Also, note that $h \neq \htilde \implies \dete(h, \htilde) \geq \epsilon$. Thus, we have,
    \[\prob{h \neq \htilde} \leq \prob{\dete(h, \htilde) \geq \epsilon} \leq \prob{\dete(h, \hhat) \geq \epsilon/2}.\]
    By Fano's inequality, we can lower bound $\pprob{h \neq \htilde}$, which in turn implies the following lower bound on the probability of $\calA$ having large end-to-end error,
    \begin{align*}
        \probunder{\hstar \sim \pi, S \sim P_{\hstar}^{\otimes m}}{\dete(\hstar, \calA(S)) \geq \epsilon/2} &\geq \probunder{\hstar \sim \pi, S \sim P_{\hstar}^{\otimes m}}{\dete(\hstar, \calA'(S)) \geq \epsilon} \\
        &\geq \probunder{\hstar \sim \pi, S \sim P_{\hstar}^{\otimes m}}{\hstar \neq \calA'(S)} \\
        &\geq 1 - \frac{I(\hstar; S) + \log 2}{\log M(\epsilon; \calH, \dete)} \\
        &\geq 1 - \frac{\sup_{\pi \in \calP(\calH)} I(\hstar; S) + \log 2}{\log M(\epsilon; \calH, \dete)}.
    \end{align*}
    The first two inequalities are just restating the implications above, the third inequality is Fano's inequality, and the last inequality simply uses $I(\hstar; S) \leq \sup_{\pi} I(\hstar; S)$. That is, we bound the mutual information under the uniform prior over the packing, $\pi = \Unif(\calH')$, by the supremum of the mutual information over all priors (i.e., the capacity). 
    % Here, note that the mutual information $I(\hstar; S)$ is defined in terms of $\hstar \sim \Unif(\calH')$, whereas in the final inequality, we bound the mutual information in terms of the capacity of the channel from hypotheses $\hstar$ to CoT-annotated datasets $S$ by taking a supremum over priors $\pi$ on the full hypothesis space $\calH$.

    Now, we compute the mutual information $I(\hstar; S)$ and relate it to the CoT information. First, let $\bar{P} = \expectunder{h \in \pi}{P_h}$ denote the mixture of $P_h$ under $\pi$, and note that
    \[I(\hstar, S) = \expectunder{h \sim \pi}{\KL{P_h^{\otimes m}}{\bar{P}^{\otimes m}}} \leq \expectunder{h_1, h_2 \sim \pi}{\KL{P_{h_1}^{\otimes m}}{P_{h_2}^{\otimes m}}} = m \cdot \expectunder{h_1, h_2 \sim \pi}{\KL{P_{h_1}}{P_{h_2}}},\]
    where the inequality follows from the convexity of the KL divergence in the second argument and Jensen's inequality, and the last equality is by the chain rule for the KL divergence.

    Now, let us compute the KL divergence between the distributions induced by a pair of hypotheses and relate it to the relative CoT information between them. For convenience, let us fold in the output into the CoT and use a bold $\z$ to denote $\z = (y, z)$.
    \begin{align*}
        \KL{P_{h_1}}{P_{h_2}} &= \expectunder{x, \z \sim P_{h_1}}{\log \frac{P_{h_1}(x, \z)}{P_{h_2}(x, \z)}} 
        = \expectunder{x, \z \sim P_{h_1}}{\log \frac{\calD(x) Q(\z \ggiven h_1(x))}{\calD(x) Q(\z \ggiven h_2(x))}}\\
        &= \expectunder{x \sim \calD}{\sum_{\z \in \calY \times \calZ} Q(\z \ggiven h_1(x)) \log \frac{Q(\z \ggiven h_1(x))}{Q(\z \ggiven h_2(x))}} \\
        &= \expectunder{x \sim \calD}{\Ind{h_1(x) = h_2(x)}\sum_{\z \in \calY \times \calZ} Q(\z \ggiven h_1(x)) \log \frac{Q(\z \ggiven h_1(x))}{Q(\z \ggiven h_2(x))}} \\
        &\quad + \expectunder{x \sim \calD}{\Ind{h_1(x) \neq h_2(x)}\sum_{\z \in \calY \times \calZ} Q(\z \ggiven h_1(x)) \log \frac{Q(\z \ggiven h_1(x))}{Q(\z \ggiven h_2(x))}} \\
        &\stepa{=} \expectunder{x \sim \calD}{\Ind{h_1(x) \neq h_2(x)}\sum_{\z \in \calY \times \calZ} Q(\z \ggiven h_1(x)) \log \frac{Q(\z \ggiven h_1(x))}{Q(\z \ggiven h_2(x))}} \\
        &\stepb{\leq} \expectunder{x \sim \calD}{\Ind{h_1(x) \neq h_2(x)} \max_{\z_1, \z_2 \in \calY \times \calZ} \sum_{\z \in \calY \times \calZ} Q(\z \ggiven \z_1) \log \frac{Q(\z \ggiven \z_1)}{Q(\z \ggiven \z_2)}} \\
        &\stepc{=} C_Q \cdot \probunder{x \sim \calD}{h_1(x) \neq h_2(x)} \\
        &\stepd{\leq} C_Q \cdot - \log \probunder{x \sim \calD}{h_1(x) = h_2(x)}\\
        &\stepd{=} C_Q \cdot \relcotinfo(h_1, h_2).
    \end{align*}
    Step (a) follows by noting that $\sum_{\z} Q(\z \ggiven h_1(x)) \log \frac{Q(\z) \ggiven h_1(x)}{Q(\z \ggiven h_2(x))} = 0$ when $h_1(x) = h_2(x)$. Steps (b) and (c) are simply bounding the KL divergence between the observations under two different hypotheses that differ by the capacity of the channel, $C_Q := \max_{\z_1, \z_2} \KL{Q(\cdot \ggiven \z_1)}{Q(\cdot \ggiven \z_2)}$. Step (d) uses the identity $x \leq - \log (1 -x)$, and step (d) is the definition of the relative CoT information between two hypotheses.  

    Plugging this into the previous bound proves the result.
    \begin{align*}
    \sup_{\hstar \in \calH} \probunder{S \sim P_{\hstar}^{\otimes m}}{\Lete{\calD}(\calA(S)) \geq \frac{\epsilon}{2}} 
    &\geq \probunder{\hstar \sim \pi, S \sim P_{\hstar}^{\otimes m}}{\dete(\hstar, \calA(S)) \geq \epsilon/2} \\
    % &\geq \probunder{\hstar \sim \pi, S \sim P_{\hstar}^{\otimes m}}{\dete(\hstar, \calA'(S)) \geq \epsilon} \\ 
    &\geq 1 - \frac{m \cdot C_Q \cdot  \sup_\pi \expectunder{h_1, h_2 \sim \pi}{\relcotinfo(h_1, h_2)} + \log 2}{\log M(\calH, \dete, \epsilon)}.
\end{align*}

\end{proof}

\aawarning{An alternate approach is to bound the capacity as follows, using the interpretation of the capacity as a KL-radius: 
$C(\calH) := \sup_\pi I(\hstar; S) = \inf_Q \sup_{h \in \calH} \KL{P_h^{\otimes m}}{Q} \leq \sup_{h \in \calH} \KL{P_h^{\otimes m}}{Q}$, for any choice of $Q$. Is this helpful? Can this be used to relate this $\cotinfo(\epsilon)$?
}

In particular, this result implies that when
\[m \leq \frac{\log M(\calH, \dete, \epsilon)}{2 \cdot \paren{C_Q \cdot  \sup_\pi \expectunder{h_1, h_2 \sim \pi}{\relcotinfo(h_1, h_2)} + \log 2}},\]
the probability that the error is more than $\epsilon/2$ is at least $1/2$,
\[\sup_{\hstar \in \calH} \probunder{S \sim P_{\hstar}^{\otimes m}}{\Lete{\calD}(\calA(S)) \geq \frac{\epsilon}{2}}  \geq \frac{1}{2}.\]

\aawarning{Are there more interesting things to be said about the role of the channel $Q$? E.g., what types of things, unique to the CoT setting it can model?}

Now, let us discuss the role of the noisy channel in the setting of this result. The noisy channel models noise in the learning process. For example, errors in human-created CoT annotations, errors in the output labels, or any other type of noise. For simplicity, let us consider a symmetric channel over $\calY \times \calZ$ parameterized by an error level $e$ and defined as 
\[Q(\cdot \ggiven \z) = (1 - e) \cdot \delta_x + e \cdot \Unif(\calY \times \calZ).\]
We can compute the channel capacity factor $C_Q$ for this channel. Let $\x \neq \y, \x, \y \in \calY \times \calZ$ be two different symbols to be transmitted through the channel, and suppose the error level is non-zero, $e \in (0, 1]$. For convenience, denote $\abs{\calY \times \calZ} = N$.
\begin{align*}
&\KL{Q(\cdot \ggiven \x)}{Q(\cdot \ggiven \y)} \\
&= \sum_{\z} Q(\z \ggiven \x) \log \frac{Q(\z \ggiven \x)}{Q(\z \ggiven \y)} \\
&= Q(\x \ggiven \x) \log \frac{Q(\x \ggiven \x)}{Q(\x \ggiven \y)} + Q(\y \ggiven \x) \log \frac{Q(\y \ggiven \x)}{Q(\y \ggiven \y)} + \sum_{\z \neq \x, \z \neq \y} Q(\z \ggiven \x) \log \frac{Q(\z \ggiven \x)}{Q(\z \ggiven \y)} \\
&= \left(1 - e + \frac{e}{N}\right) \log \frac{1 - e + e/N}{e/N} + \left(\frac{e}{N}\right) \log \frac{e/N}{1 - e + e/N} + \sum_{\z \neq \x, \z \neq \y} \left(\frac{e}{N}\right) \log \frac{e/N}{e/N} \\
% &= \left(1 - e + \frac{e}{N}\right) \log \left( \frac{1 - e + e/N}{e/N} \right) + \left(\frac{e}{N}\right) \log \left( \frac{e/N}{1 - e + e/N} \right) + \sum_{\z \neq \x, \z \neq \y} \left(\frac{e}{N}\right) \log(1) \\
% &= \left(1 - e + \frac{e}{N}\right) \log \left( \frac{1 - e + e/N}{e/N} \right) + \left(\frac{e}{N}\right) \log \left( \left[ \frac{1 - e + e/N}{e/N} \right]^{-1} \right) \\
&= \left(1 - e + \frac{e}{N}\right) \log \left( \frac{1 - e + e/N}{e/N} \right) - \left(\frac{e}{N}\right) \log \left( \frac{1 - e + e/N}{e/N} \right) \\
% &= \left( \left(1 - e + \frac{e}{N}\right) - \frac{e}{N} \right) \log \left( \frac{1 - e + e/N}{e/N} \right) \\
% &= (1 - e) \log \left( \frac{1 - e + e/N}{e/N} \right) \\
&= (1 - e) \log \left(1 + \frac{N(1-e)}{e} \right)
\end{align*}

Intuitively, this is a decreasing function in the error level $e$, decreasing towards $0$ as $e \to 1$. This corresponds to the fact that it is harder to distinguish between hypotheses when the observations are more noisy. For $e = 1/ 100$ and $\aabs{\calY \times \calZ} = 1,000$, the capacity factor is approximately $C_Q \approx 11.39$.

\begin{figure}[h]
    \centering
    \begin{tikzpicture}
    \begin{axis}[ % Start of axis options
        % --- Declare function ---
        declare function={ N=1000; }, 
        % --- Axis options ---
        % title={Plot ...}, % Optional title
        xlabel={$e$},                                     % X-axis label
        ylabel={$C_Q$},                                   % Y-axis label
        xmode=log,                                        % Logarithmic x-axis
        xmin=0.0001,                                       % Min x-value
        xmax=1,                                           % Max x-value
        ymin=0,                                           % Min y-value 
        domain=0.0001:0.999,                               % Calculation domain
        samples=150,                                      % Number of points
        % --- Grid setting updated ---
        grid=both,                                        % Show both major and minor grid lines << CHANGE HERE
        % minor tick num=9,                                 % Explicitly sets 9 minor ticks between major log ticks (usually default)
        % --- Other options ---
        legend pos=north east,                            % Legend position
        axis lines=left,                                  % Axis lines style
        width=4in,                                       % Plot width
        height=2.5in,                                       % Plot height
    ] % End of axis options
        
        % Plot the function
        \addplot [RoyalBlue, thick] {(1-x) * ln(1 + N*(1-x)/x)}; 

        % Add legend entry
        % \addlegendentry{$N = \pgfmathprintnumber{N}$} 
        
    \end{axis}
\end{tikzpicture}
    \caption{Capacity factor of $Q$ as a function of error level for $\aabs{\calY \times \calZ} = 1, 000$}
    \label{fig:capacity_factor}
\end{figure}
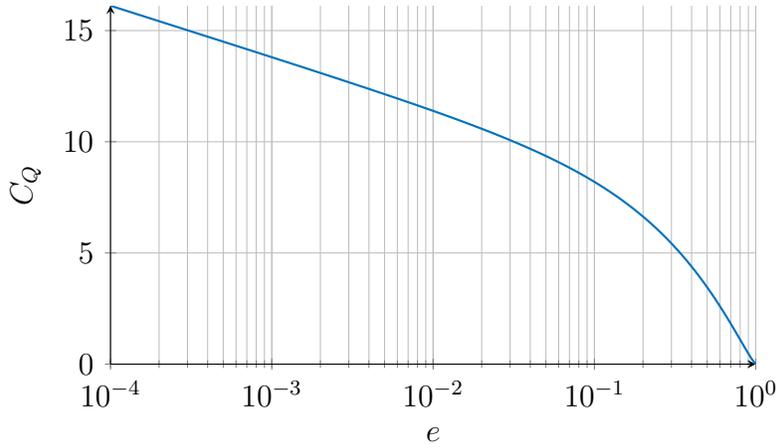
\aanote{I'm including this figure just to give us an idea of what $C_Q$ looks like. Not sure if this is actually necessary to include in the paper?}

\section{Other Topics}\label{sec:other_topics}

\subsection{Learning with Mixed Supervision}\label{ssec:mixed_supervision}
In many situations, CoT training examples may be difficult or expensive to obtain, for example, because they require manual human annotation. On the other hand, input-output examples without CoT annotation might be much more readily available. In such cases, one might have a dataset that includes a large number of end-to-end input-output examples and a small number of CoT-annotated examples. What types of learning guarantees can we establish in such a setting?

The following result, an extension of~\Cref{result:cotinfo_ete_learning}, analyzes the sample complexity of the consistency rule when applied to $m_{\ete}$ end-to-end examples and $m_{\CoT}$ CoT examples.

\begin{result}[Learning with Mixed Datasets]\label{result:cotinfo_ete_learning_mixed}
  Let $\calH \subset (\calZ \times \calY)^{\calX}$ be a finite CoT hypothesis class, and let $\calD$ be a distribution over $\calX$. Consider an i.i.d dataset, $S = S_{\ete} \cup S_{\cot}$, consisting of $m_{\ete}$ input-output examples and $m_{\cot}$ CoT-annotated examples
  \[S_{\ete} = \sset{(x_i^{\ete}, \hete{\hstar}(x_i^\ete))}_{i \in m_{\ete}}, S_{\cot} = \sset{(x_i^{\ete}, \hcot{\hstar}(x_i^{\cot}), \hete{\hstar}(x_i^{\cot}))}_{i \in m_{\cot}},\]
  where $x_1^{\ete}, \ldots, x_{m_\ete}^{\ete}, x_1^{\cot}, \ldots, x_{m_{\cot}}^{\cot} \simiid \calD$. Suppose the number of end-to-end examples is $\gamma$ times the number of CoT examples, so $m_{\ete} = \gamma \cdot m_{\cot}$ and let $m_{\cot} \equiv m$.
  Then, the $\ete-\CoT$-consistency rule has sample complexity with respect to the end-to-end error of
  \begin{equation*}
    m(\epsilon, \delta) = \frac{\log \abs{\calH} + \log(1 / \delta)}{\gamma \cdot \epsilon + \cotinfo(\epsilon; \calH)}.
  \end{equation*}
  That is, for $m \geq m(\epsilon, \delta)$, with probability at least $1 - \delta$ over the draw of $S = S_{\ete} \cup S_{\cot}$,
  \begin{equation*}
    \forall h \in \calH \suchthat \empiricaletecotrisk{S}(h) = 0, \ \text{we have} \ \Lete{\calD}(h) \leq \epsilon.
  \end{equation*}
\end{result}
\begin{proof}
  We would like to bound the probability of the bad event
  \[\sset{\exists h \in \calH : \Lete{\calD}(h) > \epsilon,\, \empiricaletecotrisk{S}(h) = 0}\]
  over $x_1^{\ete}, \ldots, x_{m_\ete}^{\ete}, x_1^{\cot}, \ldots, x_{m_{\cot}}^{\cot} \simiid \calD$.
  Fix any $h \in \calH$ with end-to-end error $\Lete{\calD}(h) > \epsilon$ (i.e., $h \in \Deltaete_{\calD}(\epsilon; \calH, \hstar)$). We bound the probability that $h$ is consistent with $S = S_{\ete} \cup S_{\cot}$, $h \in \CoTCons(S; \calH) = \sset{h \in \calH : \empiricaletecotrisk{S}(h) = 0}$, as follows
  \begin{align*}
    &\probunder{S = S_{\ete} \cup S_{\cot}}{h \in \CoTCons(S; \calH)} \\ %&= \probunder{S = S_{\ete} \cup S_{\cot}}{ \empiricaletecotrisk{S}(h) = 0} \\
    &= \probunder{S_{\ete}}{\forall i \in [m_{\ete}],\ \hete{h}(x_i^{\ete}) = \hete{\hstar}(x_i^{\ete})} \\
    &\qquad \cdot \probunder{S_{\cot} }{\forall i \in [m_{\cot}],\ \hcot{h}(x_i^{\cot}) = \hcot{\hstar}(x_i^{\cot}),\,\hete{h}(x_i^{\cot}) = \hete{\hstar}(x_i^{\cot})} \\
    &= \probunder{x \sim \calD}{\hete{h}(x_i) = \hete{\hstar}(x)}^{m_{\ete}} \cdot \probunder{x \sim \calD}{\hcot{h}(x) = \hcot{\hstar}(x_i),\, \hete{h}(x_i) = \hete{\hstar}(x)}^{m_{\cot}} \\
    &\stepa{\leq} \paren{1 - \epsilon}^{m_{\ete}} \cdot \paren{\exp\paren{- \relcotinfo(\hstar, h)}}^{m_{\cot}} \\
    &\stepb{\leq} \exp\paren{- m_{\ete} \cdot \epsilon} \cdot \paren{\exp\paren{- \relcotinfo(\hstar, h)}}^{m_{\cot}} \\
    &\stepc{\leq} \exp\paren{- m_{\ete} \cdot \epsilon} \cdot \exp\paren{-m_{\cot} \cdot \cotinfo(\epsilon; \calH)} \\
    &= \exp\paren{- m_{\ete} \cdot \epsilon - m_{\cot} \cdot \cotinfo(\epsilon; \calH)},
  \end{align*}
  where step (a) applies $\Lete{\calD}(h) > \epsilon$ (since $h \in \Deltaete_{\calD}(\epsilon; \calH, \hstar)$) for the first factor and uses the definition of $\relcotinfo(\hstar, h)$ for the second factor, step (b) is the identity $\log(1 - x) \leq -x$ for $x \in (0,1)$, and step (c) is by the definition of $\cotinfo(\epsilon; \calH)$ and the fact that $h \in \Deltaete_{\calD}(\epsilon; \calH, \hstar)$.

  Now, we write $m_{\ete} = \gamma \cdot m$, $m_{\cot} = m$, and set $m = \frac{\log \aabs{\calH} + \log (1/\delta)}{\gamma \cdot \epsilon + \cotinfo(\epsilon; \calH)}$. This guarantees that the probability that any fixed hypothesis with error larger than $\epsilon$ is consistent with $S = S_{\ete} \cup S_{\cot}$ is at most
  \[\probunder{S = S_{\ete} \cup S_{\cot}}{h \in \CoTCons(S; \calH)} \leq \frac{\delta}{\abs{\calH}}.\]
  Applying a union bound over $\calH$ then shows that
  \[\probunder{S = S_{\ete} \cup S_{\cot}}{\exists h \in \calH : \Lete{\calD}(h) > \epsilon,\, \empiricaletecotrisk{S}(h) = 0} \leq \delta.\]

\end{proof}

\subsection{CoT Learning with Inductive Priors}\label{ssec:inductive_priors}
Encoding prior knowledge about solution structure is critical for learning complex functions, such as those representing multi-step reasoning processes, particularly from limited data. This is especially relevant in chain-of-thought settings, which are often applied to learning such reasoning processes. The chain-of-thought trajectories can be viewed as additional supervision for the intermediate steps of an algorithm, helping to align the learner to the ground-truth algorithm, which may otherwise be very difficult to learn if only input-output examples are observed.

A key idea in learning algorithms is the so-called \textit{Minimum Description Length} (MDL) principle. This encodes a prior or an inductive bias in the learner that favors hypotheses that have a small description length (e.g., thought of as the length of program code or the size of a Turing machine's state space). This is also related to the notion of algorithmic complexity in algorithmic information theory~\citep{solomonoff1964formal,kolmogorov1965three}.

In this section, we consider a minimum description length type of learning rule for the chain-of-thought setting. This also provides an extension of~\Cref{result:cotinfo_ete_learning} to CoT hypothesis classes that are \textit{countably infinite}. 

For a prior $p$ over a hypothesis class $\calH$, we define the \textit{chain-of-thought} MDL rule corresponding to the prior $p$ as
\begin{equation}\label{eq:aut_MDL}
  \MDL_p^{\CoT}(S; \calH) = \argmax_{h \in \CoTCons(S; \calH)} p(h).
\end{equation}
That is, given a CoT dataset $S$, $\MDL_p^{\CoT}(S; \calH)$ selects the hypothesis that maximizes the prior $p$ among hypotheses that are CoT-consistent with $S$. One way to define such a prior $p$ over $\calH$ is through a \textit{prefix-free description language} $d: \calH \to \sset{0,1}$ which maps a hypothesis to a bitstring description. The prior $p$ can then be defined as $p(h) = 2^{- |d(h)|}$, in which case $\MDL_p^{\CoT}(S; \calH) = \argmin_{h \in \CoTCons(S; \calH)} |d(h)|$ selects the minimum-description CoT-consistent hypothesis.

The following result provides a learning guarantee for $\MDL_p^{\CoT}$ in terms of the likelihood of $\hstar$ under the prior $p$ and the CoT-information metric $\cotinfo(\epsilon; \calH)$.

\begin{result}[Learning with Chain-of-Thought and MDL]\label{result:cotinfo_ete_learning_mdl}
  Let $\calH$ be a \textit{countable} autoregressive hypothesis class and consider a prior $p$ over $\calH$. Let $S$ be an i.i.d. dataset of $m$ examples drawn from a distribution $\calD$ over $\calX$. Then, with probability at least $1-\delta$ over the draw of $S$, any CoT-consistent hypothesis $h \in \CoTCons(S; \calH)$ has its end-to-end error bounded by
  \[\Lete{\calD}(h) \leq \epsilon_h := \inf \set{\epsilon > 0: \cotinfo(\epsilon; \calH) \geq \frac{\log (1 / p(h)) + \log(1 / \delta)}{m}}\]
  This in turn implies a sample complexity with respect to the \textit{end-to-end} error for the autoregressive MDL rule $\MDL_p^{\CoT}$ of
  \[m(\epsilon, \delta) = \calO\paren{\frac{\log \frac{1}{p(\hstar)} + \log\frac{1}{\delta}}{\cotinfo(\epsilon; \calH)}}.\]
  That is, for $m \geq m(\epsilon, \delta)$, with probability at least $1 - \delta$ over $S = \sset{\x_1, \ldots, \x_m} \simiid \calD$,
  \begin{equation*}
    \forall h \in \MDL_p^{\CoT}(S; \calH) \ \text{we have} \ \Lete{\calD}(h) \leq \epsilon.
  \end{equation*}
\end{result}
\begin{proof}
  Following the argument in the proof of~\Cref{result:cotinfo_ete_learning}, for any fixed $h \in \calH$ and any $\epsilon_h > 0$, the probability that $h$ is CoT-consistent on $S$, $h \in \CoTCons(S; \calH) = \sset{h \in \calH : \empiricalcotrisk{S}(h) = 0}$ yet has end-to-end error $\Lete{\calD}(h) > \epsilon_h$, is bounded by
  \begin{align*}
    \probunder{S \sim \calD^{\otimes m}}{\empiricalcotrisk{S}(h) = 0, \Lete{\calD}(h) > \epsilon_h} \leq \exp(- m \cdot \relcotinfo(\hstar, h)) \leq \exp(- m \cdot \cotinfo(\epsilon_h; \calH)).
  \end{align*}
  For each $h \in \calH$, we will target an end-to-end error $\epsilon_h$ in a prior-dependent manner such that
  \[\probunder{S \sim \calD^{\otimes m}}{\empiricalcotrisk{S}(h) = 0, \Lete{\calD}(h) > \epsilon_h} \leq \delta p(h).\]
  This occurs if $\epsilon_h$ is chosen such that
  \begin{align*}
    &\exp(- m \cdot \cotinfo(\epsilon_h; \calH)) \leq \delta p(h) \\
    \iff& \cotinfo(\epsilon_h; \calH) \geq \frac{\log (1 / p(h)) + \log(1 / \delta)}{m}.
  \end{align*}
  Thus, we define the target error $\epsilon_h$ for $h$ as
  \begin{equation*}
    \epsilon_h := \inf \set{\epsilon > 0: \cotinfo(\epsilon; \calH) \geq \frac{\log (1 / p(h)) + \log(1 / \delta)}{m}}.
  \end{equation*}
  Note that this can be viewed in terms of the generalized inverse $(\cotinfo(\cdot; \calH))^{-}$, where the generalized inverse of a function $f$ is defined as $f^-(x) = \inf\sset{y: f(y) > x}$. This is well-defined since $\cotinfo(\cdot; \calH)$ is an increasing function by~\Cref{lemma:cotinfo_props}.

  Now, by a union bound, we have that the probability that \textit{any} CoT-consistent hypothesis exceeds its target end-to-end error is bounded by
  \begin{align*}
    &\probunder{S \sim \calD^{\otimes m}}{\exists h \in \calH: \empiricalcotrisk{S}(h) = 0, \Lete{\calD}(h) > \epsilon_h} \\
    &\leq \sum_{h \in \calH} \probunder{S}{h \in \CoTCons(S; \calH), \Lete{\calD}(h) > \epsilon_h} \\
    &\leq \sum_{h \in \calH} \delta p(h) = \delta.
  \end{align*}
  Since $\hstar \in \CoTCons(S; \calH)$, we have that $p(\hstar) \leq p(\MDL_p^{\CoT}(S; \calH))$ by definition of $\MDL_p^{\CoT}$. This, in turn, implies that
  \[\Lete{\calD}(\MDL_p^{\CoT}(S; \calH)) \leq \inf \set{\epsilon > 0: \cotinfo(\epsilon; \calH) \geq \frac{\log (1 / p(\hstar)) + \log(1 / \delta)}{m}},\]
  with probability at least $1 - \delta$. Noting the property that $f(x) \geq y \implies x \geq f^-(y)$ for generalized inverses, we obtain a sample complexity of
  \[m(\epsilon, \delta) = \calO\paren{\frac{\log \frac{1}{p(\hstar)} + \log\frac{1}{\delta}}{\cotinfo(\epsilon; \calH)}}.\]
\end{proof}

\begin{corollary}
  If the prior $p$ is defined as $p(h) = 2^{- \aabs{d(h)}}$ in terms of a prefix-free description language $d: \calH \to \sset{0,1}^*$ satisfying Kraft's inequality $\sum_h 2^{-\aabs{d(h)}} \leq 1$, then the sample complexity of autoregressive MDL rule $\MDL_p^{\CoT}$ satisfies
  \[m(\epsilon, \delta) = \calO\paren{\frac{\aabs{d(\hstar)} + \log\frac{1}{\delta}}{\cotinfo(\epsilon; \calH)}}.\]
\end{corollary}

One advantage of such MDL-style analysis is obtaining a sample complexity that is instance-dependent (i.e., $\hstar$-dependent). In the standard end-to-end setting, we obtain a sample complexity of $\calO(\log (1/p(\hstar)) / \epsilon)$. In the CoT setting, we obtain an instance-dependent description of the sample complexity in both the numerator and denominator through the CoT information $\cotinfo(\epsilon; \calH)$.

\subsection{Transfer Learning \& Out-of-Distribution Generalization Under CoT Supervision}\label{ssec:transfer_learning}

In many applications where chain-of-thought learning is applied, \textit{out-of-distribution generalization} is a key aspect. This type of generalization requires compositional reasoning abilities, for example, learning a set of generally applicable atomic skills that can be recombined to generalize systematically to novel combinations of known elements. Of particular interest is learning from simple problem instances and generalizing to larger and more complex instances. This is sometimes referred to as \textit{length-generalization} when the notion of increased complexity corresponds to longer inputs.

Chain-of-thought learning has important implications for this type of generalization. In particular, it allows direct supervision on the ``atomic skills'' and how to combine them to solve problems. This can, in principle, enable systematic generalization by transforming the learning problem from one of learning input-output patterns to one of learning general principles that can be applied beyond the training distribution. 

In this section, we explore how the CoT information measure can be extended to analyze generalization under distribution shift.

% \subsection{Relative CoT Information Between a Pair of Distributions}

The following generalized definition of CoT information captures the amount of information revealed about the end-to-end behavior on a test distribution $\calD_{\test}$ from observing a CoT-annotated sample drawn from a training distribution $\calD_{\tr}$.

\begin{definition}[Relative CoT Information Between a Pair of Distributions]\label{def:cotinfo_reldist}
    For a CoT hypothesis class $\calH \subset (\calZ \times \calY)^{\calX}$, we define the relative CoT-Information between a distribution $\calD_\tr$ and $\calD_\test$ as
    \begin{align*}
      \cotinfodomain{\calD_{\tr} \to \calD_{\test}}(\epsilon; \calH, \hstar) &= \inf_{h \in \Deltaete_{\calD_{\test}}(\epsilon; \calH, \hstar)} \set{ - \log \probunder{x \sim \calD_{\tr}}{\hcot{h}(x) = \hcot{\hstar}(x), \hete{h}(x) = \hete{\hstar}(x)}}.% \ \cotinfo(\epsilon; \calH) = \min_{h \in \calH} \cotinfo_{\calD}(\epsilon; \calH, h).
    \end{align*}
    where the infimum is over $\Deltaete_{\calD_{\test}}(\epsilon; \calH, \hstar)$, the set of hypotheses that disagree with $\hstar$'s end-to-end behavior (i.e., output) on the test distribution with probability greater than $\epsilon$,
    \[\Deltaete_{\calD_{\test}}(\epsilon; \calH, \hstar) := \set{h \in \calH: \probunder{x \sim \calD_{\test}}{\hete{\hstar}(x) \neq \hete{h}(x)} > \epsilon}.\]
\end{definition}

Unlike the in-distribution setting (where $\calD_{\test} = \calD_{\tr}$), $\cotinfodomain{\calD_{\tr} \to \calD_{\test}}(\epsilon; \calH, \hstar) \geq \epsilon$ does not necessarily hold for arbitrary CoT hypothesis classes and pairs of distributions. However, we do have
\[\cotinfodomain{\calD_{\tr} \to \calD_{\test}}(\epsilon; \calH, \hstar) \geq \inf_{h \in \Deltaete_{\calD_{\test}}(\epsilon; \calH, \hstar)} \set{ \Lete{\calD_{\tr}}(h)}.\]

With strong CoT supervision on a sufficiently diverse training distribution, we would expect the relative CoT information to be large.

Analogous results to those presented in the main text in terms of the (standard) CoT information also apply in the transfer learning setting via the relative CoT information measure between two distributions defined above. For example, in the finite hypothesis class case, we have the following analogue of~\Cref{result:cotinfo_ete_learning}.

\begin{result}[Transfer Learning with Chain-of-Thought Supervision]\label{results:cotinfo_ete_transferlearning}
    For any finite CoT class $\calH$ and distributions $\calD_{\tr}$ (training) and $\calD_{\test}$ (test) over $\calX$,, the CoT consistency rule has sample complexity with respect to the $\calD_{\test}$-end-to-end error of
    \begin{equation*}
      m(\epsilon, \delta) = \frac{\log \abs{\calH} + \log(1 / \delta)}{\cotinfodomain{\calD_{\tr} \to \calD_{\test}}(\epsilon; \calH, \hstar)}.
    \end{equation*}
    That is, for $m \geq m(\epsilon, \delta)$, we have that with probability at least $1 - \delta$ over $S = \sset{x_1, \ldots, x_m} \simiid \calD_{\tr}$,
    \begin{equation*}
      \forall h \in \CoTCons(S; \calH), \ \text{we have} \ \Lete{\calD_{\test}}(h) \leq \epsilon.
    \end{equation*}
  \end{result}
  \begin{proof}
    % We would like to bound the probability of the bad event
    % \[\sset{\exists h \in \calH : \Lete{\calD_{\test}}(h) = \probunder{x \sim \calD_{\test}}{\hete{h}(x) \neq \ete(\hstar)(x)} > \epsilon,\, \empiricalcotrisk{S}(h) = 0}\]
    % over $(x_1, \ldots, x_m) \simiid \calD_{\tr}$. Here, the learner uses \textit{CoT-annotated} samples from the training distribution $\calD_{\tr}$ while the evaluation metric is with respect to the \textit{end-to-end} error on a different \textit{test} distribution $\calD_{\test}$.
    Fix any $h \in \calH$ with end-to-end error under $\calD_{\test}$ larger than $\epsilon$, i.e., $\Lete{\calD_{\test}}(h) > \epsilon$ (so $h \in \Deltaete_{\calD_{\test}}(\epsilon; \calH, \hstar)$). We bound the probability that $h$ is CoT-consistent on $S$, $h \in \CoTCons(S; \calH) = \sset{h \in \calH : \empiricalcotrisk{S}(h) = 0}$, as follows
    \begin{align*}
      \probunder{S \sim \calD_{\tr}^{\otimes m}}{h \in \CoTCons(S; \calH)} &= \probunder{S \sim \calD_{\tr}^{\otimes m}}{\forall i, \hcot{h}(x_i) = \hcot{\hstar}(x_i), \hete{h}(x_i) = \hete{\hstar}(x_i)} \\
      &= \probunder{x \sim \calD_{\tr}}{\hcot{h}(x) = \hcot{\hstar}(x_i),\, \hete{h}(x_i) = \hete{\hstar}(x)}^m \\
      &\leq \exp\paren{-m \cdot \cotinfodomain{\calD_{\tr} \to \calD_{\test}}(\epsilon; \calH, \hstar)},
    \end{align*}
    where we use the definition of the relative CoT information and the fact that $h \in \Deltaete_{\calD_{\test}}(\epsilon; \calH, \hstar)$.

    Choosing $m \geq m(\epsilon, \delta)$ in the theorem statement guarantees that
    \[\probunder{S \sim \calD_{\tr}^{\otimes m}}{\exists h \in \calH : \Lete{\calD_{\test}}(h) > \epsilon,\, \empiricalcotrisk{S}(h) = 0} \leq \delta.\]
\end{proof}

An analogous result also holds for the infinite hypothesis class setting.

\begin{example}[Length-Generalization in Finite-State Machines]
    \Cref{example:fsm_cotinfo_lowerbound} shows that
    \[\min_{\epsilon > 0} \cotinfo(\epsilon; \calH) \geq \abs{\Sigma}^{-(\ell + 1)}\]
    for the class of finite-state machines when $\hstar$'s transition graph is $\ell$-connected and $\calD$ is a uniform distribution over inputs of length $n \geq \ell$. In fact, the same line of reasoning shows that
    \[\min_{\epsilon > 0} \cotinfodomain{\calD_{\tr} \to \calD_{\test}}(\epsilon; \calH, \hstar) \geq \abs{\Sigma}^{-(\ell + 1)}\]
    for any distribution $\calD_{\test}$ (i.e., of arbitrary length). The chain-of-thought annotations allow each component of the FSM's transition function to be identified and hence enable generalization to arbitrary test distributions.
\end{example}

\Cref{fig:cotinfo_transfer_simulations} depicts simulation results with the DFA example presented in~\Cref{sec:simulations}, exploring the relative CoT information between a pair of distributions, $\cotinfodomain{\calD_\tr \to \calD_\test}(\epsilon; \calH, \hstar)$ as defined above.~\Cref{fig:cotinfo_transfer:CoTInfo_by_trainlength} depicts the CoT information for a fixed test-distribution, varying the input length in the training distribution, showing that the CoT information curves are increasing with the input length. This suggests that longer, more complex inputs reveal more information about the underlying hypothesis and its input-output behavior. On the other hand,~\Cref{fig:cotinfo_transfer:CoTInfo_by_testlength} depicts the CoT information for fixed training distribution, with input length $L = 5$, varying the test distribution. We see that the CoT information remains relatively large, even for longer and more complex test inputs. This suggests that, under CoT supervision, the observations reveal enough about the hypothesis to identify its behavior beyond the training distribution.

\begin{figure}
   \centering
    \begin{subfigure}{0.475\linewidth}
        \includegraphics[width=\linewidth]{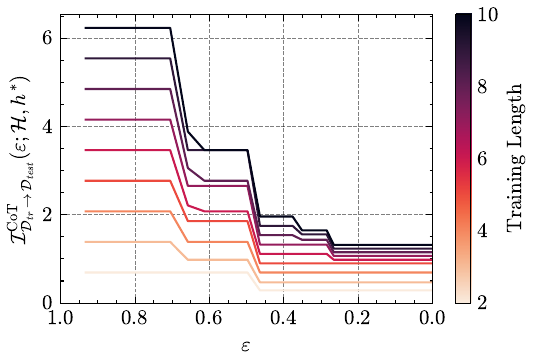}
        \caption{$\cotinfodomain{\calD_{\tr} \to \calD_{\test}}(\epsilon; \calH, \hstar)$ with fixed test distribution at length $L = 5$, and varying the training distribution.}\label{fig:cotinfo_transfer:CoTInfo_by_trainlength}
    \end{subfigure}
    \hfill
    \begin{subfigure}{0.475\linewidth}
        \includegraphics[width=\linewidth]{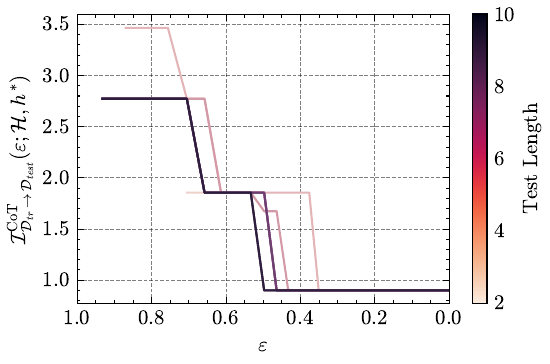}
        \caption{$\cotinfodomain{\calD_{\tr} \to \calD_{\test}}(\epsilon; \calH, \hstar)$ with fixed training distribution at length $L = 5$, and varying the test distribution.}\label{fig:cotinfo_transfer:CoTInfo_by_testlength}
    \end{subfigure}

    % \begin{subfigure}{0.6\linewidth}
    %     \includegraphics[width=\linewidth]{figs/simulation_figs/DFA_TransferLearning_S4A2L5/20250514-021847/figs/learning_curves.pdf}    \caption{Learning Curves with training distribution $\calD_{\tr} = \Unif(\Sigma^5)$, and test distributions $\Unif(\Sigma^L)$, $L = 2, 3, \ldots, 10$.}
    % \end{subfigure}
    \caption{Simulations exploring $\cotinfodomain{\calD_\tr \to \calD_\test}(\epsilon; \calH, \hstar)$}\label{fig:cotinfo_transfer_simulations}
\end{figure}

%%%%%%%%%%%%%%%%%%%%%%%%%%%%%%%%%%%%%%%%%%%%%%%%%%%%%%%%%%%%

\end{document}